%% file: main.tex
\documentclass[pdflatex,sn-mathphys,Numbered]{sn-jnl}
\usepackage{lmodern}
\usepackage{graphicx}%
\usepackage{multirow}%
\usepackage{amsmath,amssymb,amsfonts}%
\usepackage{amsthm}%
\usepackage{mathrsfs}%
\usepackage[title]{appendix}%
\usepackage{xcolor}%
\usepackage{textcomp}%
\usepackage{manyfoot}%
\usepackage{booktabs}%
\usepackage[ruled]{algorithm2e}
\usepackage{listings}%
\usepackage{enumitem}
\usepackage{easy-todo}
\usepackage{dsfont}
\usepackage{bbm}
\usepackage{nicefrac}
\usepackage{bookmark}
\usepackage{subcaption}
\usepackage{hyperref}
\usepackage{tikz}
\usepackage{titlesec}

\graphicspath{{figures/}}

\titleformat{\paragraph}[hang]{\itshape\normalsize\bfseries}{\theparagraph}{1em}{}
\titlespacing*{\paragraph}{0pt}{5pt plus 2pt minus 2pt}{0pt}
\newtheorem{theorem}{Theorem}%  meant for continuous numbers
\newtheorem{lemma}{Lemma}
\newtheorem{corollary}{Corollary}
\newtheorem{prop}[theorem]{Proposition}%

\newtheorem{fact}{Fact}
\newtheorem{assumption}{Assumption}

\newtheorem{innercustomgeneric}{\customgenericname}
\providecommand{\customgenericname}{}
\newcommand{\newcustomtheorem}[2]{%
  \newenvironment{#1}[1]
  {%
   \renewcommand\customgenericname{#2}%
   \renewcommand\theinnercustomgeneric{##1}%
   \innercustomgeneric
  }
  {\endinnercustomgeneric}
}

\newcustomtheorem{customthm}{Theorem}
\newcustomtheorem{customlemma}{Lemma}
\newcustomtheorem{customprop}{Proposition}

\newcommand{\cA}{\mathcal{A}}
\newcommand{\cN}{\mathcal{N}}
\newcommand{\cO}{\mathcal{O}}
\newcommand{\cB}{\mathcal{B}}

\newcommand{\cX}{\mathcal{X}}
\newcommand{\cY}{\mathcal{Y}}

\newcommand{\cJ}{\mathcal{J}}
\newcommand{\E}{\mathbb{E}}
\newcommand{\N}{\mathbb{N}}
\newcommand{\R}{\mathbb{R}}
\newcommand{\Z}{\mathbb{Z}}
\DeclareMathOperator*{\argmax}{arg\,max}

\newcommand{\indic}{\mathds{1}}
\renewcommand{\epsilon}{\varepsilon}
\renewcommand{\hat}{\widehat}
\renewcommand{\tilde}{\widetilde}
\renewcommand{\leq}{\leqslant}
\renewcommand{\geq}{\geqslant}
\newcommand{\demi}{\nicefrac{1}{2}}
\newcommand{\eqdef}{:=}
\newcommand{\abs}[1]{\left| #1 \right|}
\newcommand{\normp}[1]{\Vert #1 \Vert_p}
\newcommand{\bnormp}[1]{\bigl\Vert #1 \bigr\Vert_p}
\newcommand{\Bnormp}[1]{\left\Vert #1 \right\Vert_p}

\newcommand{\norm}[1]{\left\lVert#1\right\rVert}

\DeclareMathOperator{\GreedyBox}{GreedyBox}
\DeclareMathOperator{\SGreedyBox}{Stochastic GreedyBox}
\DeclareMathOperator{\GreedyWB}{GreedyWidthBox}

\newcommand{\point}{b}
\newcommand{\Point}{X}

\newcommand{\lleft}{\mathrm{left}}
\newcommand{\rright}{\mathrm{right}}
\newcommand{\width}{\mathrm{width}}

\begin{document}

\title[Adaptive approximation of monotone functions]{Adaptive approximation of monotone functions}


\author*[1]{\fnm{Pierre} \sur{Gaillard}}\email{pierre.gaillard@inria.fr}
\equalcont{These authors contributed equally to this work.}

\author[2,3]{\fnm{S\'{e}bastien} \sur{Gerchinovitz}}\email{sebastien.gerchinovitz@irt-saintexupery.com}
\equalcont{These authors contributed equally to this work.}

\author[4]{\fnm{\'{E}tienne} \sur{de Montbrun}}\email{edemontb@ens-paris-saclay.fr}
\equalcont{These authors contributed equally to this work.}

\affil[1]{\orgdiv{Univ. Grenoble Alpes}, \orgname{INRIA}, \orgname{CNRS}, \orgname{Grenoble INP}, \orgname{LJK},
  \orgaddress{
    \city{Grenoble}, \postcode{38000},
    \country{France}}}
    
\affil[2]{\orgdiv{DEEL}, \orgname{IRT Saint Exup\'{e}ry}, \orgaddress{\street{3 rue Tarfaya}, \city{Toulouse}, \postcode{31400},
    \country{France}}}
    
\affil[3]{\orgdiv{UMR5219}, \orgname{Institut Math\'{e}matiques de Toulouse}, \orgaddress{\street{118 route de Narbonne}, \city{Toulouse}, \postcode{31400},
    \country{France}}}

\affil[4]{
  \orgname{TSE}, \orgaddress{\street{1, Esplanade de l'Universit\'{e}}, \city{Toulouse}, \postcode{31000},
    \country{France}}}
    

\abstract{
    We study the classical problem of approximating a non-decreasing function $f: \cX \to \cY$ in $L^p(\mu)$ norm by sequentially querying its values, for known compact real intervals $\cX$, $\cY$ and a known probability measure $\mu$ on $\cX$. For any function~$f$ we characterize the minimum number of evaluations of $f$ that algorithms need to guarantee an approximation $\hat{f}$ with an $L^p(\mu)$ error below $\epsilon$ after stopping.
    Unlike worst-case results that hold uniformly over all $f$, our complexity measure is dependent on each specific function $f$.
    To address this problem, we introduce $\GreedyBox$, a generalization of an algorithm originally proposed by Novak (1992) for numerical integration. We prove that $\GreedyBox$ achieves an optimal sample complexity for any function $f$, up to logarithmic factors. Additionally, we uncover results regarding piecewise-smooth functions. Perhaps as expected, the  $L^p(\mu)$ error of $\GreedyBox$ decreases much faster for piecewise-$C^2$ functions than predicted by the algorithm (without any knowledge on the smoothness of $f$). A simple modification even achieves optimal minimax approximation rates for such functions, which we compute explicitly. In particular, our findings highlight multiple performance gaps between adaptive and non-adaptive algorithms, smooth and piecewise-smooth functions, as well as monotone or non-monotone functions. Finally, we provide numerical experiments to support our theoretical results.
}

 % TODO: Fill with appropriate keywords



\maketitle

\section{Introduction}
\label{sec:intro}

Let $\cX,\cY$ be any non-empty compact intervals in $\R$. The problem we consider in this paper is the following. Given any non-decreasing function $f:\cX \to \cY$ that is initially unknown but that a learner can sequentially evaluate at points $x_1, x_2, \ldots \in \cX$ of their choice, how to best estimate $f$ with as few evaluations of $f$ as possible? We will study the $L^p(\mu)$ error as a performance criterion, for some known integer $p \geq 1$ and some known probability measure $\mu$ on $\cX$. More precisely, we will study algorithms that can guarantee an $L^p(\mu)$ error observably below $\epsilon$ after a finite number of evaluations, and will characterize the minimum number of such evaluations to reach this goal, for any non-decreasing function $f$.
Even though this problem is a classical one, finding the best $f$-dependent sample complexity and an algorithm that achieves it is still an open question. 

To make things more formal, we first describe how the learner interacts with the unknown function $f$.

\paragraph{Online protocol.}  Given an accuracy level $\epsilon >0$, the learner first chooses a point $x_1 \in \cX$, then observes $f(x_1)\in \cY$, then chooses $x_2 \in \cX$, then observes $f(x_2)\in \cY$, etc. At each round $t \geq 2$, the point $x_t \in \cX$ is chosen as a measurable function of the whole history $h_{t-1} \eqdef \bigl(x_1,f(x_1),\ldots,x_{t-1},f(x_{t-1})\bigr)$. The process ends after a finite number $\tau_{\epsilon} \geq 1$ of rounds whose value may be determined during the observation process ($\tau_{\epsilon}$ is a stopping time).\footnote{This means that, for any integer $t \geq 1$, whether the inequality $\tau_{\epsilon} \leq t$ is true or not is fully known after observing $h_t$ (in a measurable way).} Finally, after observing the whole history $h_{\tau_{\epsilon}}$, the learner outputs a function $\hat{f}_{\tau_{\epsilon}}:\cX \to \cY$ as a candidate for estimating $f$. We will call \emph{algorithm}\footnote{Throughout the paper our definition of algorithms refers in fact to \emph{adaptive algorithms} that adjust their sequence of points $x_1,\dots,x_t$ to the function $f$ to be approximated based on previous observations. The latter should be contrasted with \emph{non-adaptive algorithms}, for which the sequence of points $(x_t)_{t\geq 1}$ is fixed for all functions $f$'s. } any procedure that, given $\epsilon>0$ and $f$, returns a tuple $\bigl(\tau_\epsilon, (x_t)_{1\leq t\leq \tau_\epsilon },\hat{f}_{\tau_\epsilon}\bigr)$ in $\N^*\times \cX^{\tau_\epsilon} \times (\cX\to \cY)$ satisfying the above online protocol. We only consider deterministic algorithms, except for the integral estimation problem (Section~\ref{sec:integralestimation}) for which randomized algorithms achieve better rates in expectation. 

\paragraph{Learning goal: small number of evaluations with guaranteed \texorpdfstring{$L^p(\mu)$}{Lp(mu)} error.} Let $p \geq 1$ be any positive integer and $\mu$ be any probability measure on $\cX$. The performance of the learner will be evaluated by its \emph{$L^p(\mu)$ error} defined by
\begin{equation}
\Bnormp{\hat{f}_{\tau_\epsilon} - f} \eqdef \left(\int_\cX \big|\hat{f}_{\tau_\epsilon}(x) - f(x)\big|^p d\mu(x)\right)^{1/p} \;.
\label{eq:LpError}
\end{equation}
In all the sequel, an accuracy level $\epsilon > 0$ will be initially given to the learner, who will be required to guarantee that $\normp{\hat{f}_{\tau_\epsilon} - f} \leq \epsilon$ after stopping, for any (initially unknown) non-decreasing function $f:\cX \to \cY$. Given this constraint, the goal of the learner is to make as few evaluations of $f$ as possible, that is, to minimize the stopping time $\tau_\epsilon$. We will also refer to $\tau_\epsilon$ as the \emph{sample complexity} of the algorithm.

\paragraph{Main intuitions and informal presentation of the results.} Before detailing our results in the next sections, we describe the main intuitions in the special case where $p=1$ and $\mu$ is the Lebesgue measure on $\cX = \cY = [0,1]$. The ideas are introduced informally and will be made more precise later.

Imagine that we have already evaluated $f$ at some points $x_1,\ldots,x_t \in [0,1]$. Since we know that $f$ is non-decreasing and bounded between $0$ and $1$, we can deduce that the graph of $f$ is contained inside the $t+1$ adjacent rectangles (or \emph{boxes}) shown in Figure~\ref{fig:mainintuitions}.\footnote{Two boxes are degenerate (and thus not visible) on Figure~\ref{fig:mainintuitions}, since GreedyBox evaluates $f$ at the endpoints $0$ and $1$.} Therefore, estimating $f$ with any function $\hat{f}_t$ whose graph also lies in these $t+1$ adjacent boxes will guarantee an $L^1$ error $\Vert \hat{f}_t - f \Vert_1$ of at most the total area $\xi_t$ of these boxes. We can even achieve $\Vert \hat{f}_t - f \Vert_1  \leq \xi_t/2$ by estimating $f$ with a piecewise-constant function (on each box $B_j = [c_{j-1},c_j] \times \bigl[y_j^-,y_j^+\bigr]$, choose $\hat{f}_t(x) = (y_j^- + y_j^+)/2$) or with a piecewise-affine function (choose $\hat{f}_t$ that linearly interpolates all the observed points $(c_j,f(c_j))$).

Now let $\epsilon > 0$, and suppose that we want to guarantee an $L^1$ error below $\epsilon$ as quickly as possible. Given the above comment, it seems that an ideal choice of the sequence $x_1,x_2,\ldots$ is such that the total area $\xi_t$ of the $t+1$ adjacent boxes at round $t$ falls below $2\epsilon$ for the smallest value of $t$ possible.
We derive lower and upper bounds that support this intuition:
\begin{itemize}[nosep]
    \item Lower bound: if after stopping (at time $\tau_{\epsilon}$) we want to guarantee that $\Vert \hat{f}_{\tau_{\epsilon}} - f \Vert_1 \leq \epsilon$ whatever $f$, then $\tau_{\epsilon}+1$ must be larger than or equal to the minimum number (denoted by $\cN_1(f,2\epsilon)$) of adjacent boxes that contain the graph of $f$ and whose total area is at most of $2 \epsilon$ (see Theorem~\ref{thm:lowerbound}).
    \item A nearly optimal greedy algorithm: of course an optimal choice of $x_1,x_2,\ldots$ is impractical (it would require the full knowledge of the function~$f$). However, a natural algorithm is to choose at each round~$t$ the next point $x_{t+1}$ in the middle of the box with maximum area, so as to greedily reduce the total area of the boxes. See Figure~\ref{fig:mainintuitions} for an illustration. This algorithm, which we call $\GreedyBox$, was suggested in a similar form by \citet{Novak1992}. One of our main contributions is to show that the stopping time $\tau_{\epsilon}$ of $\GreedyBox$ is always at most of the order of the lower bound $\cN_1(f,2\epsilon)$ up to a logarithmic factor in $1/\epsilon$ (see Theorem~\ref{thm:GreedyBoxUB}).
\end{itemize}
Both the lower and upper bounds are proved in a more general setting, in $L^p(\mu)$ norm, for any integer $p\geq 1$ and any probability measure $\mu$ on $\cX$.

\begin{figure}
  \begin{center}
    \input{example_greedybox.tex}
  \end{center}
  \caption{\label{fig:mainintuitions} An illustrative example of the problem on the square function. After 5 iterations, the output of $\GreedyBox$ is represented by the gray boxes. The estimated function is represented by the solid line and its approximation by the dashed line. The next evaluation point $x_6$ divides the box with the largest area in half.}
\end{figure}
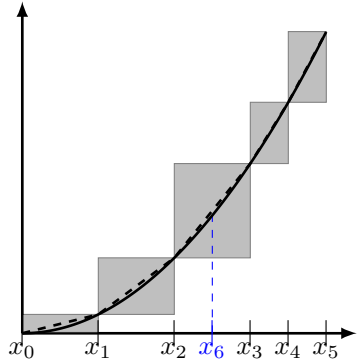

\subsection{Contributions and outline of the paper}
Our main contribution is to characterize the optimal sample complexity of algorithms with guaranteed $L^p(\mu)$ error after stopping (see after Equation \eqref{eq:LpError}), for any non-decreasing function $f$, any $p \geq 1$ and any probability measure $\mu$. 
More precisely:
\begin{itemize}[nosep]
    \item In Section~\ref{sec:lowerbound} we prove a general $f$-dependent lower bound that applies to any algorithm with guaranteed $L^p(\mu)$ error after stopping (see Theorem~\ref{thm:lowerbound}).
    \item In Section~\ref{sec:upperbound} we study GreedyBox (Algorithm \ref{alg:GreedyBox}) and show that its sample complexity matches our lower bound up to logarithmic factors (see Theorem~\ref{thm:GreedyBoxUB}). An important practical feature of GreedyBox is that, at each iteration $t$, it provides a certificate that upper bounds its error and stops as soon as this certificate falls below $\epsilon$.
\end{itemize}

All the results are written in the case where $\cX = \cY = [0,1]$ for convenience, but all of them can be rescaled to any non-empty compact intervals $\cX$ and $\cY$ of $\R$.\footnote{The case where $\cX$ is not closed can be addressed similarly via a simple extension argument and by replacing the values of $f$ at the endpoints of $\cX$ by the endpoints of $\cY$.}

In Section~\ref{sec:improvement_special} we study consequences (with improved rates) for two specific subproblems: 
\begin{itemize}[nosep,leftmargin=*]
    \item \emph{Integral approximation.}
    For this problem, we show that the deterministic version of GreedyBox (Algorithm~\ref{alg:GreedyBox}) is also optimal up to logarithmic factors. However, drawing inspiration from \citet{Novak1992}, we introduce a randomized version in Section~\ref{sec:stoGreedyBox} that improves the accuracy by a factor of $t^{-1/2}$ in expectation after $t$ iterations (Theorem~\ref{thm:SGreedyBoxUB}).
    \item \emph{Worst-case function approximation under a smoothness assumption.}
    In the worst case, the upper bound of Theorem~\ref{thm:GreedyBoxUB} is of the order of $\epsilon^{-1}$ for monotone functions.
    It is well known that for smooth functions, better rates can be achieved using improved quadrature formulas (e.g., \citet{davis2007methods}). For example, $C^2$ functions can be $\epsilon$-approximated in any $L^p(\mu)$ norm after roughly $\epsilon^{-1/2}$ evaluations (also discussed in Appendix~\ref{app:C2trapeze}). In Section~\ref{sec:piecewiseregular}, we provide a minimax lower bound showing that $\smash{\Omega\big(\epsilon^{-1+(\frac{1-\alpha}{1+p})_+}\big)}$ function evaluations are necessary for any algorithm seeking to approximate a piecewise-affine function with $\epsilon^{-\alpha}$ singularities (see Proposition~\ref{prop:counterexample_t_new}). We establish that GreedyBox in fact achieves this rate for piecewise-$C^2$ functions when $\alpha \geq 1/2$, but is suboptimal in the regime $0\leq \alpha < 1/2$ (Theorem~\ref{thm:c2upper_bound} and Proposition~\ref{prop:counterexample_t}). Lastly, we propose a simple modification of GreedyBox that optimally addresses both regimes. These results highlight three significant differences for the $L^p$-approximation problem: 
    \begin{itemize}[nosep,leftmargin=*]
    \item  between monotone piecewise-$C^k$ functions with two or more discontinuities, for which a better rate than $\epsilon^{-1/2}$ is not achievable, and $C^k$ functions (with no singularities) which can be approximated at a rate of $\cO(\epsilon^{-1/k})$;
    \item between general piecewise-$C^2$ functions and monotone piecewise-$C^2$ functions, with a minimax rate respectively of at least $\Omega(\epsilon^{-p})$ and at most $o(\epsilon^{-1})$ for $\alpha <1$; 
    \item between non-adaptive algorithms, which need $\Omega(\epsilon^{-1})$ function evaluations, and adaptive ones, which only require $o(\epsilon^{-1})$  for $\alpha <1$.
    \end{itemize}
\end{itemize}

Finally, in Section~\ref{sec:experiments}, we provide numerical experiments that compare GreedyBox to the trapezoidal method on several functions $f$. Our simulations validate the rates anticipated by our analysis and demonstrate the superiority of our approach compared to the uniform trapezoidal rule in approximating monotone piecewise-smooth functions within the $L^p$ norm.

\subsection{Important definitions and notation}
\label{sec:definitions}

We now introduce several definitions and notations that will be useful to present our results more formally. In all the sequel, we work with $\cX = \cY = [0,1]$, some fixed integer $p \geq 1$ and a probability measure $\mu$ on $[0,1]$.

\paragraph{Standard notation.} We denote by $\N^*$ the set of integers greater than or equal to $1$. For any $x \in \R$, we denote the floor and ceiling functions at $x$ by $\lfloor x \rfloor \eqdef \max\{k \in \Z: k \leq x\}$ and $\lceil x \rceil \eqdef \min\{k \in \Z: k \geq x\}$ respectively. For any measurable function $g:[0,1] \to \R$, we denote the $L^p(\mu)$-norm of $g$ by $\normp{g} \eqdef \bigl(\int_{[0,1]} |g(x)|^p d \mu(x)\bigr)^{1/p}$.

\paragraph{Box, width, generalized area.}
We call \emph{box} any subset $B = [x^-,x^+] \times [y^-,y^+] \subseteq [0,1]^2$ with $x^- < x^+$
and $y^- \leq y^+$.  Its \emph{width} is given by
\[
\width(B) \eqdef \mu\bigl((x^-,x^+)\bigr)
\]
and its \emph{generalized area} is given by
\[
\cA_p(B) \eqdef \Bigl((y^+-y^-)^p \, \mu\bigl((x^-,x^+)\bigr)\Bigr)^{1/p} \;.
\]
Note that the above definitions consider open intervals $(x^-,x^+)$.
The generalized area corresponds to the usual notion of area when $p=1$ and $\mu$ is the Lebesgue measure.

\paragraph{Adjacent boxes.} We denote by $\cB$ the set of all boxes. We say that a finite sequence of $t \geq 1$ boxes $B_1,\dots,B_t \in \cB$ are \emph{adjacent} if and only if they are of the form $B_j = [c_{j-1},c_j] \times [y_j^-,y_j^+]$ for all $j=1,\ldots,t$, for some sequence $0=c_0 < c_1 < \ldots < c_{t-1} < c_t = 1$.

\paragraph{Box-cover and box-covering number of a function.}
Let $f:[0,1]\to [0,1]$ be any non-decreasing function.
We call \emph{box-cover of $f$} any sequence $B_1,\dots,B_t \in \cB$ of adjacent boxes that contains the graph of~$f$ except maybe at the boxes' endpoints, i.e., writing $B_j = [c_{j-1},c_j] \times [y_j^-,y_j^+]$ as above, such that $\bigl\{(x,f(x)) : x \in [0,1] \setminus \{c_0,\ldots,c_t\} \bigr\} \subseteq \cup_{i=1}^t B_i$.
Furthermore, given $\epsilon > 0$, we define two complexity quantities:
\begin{enumerate}
    \item[(i)] $\cN_p(f,\epsilon)$ denotes the minimum cardinality $t$ of a box-cover $B_1,\dots,B_t$ of $f$ whose generalized areas satisfy $\smash{\bigl(\sum_{i=1}^t \cA_p(B_i)^p\bigr)^{1/p} \leq \epsilon}$. We call this quantity the \emph{box-covering number of $f$ at scale $\epsilon$}.
    \item[(ii)] $\cN_p'(f,\epsilon)$ denotes the minimum cardinality $t$ of a box-cover $B_1,\dots,B_t$ of $f$ with generalized areas $\cA_p(B_i) \leq \epsilon$ for all $i=1,\ldots,t$.
\end{enumerate}

All our main results will be expressed in terms of $\cN_p(f,\epsilon)$; the other quantity $\cN_p'(f,\epsilon)$ will however be useful in the proofs. The connections between the two are described in Appendix~\ref{sec:twocomplexitynotions}.

We now reinterpret the condition $\bigl(\sum_{i=1}^t \cA_p(B_i)^p\bigr)^{1/p} \leq \epsilon$ appearing in (i) in an equivalent way. When $p=1$, this condition  corresponds to the total area of the boxes being bounded by $\epsilon$ (as mentioned earlier in the introduction). For general $p \geq 1$, an equivalent and useful formulation is the following. Denote by $B_j = [c_{j-1},c_j] \times [y_j^-,y_j^+]$, $j=1,\ldots,t$, the boxes of the cover, and by $f^-(x) \eqdef \sum_{j=1}^t y_j^- \indic_{x \in (c_{j-1},c_j)}$ and $f^+(x) \eqdef \sum_{j=1}^t y_j^+ \indic_{x \in (c_{j-1},c_j)}$ the best known lower and upper bounds on the function $f$ inside the boxes $B_j$. Then, $\bigl(\sum_{i=1}^t \cA_p(B_i)^p\bigr)^{1/p} \leq \epsilon$ is equivalent to
\begin{equation}
\label{eq:equivalentBoxCoveringNumber}
    \biggl(\, \sum_{j=1}^t \bigl(y_j^+-y_j^-\bigr)^p \mu\bigl((c_{j-1},c_j)\bigr)\biggr)^{1/p} \leq \epsilon\;, \qquad \textrm{that is,} \qquad \normp{f^+-f^-} \leq \epsilon \;.
\end{equation}
The last condition $\normp{f^+-f^-} \leq \epsilon$ implies that $f^-$ and $f^+$ are good lower and upper bounds on the function $f$ outside of the points $c_j$. Note an interesting connection with the definition of bracketing entropy in empirical processes theory (see, e.g., \cite[Chapter 19]{vdV98-AsymptoticStatistics} and references therein). 

The following lemma shows that $\cN_p(f,\epsilon)$ is always well defined and at most of the order of $1/\epsilon$. The proof is postponed to Appendix~\ref{sec:behaviorOfN}, where we collect other useful properties about $\cN_p(f,\epsilon)$.

\medskip
\begin{lemma}
\label{lem:N-basic}
For all non-decreasing functions $f:[0,1] \to [0,1]$ and $\epsilon > 0$, the quantity $\cN_p(f,\epsilon)$ is well defined and upper bounded by
\begin{equation}
\label{eq:trivialupperbound}
\cN_p(f,\epsilon) \leq \lceil 1/\epsilon \rceil \;.
\end{equation}
\end{lemma}

Though the rate of $1/\epsilon$ is tight in the limit $\epsilon \to 0$ for many functions such as $f:x \mapsto x$ and probability measures such as $\mu = \mathrm{Leb}$, the asymptotic behavior of $\cN_p(f,\epsilon)$ when $\epsilon \to 0$ does not necessarily reflect the shape of $f$ for a given $\epsilon \in (0,1]$. Indeed all functions $f$ which are $\epsilon_0$-close to some piecewise-constant function have a very small box-covering number $\cN_p(f,\epsilon_0)$, even if $\cN_p(f, \epsilon)$ is large in the limit $\epsilon \to 0$. Importantly, as we show in the next sections, the estimation problem addressed in this paper is very easy for such functions~$f$ and scales $\epsilon_0$.

\subsection{Related works}
Quadrature formulas (approximation formulas for the computation of an integral) have long been studied for diverse classes of functions and algorithms along the past decades.
We focus on references with strong connections to our problem.
Sukharev \cite{sukharev1986existence} proved that affine methods are minimax optimal in the set of nonadaptive methods for every convex set of functions. It is in particular the case for non-decreasing functions.
Since the work of Bakhvalov \cite{bakhvalov1971optimality}, it is known that nonadaptive methods are as good as adaptive ones for any class of functions that is both convex and symmetric.
Note that the last hypothesis is not verified in our problem.
However, Kiefer \cite{kiefer1957optimum} later showed that the trapezoidal rule is optimal among all deterministic (possibly adaptive) methods for integral approximation in the case of monotone functions.
His work was completed by the one of Novak \cite{Novak1992}, who gave optimal bounds for  different possible types of algorithms as summarized in Table~\ref{tab:ratesNovak}. In particular it is shown that adaption combined with randomization is key to obtaining an improved rate in expectation.
These bounds were later extended by Papageorgiou \cite{papageorgiou1993integration} to the integral approximation of multivariate monotone functions.
The books from Davis \cite{davis2007methods} and Brass \cite{brass2011quadrature} give a larger panel of results on numerical integration under various assumptions.

\begin{table}[!th]
\small
\begin{tabular}{c|cc} \toprule
\bfseries Strategy & \bfseries Non-adaptive  & \bfseries  Adaptive  \\ \midrule
        \bfseries Deterministic & $\epsilon^{-1}$ &  $\epsilon^{-1}$ \\
        \bfseries Stochastic & $\epsilon^{-1}$ & $\epsilon^{-2/3}$ \\
\bottomrule
\end{tabular}
\caption{Minimax rates proved by \citet{Novak1992}.}
\label{tab:ratesNovak}
\end{table}

Of course, many other function sets were studied.
A non-exhaustive list includes work on the set of unimodal functions \cite{novak1996numerical}, on the set of functions with bounded variation \cite{graf1990average}, on convex and symmetric classes of functions \citep{novak1993quadrature,novak1995optimal,hinrichs2011curse}, and on various classes of multivariate functions \citep{ritter1993multivariate,katscher1996quadrature,krieg2017universal}.
Note that the bounds in the previous papers are not $f$-dependent. 

All of the aforementioned works were on integral approximation, an easier problem than $L^p$ approximation.
Numerous studies investigate the $L^p$ approximation of diverse classes of functions, mostly using interpolation.
A known example is polynomial interpolation for $k$-times continuously differentiable functions.
Many works employ the modulus of smoothness of the function $f$ to bound its $L^p$ approximation, providing $f$-dependent bounds.
It is for example the case of \cite{chandra2002trigonometric} for the approximation of periodic functions using trigonometric polynomials.
The computation of a best linear  $L^p$-approximation, where a basis $\phi_1, \ldots, \phi_k$ of functions is previously given and one looks for the weights $w \in \R^k$ that minimize $\norm{f - \sum_i w_i \phi_i}_p$, was studied in depth (e.g. \citep{fletcher1971calculation,fletcher1974linear}).
\cite{plaskota2005adaption} studied the class of functions with their first $k$ derivatives continuous except at one singularity, and showed that adaptive algorithms are better that non-adaptive in the case of integral estimation. The same remark and work was later carried to approximation in \cite{plaskota2008power}.
Some work involves $k$-monotone functions, that is, functions with monotone $k$-th derivatives.
The papers \cite{kopotun2003k,kopotun2009nearly} studied the rate of convergence of interpolation methods on this set of functions, and showed results depending on the modulus of smoothness of the function.
To our knowledge, little is known about $L^p(\mu)$ approximation of general non-decreasing functions, without any continuity or smoothness assumptions.

Adaptive methods have a long history in numerical integration and other approximation or learning problems.
For numerical integration of monotone functions, as mentioned above, adaption combined with randomization is key to improving worst-case rates (see \cite{Novak1992} and Table~\ref{tab:ratesNovak} above, as well as \cite{novak1996power}).
We show that adaption is also key to obtaining less pessimistic, $f$-dependent error bounds. This work also shares algorithmic principles with online learning methods for (possibly noisy) black-box optimization, such as bandit algorithms \cite{lattimore2020bandit} or the EGO algorithm \cite{jones1998efficient}.
Indeed such algorithms rely on adaptive sampling strategies to reduce the current uncertainty (via, e.g. UCB or Bayesian approaches), which is reminiscent of the way GreedyBox selects the box to be split at time $t$.
Close to our paper is the work by Bachoc et al. \cite{bachoc2021instance}, who derive $f$-dependent error bounds for certified black-box Lipschitz optimization. 

Lastly, \citet{bonnet2020adaptive} explore a close variant of our main algorithm (Algorithm~\ref{alg:GreedyBox}), which itself draws significant inspiration from~\citet{Novak1992} for numerical integration. \citet{bonnet2020adaptive} address the problem of adaptively reconstructing a monotone function from imperfect observations. In contrast to our approach, they do not provide any guarantees regarding sample complexity or $L^p$ approximation. Their focus is on asymptotic convergence guarantees, including point-wise, $L^1$, or $L^\infty$ norm convergence, as the number of function evaluations tends towards infinity. They do not provide any convergence rate information, nor finite time or f-dependente guarantees. Nevertheless, they highlight an intriguing application in uncertainty quantification that could potentially benefit from our analysis.

\section{Lower bound}
\label{sec:lowerbound}

In this section we prove a lower bound on the number of evaluations of $f$ that any deterministic algorithm must request in order to guarantee an $\epsilon$-approximation of $f$ in $L^p(\mu)$-norm after stopping, when only given the prior knowledge that $f:[0,1] \to [0,1]$ is non-decreasing.\textbf{} The next theorem states that in such a case, at least $\cN_p(f,2\epsilon)-1$ evaluations of $f$ are necessary.

In the sequel we write $\tau(f)$ to make it explicit that the stopping time $\tau$ of the algorithm depends on the underlying function $f$ (through the sequentially observed values $f(x_1),f(x_2),\ldots$).

\medskip
\begin{theorem}
\label{thm:lowerbound}
Let $\epsilon > 0$ and $\bigl(\tau(f), (x_t)_{1\leq t\leq \tau(f)},\hat{f}_{\tau(f)}\bigr)$ be the output of any deterministic algorithm such that, for all non-decreasing functions $f:[0,1] \to [0,1]$,
\[
\tau(f) < + \infty \qquad \textrm{and} \qquad \bnormp{\hat{f}_{\tau(f)}-f} \leq \epsilon \;.
\]
Then, for all non-decreasing functions $f:[0,1] \to [0,1]$,
\[
\tau(f) \geq \cN_p(f,2\epsilon)-1 \;.
\]
\end{theorem}

In words, any algorithm that is guaranteed to output an $\epsilon$-approximation after finitely-many evaluations whatever the non-decreasing function $f$ \emph{must} evaluate each $f$ at least $ \cN_p(f,2\epsilon)-1$ times before stopping.
In Section~\ref{sec:upperbound} we show a matching upper bound up to a multiplicative factor of the order of $p \log(1/\epsilon)$. This indicates that the box-covering number $\cN_p(f,2\epsilon)$ introduced in Section~\ref{sec:definitions} is a key quantity to describe the inherent difficulty of the estimation problem.

Note that the lower bound holds for every $f$ simultaneously. It thus has a similar flavor to distribution-dependent lower bounds that have been proved for the stochastic multi-armed bandit problem in online learning theory (see, e.g., Chapter~16 by Lattimore and Szepesv\'{a}ri \cite{lattimore2020bandit}). Recently an $f$-dependent lower bound (also based on a notion of cover) was proved by Bachoc et al. \cite{bachoc2021instance} for certified zeroth-order Lipschitz optimization, where algorithms are required to output error certificates (i.e., observable upper bounds on the optimization error).

\begin{proof}
  Assume for a moment that there exists a non-decreasing function $g:[0, 1] \to [0, 1]$ such that $\tau(g) < \cN_p(g, 2\epsilon) - 1$.
  When run on $g$, the algorithm only uses the $\tau(g)$ query points $x_1,\ldots,x_{\tau(g)}$ before stopping. Let $0 < \tilde{x}_1 <  \ldots < \tilde{x}_n < 1$ denote the ordered values after removing redundancies and the values $0$ and $1$ (if applicable), with $0 \leq n \leq \tau(g)$. Consider the adjacent boxes $B_i = [\tilde{x}_i, \tilde{x}_{i+1}] \times [g(\tilde{x}_i), g(\tilde{x}_{i+1})]$ for $i \in \{0, \ldots, n\}$ where we set $\tilde{x}_0 = 0$ and $\tilde{x}_{n+1} = 1$.
  The sequence $B_0, \ldots, B_n$ is a box-cover of $g$ (by monotonicity).
  We construct two functions $g_-$ and $g_+$ that surround~$g$:
  \[ g_- \colon x \mapsto
    \left\{
        \begin{array}{ll}
            \!g(\tilde{x}_i)  &\text{if } \tilde{x}_i \leq x < \tilde{x}_{i+1} \\
            \!g(1)\hspace*{-2pt} &\text{if } x=1
        \end{array}
    \right. \hspace*{-4pt}
    \text{and }
    g_+ \colon x \mapsto \left\{\begin{array}{ll}  \!g(\tilde{x}_{i+1}) \hspace*{-4pt} &\text{ if } \tilde{x}_i < x  \leq  \tilde{x}_{i+1} \\ \!g(0) &\text{ if } x=0\;. \end{array} \right.\]
  Since $g_-(\tilde{x}_i) = g_+(\tilde{x}_i) = g(\tilde{x}_i)$ for all $i \in \{0, \ldots, n+1\}$, the algorithm (which is deterministic) would behave the same when run with $g_-$ or $g_+$ as when run with $g$, and would construct the same approximation function $\hat{g}_{\tau(g)}$ after the same number $\tau(g_-) = \tau(g_+) = \tau(g)$ of evaluations.
  However,
  \[ \normp{g_+-g_-}^p = \sum_{i=0}^{n} (g(\tilde{x}_{i+1}) - g(\tilde{x}_i))^p\mu((\tilde{x}_i, \tilde{x}_{i+1})) = \sum_{i=0}^{n} \mathcal{A}_p^p(B_i) > (2\epsilon)^p \;,\]
  where the last inequality follows from $n+1 \leq \tau(g)+1 < \cN_p(g, 2\epsilon)$ and the definition of $\cN_p(g, 2\epsilon)$.
  The triangle inequality then yields
  \[ \normp{\hat{g}_{\tau(g)}-g_-} + \normp{\hat{g}_{\tau(g)}-g_+} \geq \normp{g_+-g_-} > 2\epsilon \;, \]
  which shows that one of $\normp{\hat{g}_{\tau(g)}-g_-}$ or $\normp{\hat{g}_{\tau(g)}-g_+}$ is larger than $\epsilon$. Since (as proved above) $\hat{g}_{\tau(g)}$ is the approximation function output by the algorithm both with $g_-$ and $g_+$, which are both non-decreasing, the last conclusion is in contradiction with the assumption that $\normp{\hat{f}_{\tau(f)} - f} \leq \epsilon$ for all non-decreasing functions $f:[0, 1] \to [0,1]$. This concludes the proof.
\end{proof}

\section{Upper bound}
\label{sec:upperbound}

In this section we introduce the GreedyBox algorithm and derive an $f$-dependent sample complexity bound (Theorem~\ref{thm:GreedyBoxUB}) that matches the lower bound of Theorem~\ref{thm:lowerbound} up to a logarithmic factor, for every non-decreasing function $f$. In Section~\ref{sec:improvement_special} we will study consequences and derive improved bounds for integral estimation (in expectation) and worst-case approximation of piecewise-$C^2$ functions.

\subsection{Algorithm and main result}

We consider Algorithm~\ref{alg:GreedyBox} below, which draws heavily on an algorithm proposed by \citet[Section~3.2]{Novak1992} for numerical integration, and which we call $\GreedyBox$ thereafter. A variant for handling imperfect observations was also considered by~\citet{bonnet2020adaptive}.

It is remarkably simple: at every round, it selects the largest box in the current box-cover of~$f$ and replaces it with two smaller boxes by evaluating $f$ at the middle or, more generally, at a conditional median for a general probability measure $\mu$. At any $t \geq 1$, we approximate $f$ with the trapezoidal estimator $\hat f_t$ defined as the piecewise-affine function that joins the points $\bigl(\point^t_k,f(\point^t_k)\bigr)_{0 \leq k \leq t}$ visited up to time $t$. Note that this estimator uses $t+1$ evaluations of $f$. We stop Algorithm~\ref{alg:GreedyBox} at time $\tau_{\epsilon}$, which is the first $t \geq 1$ when the certificate $\xi_t = \sum_{k=1}^{t} (a_k^{t})^p$ falls below $\epsilon^p$. (This is because $\xi_t$ is a valid upper bound on $\big\|\hat f_t - f\big\|_p^p$, by Lemma~\ref{lem:GreedyBoxError} below.)

\begin{algorithm}[!ht]
  {\bfseries Input:} $\epsilon \in (0, 1], p\geq 1$, probability measure $\mu$ on $[0,1]$.\\
  {\bfseries Init:} Set $t=1$, $x_0=b_0^1=0$,  $x_1=b_1^1=1$, evaluate $f(0)$ and $f(1)$, and set $\xi_1 = (a_1^1)^p = (f(1) - f(0))^p  \, \mu\bigl((0,1)\bigr)$.\\[5pt]
  \begin{minipage}{.95\textwidth}
  \While{$\xi_t > \epsilon^p$}{
    \begin{enumerate}[topsep=0pt,parsep=0pt,itemsep=0pt,leftmargin=10pt]
    \item \label{line:GB_selection} Select a box with largest generalized area: pick $k_*^t \in \argmax_{k \in \{1,\ldots,t\}} a^t_k$.
    \item \label{step:GB-median} Let $x_{t+1}$ be a median of the conditional distribution $\mu(\cdot |(\point_{k_*^t-1}^t,\point_{k_*^t}^t))$, and evaluate $f$ at $x_{t+1}$.
    \item \label{step:GB-sort} Sort the points $x_0,x_1,\ldots,x_{t+1}$ in increasing order: \\
    \hspace*{2cm}$\point_0^{t+1} = 0 < \point_1^{t+1} < \dots < \point_{t+1}^{t+1} = 1$ .
    \item Define the generalized areas for all $ k\in \{1,\dots, t+1\}$ by\\
    \hspace*{2cm} $a_k^{t+1} := \mu\!\left((\point_{k-1}^{t+1},\point_k^{t+1})\right)^{1/p} (f(\point_k^{t+1})-f(\point_{k-1}^{t+1})) \;.$
    \item Update the certificate
    \begin{equation}\label{eq:certificate}
    {\smash{
    \xi_{t+1} = \sum_{k=1}^{t+1} (a_k^{t+1})^p \,.    }}
    \end{equation}
    \item Let $t \leftarrow t+1$.
    \end{enumerate}
  }
    Set $\tau_\epsilon = t$ and approximate $f$ with the piecewise-affine function $\hat f_{\tau_\epsilon}$ defined by: \vspace*{-7pt}
    \[\smash{
        \forall x \in [0,1] \qquad \hat f_{\tau_\epsilon}(x) = \frac{f(b_k^{\tau_\epsilon}) - f(b_{k-1}^{\tau_\epsilon})}{b_k^{\tau_\epsilon} - b_{k-1}^{\tau_\epsilon}} (x - b_{k-1}^{\tau_\epsilon}) + f(b_{k-1}^{\tau_\epsilon}) \,,}
    \]
    for $k \in \{1,\dots,{\tau_\epsilon}\}$ such that $b_{k-1}^{\tau_\epsilon} \leq x \leq b_{k}^{\tau_\epsilon}$. \\
     {\bfseries Output:} $(\tau_\epsilon, (x_t)_{1\leq t\leq \tau_\epsilon}, \hat f_{\tau_\epsilon})$ \,.
  \end{minipage}
  \caption{$\GreedyBox$ (inspired from \citet[Section~3.2]{Novak1992}).}
  \label{alg:GreedyBox}
\end{algorithm}

\paragraph{Algorithmic complexity.}
We assume that a median of the conditional distribution $\mu(\cdot |(\point_{k_*^t-1}^t,\point_{k_*^t}^t))$ can be computed exactly at every round $t$. When $\mu$ is the Lebesgue measure, it can indeed be computed in closed form: it is the midpoint $\smash{(\point_{k_*^t-1}^t+\point_{k_*^t}^t)/2}$.

For the sake of simplicity, in Algorithm~\ref{alg:GreedyBox} we perform a sort (Step~\ref{step:GB-sort}) and an argmax operation (Step~\ref{line:GB_selection}) at each round $t$, to get the points in increasing order and to choose the box with the largest generalized area. However, one can get rid of the sort operation at each round and do it only once at the end, because GreedyBox does not need the order of the boxes before the last iteration, where it uses sorted points to build $\hat f_{\tau_\epsilon}$.
Furthermore, for the argmax operation, naive methods yield a computational complexity of $\mathcal{O}(t)$ at each time $t$, resulting in a quadratic complexity for $\GreedyBox$, far worse than the linear complexity of traditional algorithms such as the trapezoidal rule.
To speed up $\GreedyBox$, we use a classical algorithmic trick: a max-heap, which is a binary tree that takes logarithmic time to both remove the maximum value and add an element. This provides $\GreedyBox$ with a computational complexity of $\mathcal{O}(t\log(t))$ after $t$ rounds, which is closer to the complexity of the trapezoidal rule.

\paragraph{Upper bound on the sample complexity.}
Let $\epsilon \in (0,1]$ be some target accuracy level. 
The next theorem provides a bound on the sample complexity $\tau_\epsilon$ of $\GreedyBox$. The proof appears in Section~\ref{sec:proof-upperbound}.

\medskip
\begin{theorem}
\label{thm:GreedyBoxUB}
Let $f:[0,1]\to [0,1]$ be non-decreasing,\footnote{Recall that the input and output sets of $f$ can be rescaled to any non-empty compact intervals $\cX$ and $\cY$ of $\R$, changing the results only by a multiplicative constant.} $p\geq 1$, and $\epsilon \in (0,1]$. Then, $\GreedyBox$ defined above (Algorithm~\ref{alg:GreedyBox}) satisfies
\[
  \big\|\hat f_{\tau_\epsilon} - f\big\|_p := \bigg(\int_0^1  \big( \hat f_{\tau_\epsilon}(x) - f(x) \big)^p dx\bigg)^{1/p}  \leq \epsilon \,,
\]
at the stopping time $\tau_\epsilon$.
Furthermore, its sample complexity is bounded as follows:
\[
\tau_\epsilon \leq 32 p^2 \bigl(\log_2(2/\epsilon^2)+2\bigr)^2 \cN_p(f,\epsilon) \,.
\]
\end{theorem}

We make three comments before proving the theorem.

\textit{A new $f$-dependent bound.} Since $\cN_p(f,\epsilon) \leq \lceil 1/\epsilon \rceil$ for all non-decreasing functions~$f$ (by Lemma~\ref{lem:N-basic}), the above sample complexity bound $\tau_\epsilon = \cO \left(\cN_p(f,\epsilon) \log^2(1/\epsilon)\right)$ implies the well-known upper bound of $\cO(1/\epsilon)$ up to logarithmic factors in the worst case.
Importantly, though the rate of $1/\epsilon$ is worst-case optimal (see, e.g., \cite[Section~5.A]{kiefer1957optimum}), Theorem~\ref{thm:GreedyBoxUB} yields a much better bound for functions $f$ that are easier to approximate, such as functions close to piecewise-constant functions, because $\cN_p(f,\epsilon)$ is small in that case. Since the $\GreedyBox$ algorithm uses no prior knowledge on $f$ (beyond monotonicity) to stop at $\tau_\epsilon$, it is adaptive to the unknown complexity $\cN_p(f,\epsilon)$.

\textit{A nearly optimal bound.} Note that the lower bound of Theorem~\ref{thm:lowerbound} is in terms of $\cN_p(f,2\epsilon)$, while the upper bound of Theorem~\ref{thm:GreedyBoxUB} is proportional to $\cN_p(f,\epsilon)$. By a simple argument (dividing boxes $p$ times to reduce their generalized widths by a factor of $2^p$, similarly to the proof of Lemma~\ref{lem:areadivided}), we can prove that $\cN_p(f,\epsilon) \leq 2^p \cN_p(f,2\epsilon)$. Therefore, the lower and upper bounds of Theorems~\ref{thm:lowerbound} and~\ref{thm:GreedyBoxUB} match up to a logarithmic factor. For each non-decreasing function $f:[0,1] \to [0,1]$, GreedyBox is thus nearly optimal among all algorithms with guaranteed $L^p(\mu)$ error after stopping.

\textit{A possible minor improvement.} When $\mu$ is the Lebesgue measure,
the bound on $\tau_\epsilon$ could be slightly improved (in the constants) by replacing the certificate in Equation~\eqref{eq:certificate} with $\xi_{t+1} = \smash{\frac{1}{1+p}\sum_{k=1}^{t+1} (a_k^{t})^p}$ (see Lemma~\ref{lem:gainFactor2} in Appendix~\ref{app:proof_lemma_greedyboxerror}). While a similar minor improvement is likely to hold for general $\mu$ with a slightly different interpolation $\hat{f}_t$ (non-necessarily piecewise-affine), we decided to focus on piecewise-affine interpolations for the sake of presentation.

\subsection{Proof of Theorem~\ref{thm:GreedyBoxUB}}
\label{sec:proof-upperbound}

Before proving Theorem~\ref{thm:GreedyBoxUB}, we first state three lemmas, whose proofs are all postponed to Appendix~\ref{sec:omittedproofs}.

The first one below shows that, at any round $t$ before stopping, the error of $\GreedyBox$ is at most the sum of the generalized areas to the power $p$ of the current-box cover of~$f$.
We recall that $\smash{a_k^t := (\mu(\point_{k-1}^t, \point_k^t))^{1/p} (f(\point_k^t)-f(\point_{k-1}^t))}$ denotes the generalized area of the $k$-th box at round~$t$, and we define the trapezoidal estimato $\hat f_t$ to be the piecewise-affine function that joins the points $\bigl(\point^t_k,f(\point^t_k)\bigr)_{0 \leq k \leq t}$ visited up to time $t$. 

\smallskip
\begin{lemma}
\label{lem:GreedyBoxError}
Let $f:[0,1]\to [0,1]$ be non-decreasing, $p\geq 1$ and $\epsilon \in (0,1]$. For any $t \in \{1,\ldots,\tau_\epsilon\}$, \vspace*{-5pt}
\[
  \big\|\hat f_t - f\big\|_p^p := \int_0^1  \abs{ \hat f_t(x) - f(x) }^p d\mu(x) \leq
  \sum_{k=1}^{t} (a_k^{t})^p =: \xi_t \;.
\]
\end{lemma}

The next two lemmas are used to control $\tau_\epsilon$. We first show (by a dichotomy argument) that the algorithm can quickly make all boxes equally small. Recall from Section~\ref{sec:definitions} that $\cN'_p(f,\epsilon)$ denotes the minimum cardinality of a box-cover of $f$ for which each box has a generalized area below $\epsilon$.

\smallskip
\begin{lemma}
\label{lem:smallboxes}
Let $f:[0,1]\to [0,1]$ be non-decreasing, $p\geq 1$ and $\epsilon \in (0,1]$. Define $\tau_\epsilon' := 2 \bigl(1+\lceil p \log_2(1/\epsilon)\rceil\bigr) \cN'_p\bigl(f,\epsilon\bigr)$, and assume that $\GreedyBox$ is such that $\tau_\epsilon > \tau'_\epsilon$. Then, at time $\tau'_\epsilon$, all the boxes maintained by $\GreedyBox$ have a generalized area bounded from above by $\epsilon$, i.e., $\smash{a_k^{\tau_\epsilon'} \leq \epsilon}$ for all $k \in \{1,\dots,\tau_\epsilon'\}$.
\end{lemma}

The next lemma shows that the certificate $\xi_t = \sum_{k=1}^t (a_k^t)^p$ at round~$t$ decreases at least linearly in $t$.

\smallskip
\begin{lemma}
\label{lem:areadivided}
Let $f:[0,1]\to [0,1]$ be non-decreasing, $p\geq 1$ and $\epsilon \in (0,1]$. For any $t \in \{1,\ldots,\lfloor \tau_\epsilon/2 \rfloor\}$, we have $\xi_{2t} \leq  \xi_t / 2$. Therefore, for all $t \leq s$ in $\{1,\ldots,\tau_\epsilon\}$,
\begin{equation}
\xi_{s} \leq \frac{\xi_{t}}{2^{\lfloor \log_2(s/t) \rfloor}} \leq \Big(\frac{2 t}{s}\Big) \, \xi_{t} \,.
\label{eq:lineardecrease}
\end{equation}
\end{lemma}

\begin{proof}[Proof of Theorem~\ref{thm:GreedyBoxUB}]
We are now ready to prove the theorem.
The first inequality follows immediately from Lemma~\ref{lem:GreedyBoxError} and from the fact that $\xi_{\tau_\epsilon} \leq \epsilon^p$ by definition of $\tau_\epsilon$.

We now show by contradiction that $\tau_{\epsilon} \leq 32 \bigl(1+\lceil p\log_2(2/\epsilon^2) \rceil\bigr)^2 \cN_p(f,\epsilon)$. Assume thus for a moment that
\begin{equation}
    \label{eq:taueps-absurde}
    \tau_{\epsilon} > 32 \bigl(1+\lceil p\log_2(2/\epsilon^2) \rceil\bigr)^2 \cN_p(f,\epsilon) \;.    
\end{equation}
This assumption will be used implicitly when calling Lemmas~\ref{lem:smallboxes} and~\ref{lem:areadivided} below, since it will imply that $\tau'_{\epsilon'} \leq \tau''_{\epsilon} < \tau_{\epsilon} \leq \tau_{\epsilon'}$ (so that the algorithm has not stopped before any round considered below). We will see in the end that it raises a contradiction. \\

Let $n_\epsilon := \cN_p(f,\epsilon)$. By Lemma~\ref{lem:smallboxes} applied with $\epsilon' = \epsilon/n_\epsilon^{1/p}$, at time $\tau'_{\epsilon'} := 2 \bigl(1+\lceil p \log_2(1/\epsilon')\rceil\bigr) \cN_p'(f,\epsilon')$, the $\tau'_{\epsilon'}$ boxes maintained by $\GreedyBox$ all have generalized areas at most of $\epsilon'$ each, so that the certificate $ \xi_{\tau'_{\epsilon'}}$ satisfies
\begin{align}
    \xi_{\tau'_{\epsilon'}} \leq   \tau'_{\epsilon'} \cdot (\epsilon')^p & \leq  2 \bigl(1+\lceil p \log_2(1/\epsilon')\rceil\bigr) \cN_p'(f,\epsilon') \cdot \frac{\epsilon^p}{n_\epsilon} \nonumber \\
    & \leq 4 \bigl(1+\lceil p\log_2(2/\epsilon^2) \rceil\bigr) \,\epsilon^p \;, \label{eq:areatprime}
\end{align}
where we used the fact that $\cN_p'\big(f,\epsilon'\big) \leq 2 \cN_p(f,\epsilon)$ (by Lemma~\ref{lem:N} in Appendix~\ref{sec:twocomplexitynotions}) and that $1/\epsilon'=\cN_p(f,\epsilon)^{1/p}/\epsilon \leq 2^{1/p} /\epsilon^{(p+1)/p} \leq 2/\epsilon^2$ (since $\cN_p(f,\epsilon) \leq \lceil 1/\epsilon \rceil \leq 2/\epsilon$).
Now, we apply Lemma~\ref{lem:areadivided} with $t=\tau'_{\epsilon'}$ and
$s = \tau_\epsilon'' := 8 \bigl(1+\lceil p\log_2(2/\epsilon^2) \rceil\bigr) \tau'_{\epsilon'}$. It yields:
\[
    \xi_{\tau_\epsilon''} \leq \bigg(\frac{2 \tau'_{\epsilon'}}{ 8 \bigl(1+\lceil p\log_2(2/\epsilon^2) \rceil\bigr) \tau'_{\epsilon'}} \bigg) \, \xi_{\tau'_{\epsilon'}}\  \stackrel{\text{by } \eqref{eq:areatprime}}{\leq}\   \epsilon^p\,.
\]
This raises a contradiction with \eqref{eq:taueps-absurde}, since (using again $\cN'_p\big(f,\epsilon'\big) \leq 2 \cN_p(f,\epsilon)$)
\[
    \tau_\epsilon'' = 8 \bigl(1+\lceil p\log_2(2/\epsilon^2) \rceil\bigr) \tau'_{\epsilon'} \leq  32 \bigl(1+\lceil p\log_2(2/\epsilon^2) \rceil\bigr)^2 \cN_p(f,\epsilon)
\]
and $\tau_\epsilon$ is by definition the first time~$t$ such that $\xi_t \leq \epsilon^p$. Therefore, \eqref{eq:taueps-absurde} must be false, so that $\tau_{\epsilon} \leq 32 \bigl(1+\lceil p\log_2(2/\epsilon^2) \rceil\bigr)^2 \cN_p(f,\epsilon)$. Elementary calculations conclude the proof of Theorem~\ref{thm:GreedyBoxUB}.
\end{proof}

\section{Improvement for special cases}
\label{sec:improvement_special}

In this section, we derive consequences (with rates faster than the worst-case $\epsilon^{-1}$) in two specific cases: integral estimation and piecewise-smooth functions.

In the sequel, we adopt a slightly different yet equivalent viewpoint than in Section~\ref{sec:upperbound}. Though Algorithms~\ref{alg:SGreedyBox} and~\ref{alg:greedywidthbox}, defined in this section, formally stop at round $\tau_\epsilon$, in the proofs, we extend their definitions to all rounds $t \geq 1$, by replacing the while condition with $\xi_t > 0$, defining $\tau_0$ as the first round $t$ (if any) where the certificate $\xi_t$ reaches $0$, and setting $\hat{f}_s := \hat{f}_{\tau_0}$ for all subsequent rounds $s \geq \tau_0 + 1$. Note that their approximation error equals zero for all $s \geq \tau_0$.

\subsection{Side problem: integral estimation}
\label{sec:integralestimation}

Throughout this subsection, we focus on the case of integral estimation rather than approximation in $L^p(\mu)$-norm. The goal is to approximate the integral $\smash{I(f) = \int_0^1 f(x) d \mu(x)}$ of a non-decreasing function $f$ on $[0,1]$.
This problem is simpler than the $L^1(\mu)$-approximation problem studied previously, and thus $\GreedyBox$ can be easily extended to integral estimation while maintaining the same bound.

A deterministic algorithm for the integral estimation problem is defined as a procedure that, given $\epsilon > 0$, produces a tuple $(\tau_\epsilon, (x_t)_{1\leq \tau_\epsilon}, \hat{I}_{\tau_\epsilon}(f)) \in \N_+\times \cX^{\tau_\epsilon} \times \R$, where $(x_t)_{t \geq 1}$ and $\tau_\epsilon$ are defined sequentially, similar to the approximation in $L^p(\mu)$-norm.
That is: for all $t\geq 2$, $x_t$ is a measurable function of the history $h_{t-1}$ and $\tau_\epsilon$ is a stopping time after which the process of the algorithm ends.
The algorithm finally  outputs an approximation $\hat{I}_{\tau_\epsilon}(f)$ of the integral.
We now study the convergence speed of $\GreedyBox$ in this setting (we replace the definition of $\hat{f}_{\tau_\epsilon}$ in the last line of GreedyBox by the computation of $\hat{I}_{\tau_\epsilon}(f) = \int_0^1 \hat f_{\tau_\epsilon}(x) d\mu(x)$) and check whether it achieves optimal convergence speed.

\subsubsection{Nearly optimal performance for integral estimation}

The following upper bound on the sample complexity of GreedyBox is a direct consequence of Theorem~\ref{thm:GreedyBoxUB} with $p=1$.

\begin{corollary}
\label{cor:integralestimation}
Let $f:[0,1]\to [0,1]$ be non-decreasing, and $\epsilon \in (0,1]$. Then, $\GreedyBox$ with $p=1$ satisfies, at the stopping time $\tau_\epsilon$,
\[
  \left| \int_0^1 \hat f_{\tau_\epsilon}(x)d\mu(x) - \int_0^1 f(x)d\mu(x)\right|  \leq \epsilon \,.
\]
Besides, its sample complexity is bounded from above as follows:
\[
\tau_\epsilon \leq 32 p^2 \bigl(\log_2(2/\epsilon^2)+2\bigr)^2 \cN_p(f,\epsilon) \,.
\]
\end{corollary}

The question is now to check if the lower bound for this weaker problem is still the same one.
The following theorem asserts that this is the case.

\medskip
\begin{theorem}[Matching lower bound]
\label{thm:lb-integralestimation}
Let $\epsilon>0$ and $\mathcal{A}$ be any deterministic algorithm such that, for all non-decreasing functions $f:[0,1] \to [0,1]$:
\[
    \tau_\epsilon(f) < +\infty \text{ and } \left| \hat{I}_{\tau_\epsilon(f)}(f) - \int_0^1 f(x)d\mu(x) \right| \leq \epsilon \,.
\]
Then, for all non-decreasing functions $f:[0,1] \to [0,1]$
\[
    \tau_\epsilon(f) \geq \cN_p(f,2\epsilon)-1 \,.
\]
\end{theorem}

The proof closely follows that of Theorem~\ref{thm:lowerbound} in the case of $p=1$ and is left to the reader. Note that it is inspired from that of the well-known minimax lower bound of $1/(2n+2)$. This lower bound actually implies the lower bound of Theorem~\ref{thm:lowerbound} for $p=1$.

\subsubsection{Improvement in expectation with randomization}
\label{sec:stoGreedyBox}

We now provide a stochastic version of our algorithm; see Algorithm~\ref{alg:SGreedyBox} below, which we call $\SGreedyBox$. With this randomized variant we prove better guarantees (in expectation only)\footnote{We could easily derive high probability bound using Hoeffding's lemma.} for the integral estimation problem than with the deterministic version. Note that the improvement is not true for estimating $f$ in $L^p(\mu)$-norm. The idea to use randomization to improve the rates is due to \citet[Section~2.2]{Novak1992}; we adapt this idea to a fully sequential algorithm, whose bound is now adaptive to the complexity of $f$. It's worth mentioning that, for the sake of simplicity, the random points $\Point_k$ in Algorithm~\ref{alg:SGreedyBox} are currently sampled at the conclusion of the algorithm. However, an alternative approach could involve sequential sampling when the intervals $(b_{k-1}^t, b_k^t)$ are created, in order to get a sequential estimator $\hat I_t(f)$ for all $t\geq 1$.

\begin{algorithm}[!ht]
{\bfseries Input:} $\epsilon >0$, probability measure $\mu$ on $[0,1]$\\
{\bfseries Init:} Set $t=1$, $x_0=b_0^1=0$ and $x_1=b_1^1=1$,  evaluate $f(0)$ and $f(1)$ and set $\xi_1 = a_1^1/2 = \mu((0,1))(f(1)-f(0))/2$\;
\While{$\xi_t > \epsilon$}{
\begin{enumerate}[topsep=0pt,parsep=0pt,itemsep=0pt]
	\item Select the box with the largest area: find $k_*^t \in \argmax_{1\leq k\leq t} a_k^t$.
	\item Let $x_{t+1}$ be a median  of the conditional distribution $\mu(\cdot | (b_{k_*^t-1}^t, b_{k_*^t}^t))$ and evaluate $f$ at $x_{t+1}$.
        \item Sort the points $x_0,x_1,\ldots,x_{t+1}$ in increasing order:\\
	\hspace*{1.5cm} $\point_0^{t+1} = 0 < \point_1^{t+1}< \dots < \point_{t+1}^{t+1} = 1$\,.
        \item Define the boxes areas for $k \in \{1,\dots,t+1\}$ by \\
        \hspace*{1.5cm}$a_k^{t+1} := \mu((\point_{k-1}^{t+1},\point_k^{t+1})) (f(\point_k^{t+1})-f(\point_{k-1}^{t+1}))$\,.
        \item Update the certificate
        \[
            \xi_{t+1} = \frac{1}{2}\sqrt{\sum_{k=1}^{t+1} (a_k^{t+1})^2} \,.
        \]
        \item Let $t \leftarrow t+1$.
\end{enumerate}}
Set $\tau_\epsilon = t$ and let $S_{\tau_\epsilon} = \big\{k \in \{1,\dots,\tau_\epsilon\} \ \text{ s.t. } \mu((\point_{k-1}^{\tau_\epsilon}, \point_k^{\tau_\epsilon}))>0 \big\}$\;
\For{$k \in S_{\tau_\epsilon}$}{
    Sample $\Point_{k}$ according to $\mu$ conditionally to the interval $(\point_{k-1}^{\tau_\epsilon}, \point_k^{\tau_\epsilon})$ and evaluate $f$ at $\Point_k$.}
Approximate $I(f) = \int_0^1 f(x)d\mu(x)$ with the estimator: \vspace*{-10pt}
	\[
	    \hat{I}_{\tau_\epsilon}(f) := \sum_{k \in S_{\tau_\epsilon}}  \mu((\point_{k-1}^{\tau_\epsilon}, \point_k^{\tau_\epsilon})) f(\Point_k) + \sum_{k=0}^{\tau_\epsilon}  \mu(\{b_k^{\tau_\epsilon}\}) f(b_k^{\tau_\epsilon})  \,.
	\]
{\bfseries Output:} $\big(\tau_\epsilon, (x_t)_{1\leq t\leq \tau_\epsilon},\hat{I}_{\tau_\epsilon}(f)\big)$.
\caption{$\SGreedyBox $.}
\label{alg:SGreedyBox}
\end{algorithm}

The following lemma shows that the error of $\SGreedyBox$ (Algorithm~\ref{alg:SGreedyBox}) at stopping time is indeed at most $\epsilon$.

\medskip
\begin{lemma} \label{lem:SGreedyBoxError}
Let $\epsilon >0$ and let $f:[0,1]\to [0,1]$ be non-decreasing. Then, the output of Algorithm~\ref{alg:SGreedyBox} satisfies
$$
	\E\Big[\big|\hat I_{\tau_\epsilon}(f)-I(f) \big|\Big] \leq \xi_{\tau_\epsilon} := \frac{1}{2}\sqrt{\sum_{k=1}^{\tau_\epsilon}(a_k^{\tau_\epsilon})^2 }  \leq \epsilon \,.
$$
\end{lemma}

\begin{proof}
First, we remark that the stopping time $\tau_\epsilon$ and the ordered sequence $b_0^{\tau_\epsilon},\dots,b_{\tau_\epsilon}^{\tau_\epsilon}$ are deterministic and do not depend on the randomness of the algorithm. These deterministic points being set, the random variables $\Point_k$ for $k \in S_{\tau_\epsilon}$ are distributed according to $\mu$ conditionally to each of the deterministic intervals defined by the $b_k^{\tau_\epsilon}$'s and satisfy for all $k \in S_{\tau_\epsilon}$ \vspace*{-5pt}
\[
	\Point_k \sim \mu(\cdot | (b_{k-1}^{\tau_\epsilon}, b_{k}^{\tau_\epsilon})) \quad \text{and} \quad
	\E[f(\Point_k)] = \frac{1}{\mu((b_{k-1}^{\tau_\epsilon}, b_{k}^{\tau_\epsilon}))}\int_{\point_{k-1}^{\tau_\epsilon}}^{\point_k^{\tau_\epsilon}} f(x) d\mu(x) \,.
\]
Therefore,
\begin{align*}
    \E\big[\hat I_{\tau_\epsilon}(f)\big] 
        & = \sum_{k \in S_{\tau_\epsilon}}  \mu((b_{k-1}^{\tau_\epsilon}, b_{k}^{\tau_\epsilon})) \E\big[f(\Point_k)\big] + \sum_{k=0}^{\tau_\epsilon}  \mu(\{b_k^{\tau_\epsilon}\}) f(b_k^t) \\
        & =  \sum_{k=1}^{\tau_\epsilon} \int_{(\point_{k-1}^t,\point_k^t)} f(x) d\mu(x)  + \sum_{k=0}^{\tau_\epsilon} \mu(\{b_k^t\}) f(b_k^t)  \\
        & = \int_0^1 f(x) d\mu(x) = I(f)
\end{align*}
and
\begin{align}
	\E\Big[\big|\hat I_{\tau_\epsilon}(f) - I(f)\big|\Big]^2
		& \stackrel{\text{Jensen}}{\leq} \E\Big[\big(\hat I_{\tau_\epsilon}(f) - I(f)\big)^2\Big] = \,\mathrm{Var}\big(\hat I_{\tau_\epsilon}(f)\big) \nonumber \\
		&  \stackrel{{\text{Independence}}}{=} \sum_{k \in S_{\tau_\epsilon}}  \mu((b_{k-1}^{\tau_\epsilon}, b_{k}^{\tau_\epsilon}))^2 \ \mathrm{Var}\big(f(\Point_k)\big)\,. \label{eq:stoch1}
\end{align}
But since $\Point_k \in (\point_{k-1}^{\tau_\epsilon},\point_k^{\tau_\epsilon})$, by monotonicity of $f$, $f(\Point_{k})$ takes values into the interval $[f(\point_{k-1}^{\tau_\epsilon}),f(\point_k^{\tau_\epsilon})]$. Thus,
\[
	\mathrm{Var}\big(f(\Point_k)\big) \leq \frac{1}{4} \big(f(\point_k^{\tau_\epsilon})-f(\point_{k-1}^{\tau_\epsilon})\big)^2 = \frac{(a_k^{\tau_\epsilon})^2}{4\mu((b_{k-1}^{\tau_\epsilon}, b_{k}^{\tau_\epsilon}))^2}
\]
by definition of $a_k^{\tau_\epsilon}$ (see step 3 of Algorithm~\ref{alg:SGreedyBox}). Therefore, substituting into Inequality~\eqref{eq:stoch1} and taking the square root, we get
\[
	\E\Big[\big|\hat I_{\tau_\epsilon}(f) - I(f)\big|\Big] \leq \frac{1}{2}\sqrt{\sum_{k=1}^{\tau_\epsilon}(a_k^{\tau_\epsilon})^2 } \,,
\]
which is smaller than $\epsilon$ by the stopping criterion. 
\end{proof}

Remark that if all boxes at time $\tau_\epsilon$ have similar areas $a_k^{\tau_\epsilon} \approx a$  (which the algorithm aims at getting by splitting only largest boxes), Lemma~\ref{lem:SGreedyBoxError} is asking $a$ to be at most of order $\cO(\epsilon/\sqrt{\tau_\epsilon})$ in Algorithm~\ref{alg:SGreedyBox}, while Algorithm~\ref{alg:GreedyBox} required $\cO(\epsilon/{\tau_\epsilon})$. Therefore, this leads to an earlier stopping criterion and smaller sample complexity. Typically, after $t$ rounds, the approximation error of Algorithm~\ref{alg:SGreedyBox} is better than the one of Algorithm~\ref{alg:GreedyBox} by a factor $\sqrt{t}$. This is however not so simple to formulate in terms of sample complexity. We will thus only formulate the analog of Theorem~\ref{thm:GreedyBoxUB} under the following assumption.

\medskip
\begin{assumption} \label{ass:polyf}
Let $\epsilon > 0$. There exist $C>0$ and $0 < \alpha < 1$ such that
	 $\cN(f,\epsilon_1) \leq C \epsilon_1^{-\alpha}$
for all $\epsilon_1 \geq \epsilon$.
\end{assumption}

It is worth to notice that Assumption~\ref{ass:polyf} is mild since $\cN(f,\epsilon) \leq \epsilon^{-1}$ for any non-decreasing $f$ and $\epsilon>0$. Furthermore, the assumption is non-asymptotic in $\epsilon$ since the requirement is only for $\epsilon_1 \geq \epsilon$.

\medskip
\begin{theorem}
\label{thm:SGreedyBoxUB}
Let $f:[0,1]\to [0,1]$ be non-decreasing which satisfies Assumption~\ref{ass:polyf} for some $C,\alpha>0$. Let $\epsilon>0$, then Algorithm~\ref{alg:SGreedyBox} satisfies
$\E\big[ | \hat I_{\tau_\epsilon}(f)-I(f) | \big] \leq \epsilon$. Besides the number of function evaluations is bounded from above by
\[
	 \tau_\epsilon =  \cO\big(\log (1/\epsilon)^{3/2} \varepsilon^{-\frac{1}{1/\alpha+1/2}}\big) \,.
\]
\end{theorem}

The proof is postponed to Appendix~\ref{sub:proof_of_theorem_thm:sgreedyboxub}.
The benefit of the stochastic algorithm is thus to replace the rate $\epsilon^{-\alpha}$ obtained with the deterministic version with $\epsilon^{-\frac{1}{1/\alpha+1/2}}$. This result generalizes the one obtained by Novak \cite{Novak1992} in the case of $\alpha = 1$ (which corresponds our worst-case scenario) for a similar algorithm. 

\subsection{An intriguing result for piecewise-regular functions}
\label{sec:piecewiseregular}

For the sake of simplicity, in this section, we restrict ourselves to the \emph{Lebesgue measure}.

As seen before, $\GreedyBox{}$ has a worst-case sample-complexity of order $\epsilon^{-1}$ up to logarithmic factors.
An interesting question is how the error rate improves with regularity for non-decreasing functions, as well as how to adapt $\GreedyBox$ to achieve optimal rates for piecewise-smooth functions. More precisely, unlike the rest of the paper, in this section, we focus on analyzing the \emph{effective} $L^p$-error rate of algorithms when run on piecewise-smooth functions. We control the number of evaluations of $f$ until $\big\|\hat f_t-f\big\|_p$ falls below $\epsilon$, rather than the number $\tau_\epsilon$ of evaluations until the certificate $\xi_t$ falls below $\epsilon^p$. The first (classical) complexity quantity can be much smaller than $\tau_\epsilon$ (since the algorithm is only aware that $f$ is non-decreasing, and lacks any prior regularity knowledge). This does not contradict the lower bound of Theorem~\ref{thm:lowerbound} and reveals the effective performance of algorithms when run on simpler functions. For instance, on $C^2$ functions, the trapezoidal rule has an effective rate of order $\epsilon^{-1/2}$ (a classical result recalled in Appendix~\ref{app:C2trapeze}), but cannot guarantee $\epsilon$-accuracy before order $\epsilon^{-1}$ evaluations if only given the knowledge that the underlying function is non-decreasing.

\paragraph{Upper bound on GreedyBox effective sample complexity.}
We aim to explore an intriguing question: can we establish these guarantees for GreedyBox without making any modifications?
Moreover, the trapezoidal rule fails to adapt to piecewise-$C^2$ functions, even for simple ones such as $f(x) = \mathds{1}_{x\geq 1/3}$. On the contrary, $\GreedyBox$ adapts very well to the discontinuities: it converges exponentially fast for any piecewise-constant function.
With this in head, one could ask if $\GreedyBox$ learns the jumps quickly enough to ensure an $\epsilon$-accurate $L^p$ error within $\tilde{\mathcal{O}}(\epsilon^{-1/2})$ sample-complexity for piecewise-$C^2$ functions.\footnote{The notation $\tilde{\mathcal{O}}$ hides logarithmic factors.} The next theorem shows that the rate can indeed be improved as long as the number of $C^1$-singularities\footnote{We call $C^k$-singularity of $f$ a point $x \in [0,1]$ such that $f$ is not $C^k$ on any neighborhood of $x$. A discontinuity is a $C^0$-singularity.} is at most of order $\cO(\epsilon^{-1})$. It should be noted that the number of $C^1$-singularities may explode to infinity as $\epsilon$ approaches zero and that the number of $C^2$-singularities does not affect the upper bound. 

\medskip
\begin{theorem}
  \label{thm:c2upper_bound}
  Let $\alpha >0$ and $\epsilon \in (0,1]$. Let $f: [0, 1] \rightarrow [0, 1]$ be a non-decreasing and piecewise-$C^2$ function with a number of $C^1$-singularities bounded by $\epsilon^{-\alpha}$ and such that $|f''(x)| \leq 1$ whenever it is defined. Then, there exists
  $$
  t_\epsilon = \left\{ \begin{array}{ll}  
  \tilde \cO\Big(\epsilon^{-1+\frac{1}{2p+2}}\Big) & \text{ if } \alpha \leq \frac{1}{2} \\
  \tilde \cO\Big(\epsilon^{-1+\big(\frac{1-\alpha}{1+p}\big)_+}\Big) & \text{ if } \alpha \geq \frac{1}{2} \\
  \end{array}\right.
  $$
  such that $\big\|\hat f_t-f\big\|_p\leq \epsilon$ for all $t\geq t_\epsilon$, where $\hat{f}_{t}$ is the approximation of $f$ returned by $\GreedyBox$ after $t$ rounds.
\end{theorem}

Theorem~\ref{thm:c2upper_bound} shows that, for piecewise-$C^2$ functions with $\alpha <1$, $\GreedyBox$ achieves $\epsilon$-accuracy in $o(\epsilon^{-1})$ function evaluations which improves the worst-case guarantee of Theorem~\ref{thm:GreedyBoxUB}. In particular, when the number of singularities is finite, then $\alpha \to 0$ when $\epsilon \to 0$, and the $L^1$-error is asymptotically of order $\cO(\epsilon^{-3/4})$. Note that this result contrasts with what happens for piecewise-$C^2$ functions without the non-decreasing assumption considered by~\cite{plaskota2008power}, who showed that as soon as there are strictly more than one discontinuity, any algorithm has a worst-case $L^p$-error of order $\Omega(\epsilon^{-p})$.

\paragraph{Minimax lower bound for approximating non-decreasing piecewise-smooth functions.}
Interestingly, we now show that this result is optimal (up to logarithmic factors) among deterministic algorithms in the regime $\alpha \geq 1/2$ (Proposition~\ref{prop:counterexample_t_new}), that is when the number of $C^1$-singularities is at least of order $\Omega(\epsilon^{-1/2})$. In the other regime, which corresponds to more regular functions, we also show that our upper bound on GreedyBox cannot be improved (Proposition~\ref{prop:counterexample_t}). 

\medskip
\begin{prop}
  \label{prop:counterexample_t_new}
    Let $p\geq 1$,  $\epsilon \in (0,1)$ and $\alpha >0$. Then, for any deterministic adaptive algorithm $\cA$ and for any 
    $$
        t < (2\epsilon)^{-1 + \big(\frac{1-\alpha}{1+p}\big)_+}-1\,,
    $$
    there exists a non-decreasing piecewise-affine function $f:[0,1]\to[0,1]$ with at most $\max\{2,\lceil \epsilon^{-\alpha}\rceil\}$ discontinuities, such that $\|f - \hat f_t\|_p > \epsilon$.
\end{prop}

Remarkably, the above lower bound demonstrates a clear difference between regular functions and piecewise-regular functions, even when the number of pieces is finite. Specifically, when considering the case of $p=1$, the above lower bound shows that, for piecewise-$C^\infty$ function with a constant number of discontinuities ($\alpha = 0$), surpassing the bound of $\Omega(\epsilon^{-1/2})$ is not achievable. This demonstrates the influence of singularities, as $C^k$ functions with no singularities can be approximated at a rate $\epsilon^{-1/k}$. It is also worth pointing out that the above lower bound maybe easily extended to non-adaptive algorithms (considering Heaviside step adversarial functions) which would require $\Omega(\epsilon^{-1})$ function evaluations. This underscores, in a new scenario, the need of adaptive algorithms to approximate functions with singularities \cite{plaskota2008power}.

\paragraph{Negative result for GreedyBox.}
The previous result shows that GreedyBox is (up to logs) optimal for highly non-regular functions ($\alpha \geq 1/2$). 
We now consider the other regime ($\alpha < 1/2$) and prove an almost (up to logs) matching lower bound in the case $p=1$, $\alpha = 0$. This, shows that the rate $\epsilon^{-3/4}$ cannot be improved for GreedyBox for such classes of functions.

\medskip
\begin{prop}
  \label{prop:counterexample_t}
  Let $\epsilon \in (0,1/12)$. Then, there exists a piecewise-$C^2$ function $f_\epsilon:[0,1]\to[0,1]$ with $|f''(x)| \leq 1$  whenever it is defined and one $C^1$-singularity, such that there exists $t \geq 2^{-7} \epsilon^{-3/4}$ with
  $
    \big\|\hat f_t - f_\epsilon \big\|_1 > \epsilon
  $,
  where $\hat f_t$ is the $\GreedyBox$ approximation after $t$ rounds.
\end{prop}

\begin{figure}[!t]
  \begin{subfigure}{0.5\textwidth}
    \includegraphics[height = 3.3cm]{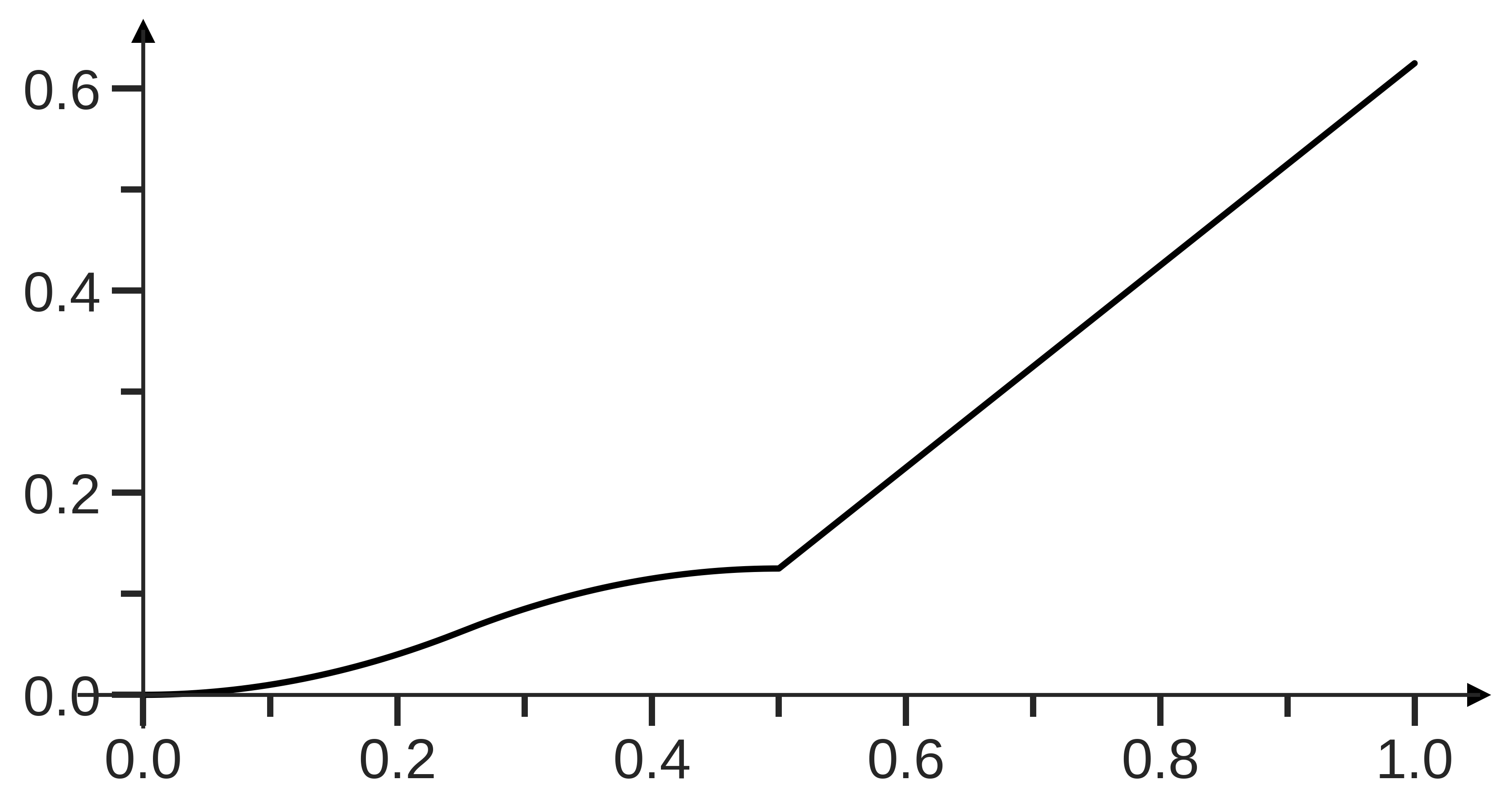}
    \caption{Worst-case function $f_\epsilon$.}
    \label{fig:f6a}
  \end{subfigure}
  \begin{subfigure}{0.5\textwidth}
     \includegraphics[height=3.3cm]{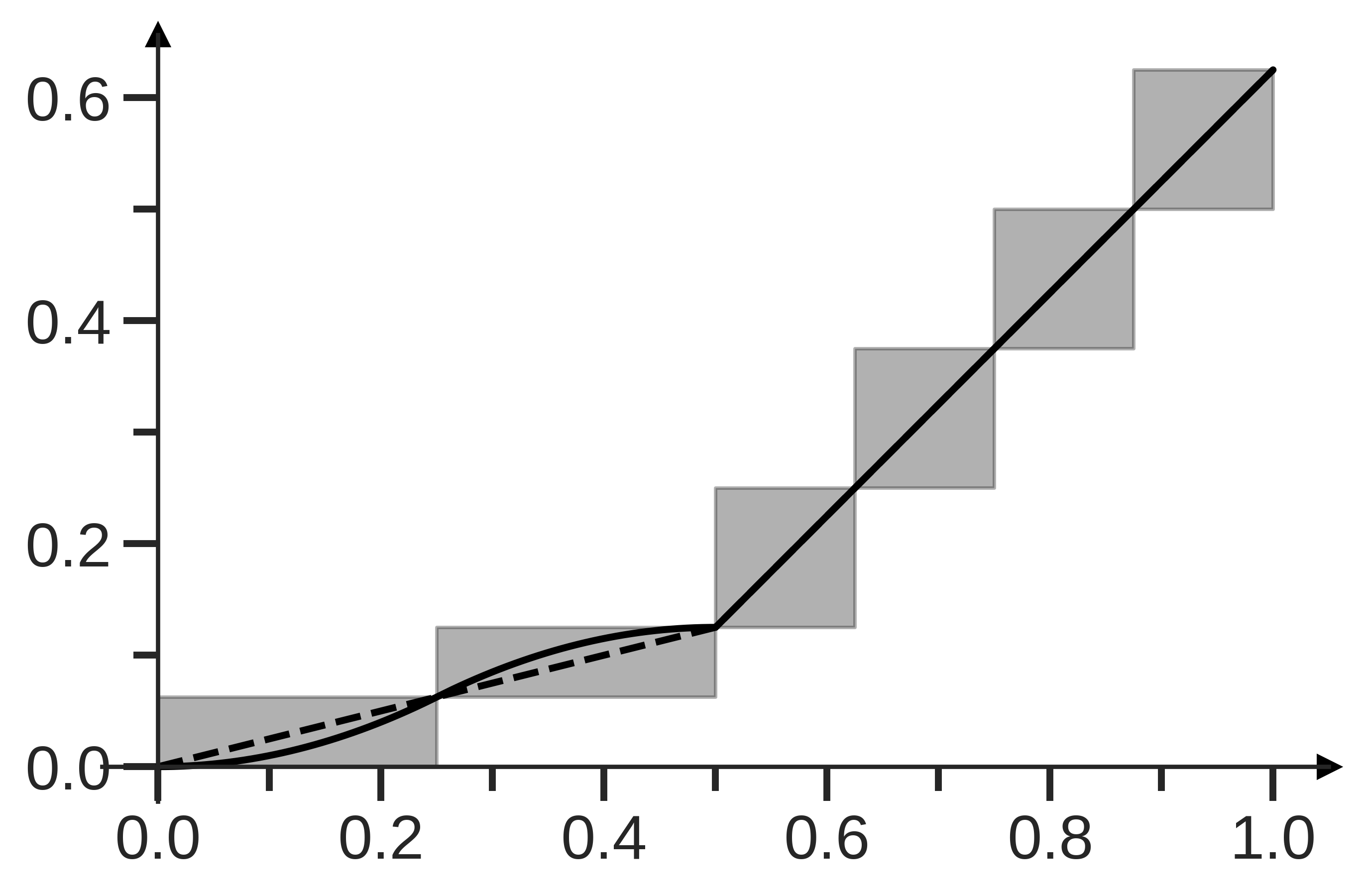}
    \caption{Box-cover produced by GreedyBox.}
    \label{fig:behavior_f6}
  \end{subfigure}\\
  \caption{Plot of the worst-case function $f_\epsilon$ for $\epsilon = 0.05$ and box-cover produced by GreedyBox after $t=6$ iterations.}
  \label{fig:f6}
\end{figure}

An example of the worst-case function $f_\epsilon$ built in the proof of Proposition~\ref{prop:counterexample_t} for $\epsilon = 0.005$, that exhibits poor performance for GreedyBox at $t = 6$, is depicted in Figure~\ref{fig:f6}. The function is formally defined in the proof and consists of two parts: one that oscillates with $|f''(t)|=1$ for $x \leq 1/2$, and the other that is linear for $x > 1/2$. The function is constructed in such a way that at a certain $t$ ($t=6$ in Fig.~\ref{fig:f6}), GreedyBox selects its points precisely between the oscillations, and focuses too much on the linear part, resulting in a maximum possible $L^1$ error.
It is worth noting that for each value of $\epsilon$, it is possible to construct an adversarial function $f_\epsilon$. However, an interesting question arises: can a single function $f$ be devised to work uniformly for all values of $\epsilon$? We believe that this question is challenging and connected to the 10\textsuperscript{th} open problem raised in \cite[Chapter 10]{brass2011quadrature}.

\paragraph{GreedyWidthBox: an optimal modification of GreedyBox.}
GreedyBox can actually be adapted to achieve the optimal rate of $\smash{\tilde \cO\big(\epsilon^{-1 + (\frac{1-\alpha}{1+p})_+}\big)}$ simultaneously for all $\alpha \geq 0$, while maintaining our adaptive guarantee in terms of $\cN_p(f,\epsilon)$. This can be accomplished by employing the GreedyBox approach for half of the iterations and the trapezoidal rule for the remaining half (see Algorithm~\ref{alg:greedywidthbox}). The proof, left to the reader, follows closely the one of Theorem~\ref{thm:c2upper_bound}, in which the upper bound~\eqref{eq:integral} can be simplified by utilizing the fact that, after conducting $t$ function evaluations, the trapezoidal rule ensures that all widths $w_i$ are at most of the order $t^{-1}$. In particular, for $p=1$ and $\alpha = 0$, this provides an algorithm that achieves $\epsilon$-acccuracy in $L^1$ norm in $\tilde \cO(\epsilon^{-1/2})$ function evaluations for non-decreasing piecewise-$C^2$ functions, as soon as the number of $C^1$-singularities remains constant as $\epsilon \to 0$.

\begin{algorithm}[!t]
  {\bfseries Input:} $\epsilon >0$, $p\geq 1$\\
  {\bfseries Init:} Set $t=1$, $x_0=b_0^1=0$ and $x_1=b_1^1=1$, evaluate $f(0)$ and $f(1)$, and define $\xi_1 = (a_1^1)^p = (f(1)-f(0))^p$\;
  \While{$\xi_t > \epsilon^p$}{
    \begin{enumerate}[topsep=0pt,parsep=0pt,itemsep=0pt,leftmargin=7pt]
    \item\eIf{$t$ is even}{
        Select the box with the largest width: find $k_*^t$ that maximizes $(b_k^t-b_{k-1}^t)$.}{Select the box with the largest area: find $k_*^t$ that maximizes $a_k^t$.}
    \item Evaluate $f$ at the midpoint $x_{t+1} := \bigl(b_{k_*^t-1}^t+b_{k_*^t}^t\bigr)/2$.
    \item Sort the points $x_0,x_1,\ldots,x_{t+1}$ in increasing order: \\
    \hspace*{1.5cm}$b_0^{t+1} = 0 \leq b_1^{t+1}< \dots < b_{t+1}^{t+1} = 1$.
    \item Define the generalized areas for all $k \in \{1,\dots,t+1\}$ by \\
    \hspace*{1.5cm} $a_k^{t+1} = (b_k^{t+1}-b_{k-1}^{t+1})^{1/p} (f(b_k^{t+1})-f(b_{k-1}^{t+1}))$.
    \item Update the certificate \vspace*{-5pt}
    \[
        \xi_{t+1} = \sum_{k=1}^{t+1} (a_k^{t+1})^p \,.
    \]
    \item Let $t \leftarrow t+1$.
    \end{enumerate}}
   Set ${\tau_\epsilon} = t$ and approximate $f$ with the piecewise-affine function  $\hat f_{\tau_\epsilon}$ defined by: \vspace*{-5pt}
      \[
        x \in [0,1] \qquad \hat f_{\tau_\epsilon}(x) = \frac{f(b_k^{\tau_\epsilon}) - f(b_{k-1}^{\tau_\epsilon})}{b_k^{\tau_\epsilon} - b_{k-1}^{\tau_\epsilon}} (x - b_{k-1}^{\tau_\epsilon}) + f(b_{k-1}^{\tau_\epsilon}) \,,
      \]
for $k \in \{1,\dots,{\tau_\epsilon}\}$ such that $b_{k-1}^{\tau_\epsilon} \leq x \leq b_{k}^{\tau_\epsilon}$\;
{\bfseries Output:} $(\tau_\epsilon, (x_t)_{1\leq t\leq \tau_\epsilon}, \hat f_{\tau_\epsilon})$ \,.
  \caption{$\GreedyWB$}
  \label{alg:greedywidthbox}
\end{algorithm}

\section{Numerical Experiments}
\label{sec:experiments}

In this section, we study empirically the performance of $\GreedyBox$ as compared to the trapezoidal rule.
All the experiments are run using $p=1$ and the Lebesgue measure for $\mu$. 
Figure~\ref{fig:comparison} displays the output given by both $\GreedyBox$ and the trapezoidal rule on the function $f$ defined by $f(x) =\smash{\frac{1}{2} x^{\nicefrac{3}{10}}}$ if $x \leq \frac{2}{3}$ and $f(x) = x$ otherwise.
One can see that on this example the $L^1$ error of $\GreedyBox$ is twice as small as that of the trapezoidal rule.
In general, $\GreedyBox$ copes far better with discontinuities than the trapezoidal rule.

\begin{figure}[ht]
  \begin{subfigure}{0.5\textwidth}
    \includegraphics[width = \textwidth]{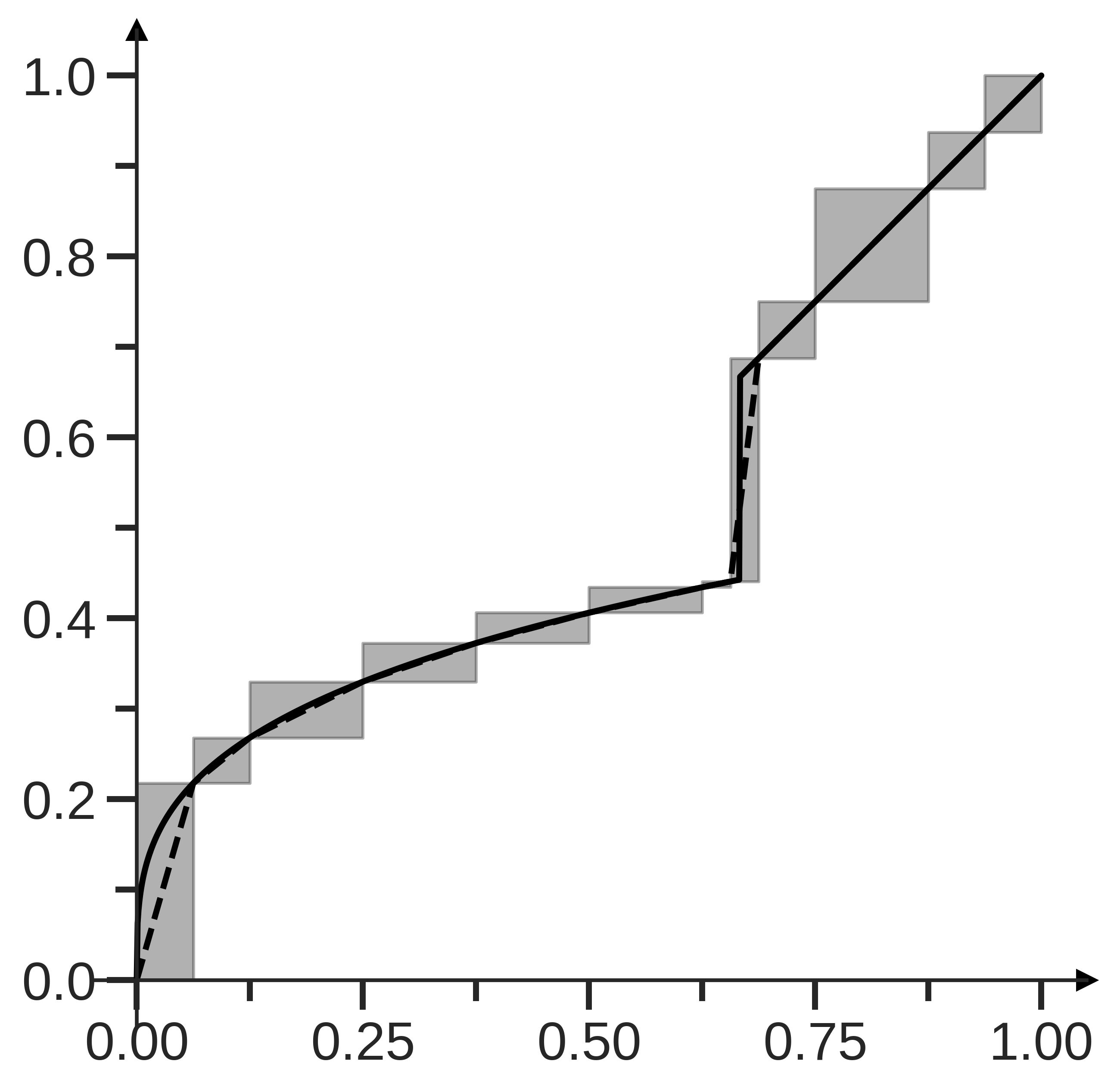}
    \caption{Error of $\GreedyBox$ after 12 iterations: 0.005}
  \end{subfigure}
  \begin{subfigure}{0.5\textwidth}
    \includegraphics[width = \textwidth]{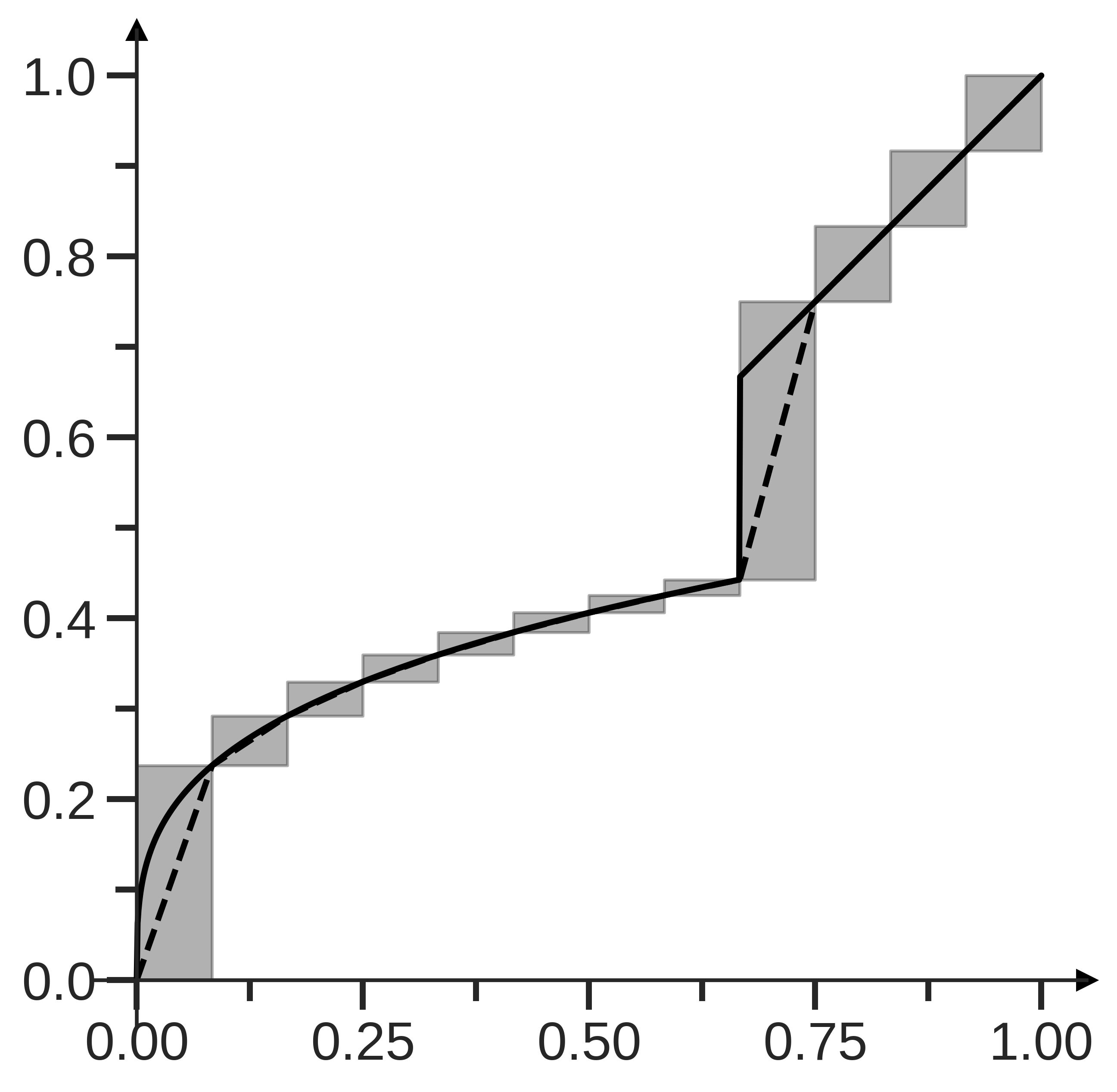}
    \caption{Error of the trapezoidal rule after 12 iterations: 0.015}
  \end{subfigure}
  \caption{Comparison of $\GreedyBox$ and the trapezoidal rule on a piecewise-$C^2$ function after 12 iterations}
  \label{fig:comparison}
\end{figure}

Remark that the trapezoidal rule is usually an offline algorithm that needs the total number $t$ of iterations from the beginning.
Fortunately, it can easily be adapted to an online version built on the same model as $\GreedyBox$.
Instead of choosing the box with the largest area on Step~\ref{line:GB_selection} of $\GreedyBox$, it picks the box with the largest width.
This online version matches exactly the offline trapezoidal rule whenever $t$ is a power of $2$ and allows for a better comparison with $\GreedyBox$.
It is this online version that we use in the next experiments.

The trapezoidal rule is known to have an $L^1$ error that decreases linearly with the number $t$ of iterations for any non-decreasing functions.
This is the same speed of convergence that we proved for $\GreedyBox$ in Theorem~\ref{thm:GreedyBoxUB}.
However, for $C^2$ functions, the $L^1$ error of the trapezoidal rule decreases quadratically with the number of iterations, which corresponds to a sample complexity $\epsilon^{-1/2}$.
This is better than the upper bound in $\epsilon^{-3/4}$ proven in Theorem~\ref{thm:c2upper_bound} for $\GreedyBox$. 
Remember however than no lower bound of order $\epsilon^{-3/4}$ was proved so far for $\GreedyBox$ on $C^2$ functions with no singularities.
Also note that the trapezoidal rule achieves a rate of $\epsilon^{-1/2}$ only for $C^2$ functions, but can have an error of order $\epsilon^{-1}$ as soon as the function is discontinuous (see Figure~\ref{fig:nearly1_rate}).

\begin{figure}[!t]
  \subcaptionbox{Error rate on $f\colon x \mapsto x^2$\label{fig:square_rate}}{\includegraphics[width=.5\linewidth]{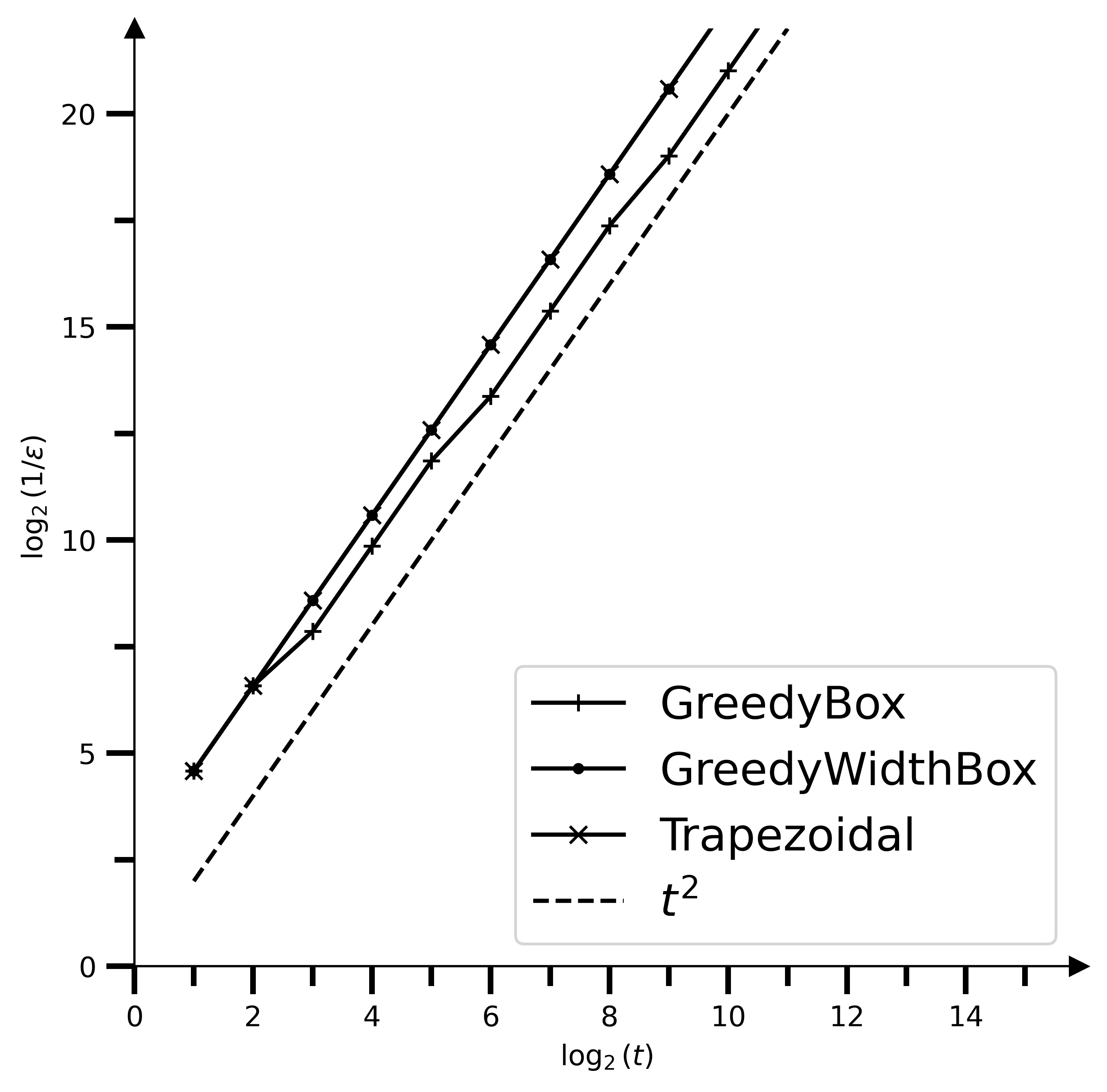}}\hfill
  \subcaptionbox{Error rate on $f\colon x \mapsto x^{1/10}$\label{fig:power_rate}}{\includegraphics[width=.5\linewidth]{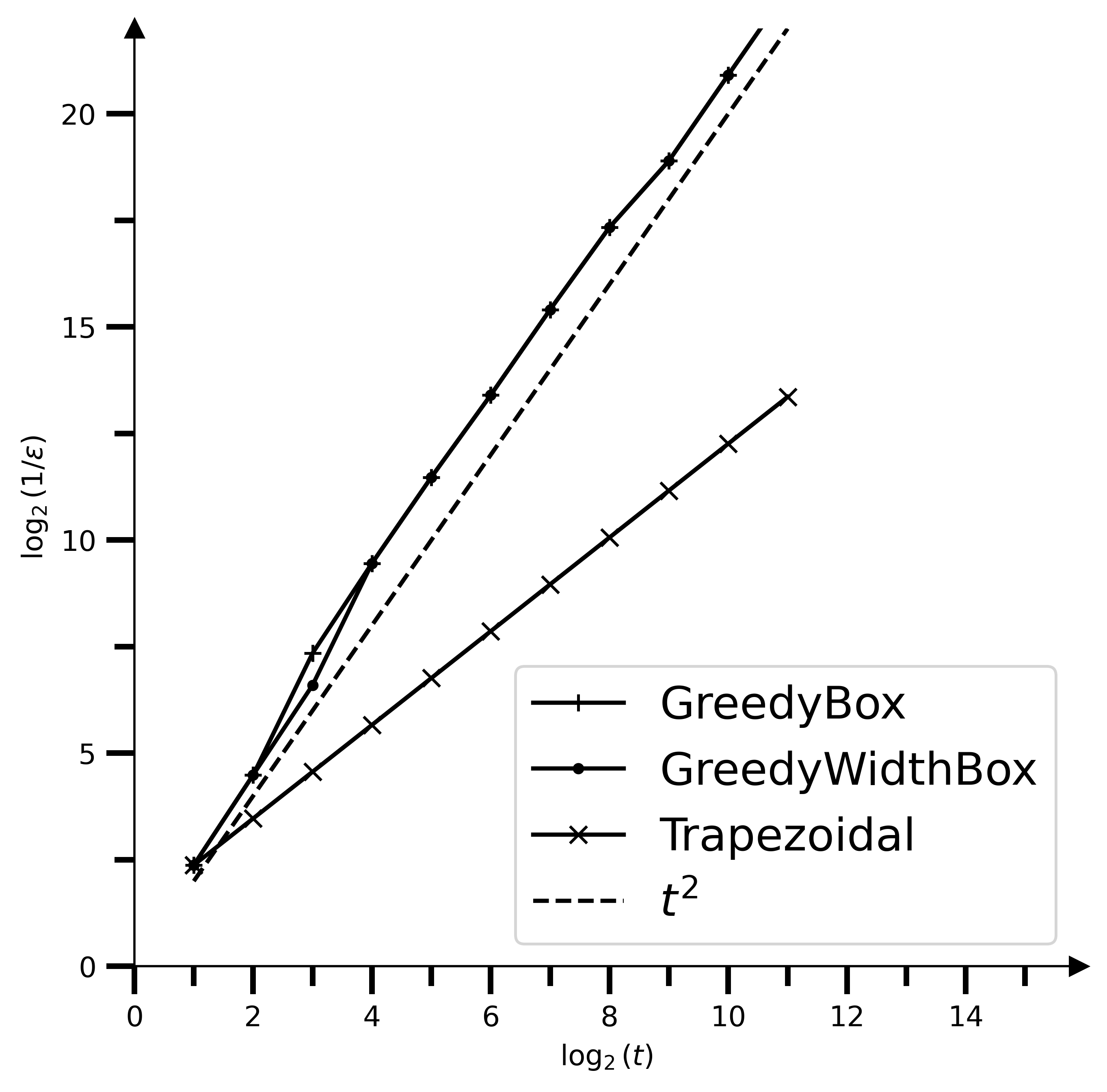}}\hfill\\
  \subcaptionbox{Error rate on $f(x)=\mathds{1}_{\{x\geq 0.3\}}$\label{fig:nearly1_rate}}{\includegraphics[width=.5\linewidth]{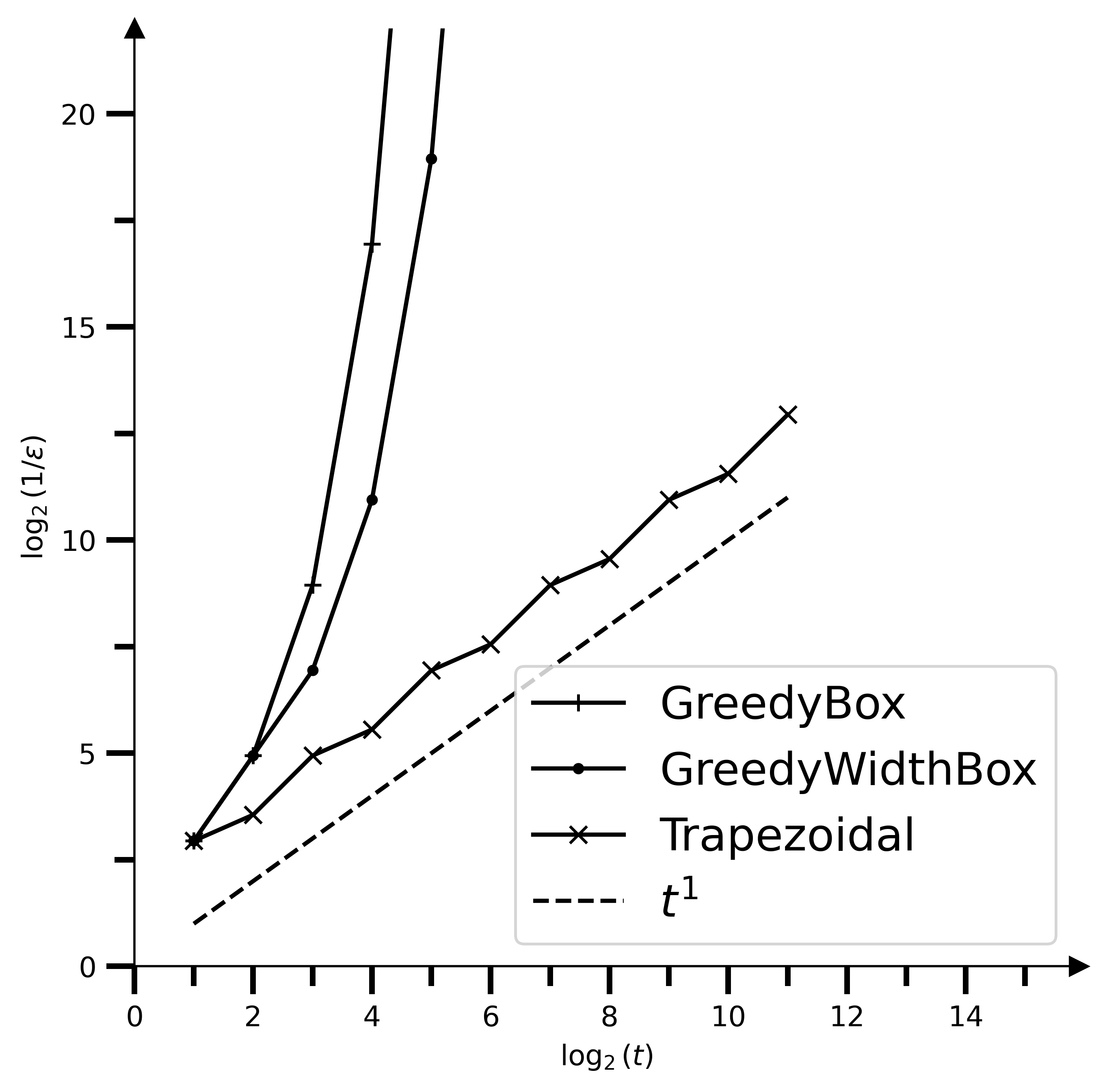}}\hfill
  \subcaptionbox{Error rate on $g^t$ \label{fig:pathological_rate}}{\includegraphics[width=.5\linewidth]{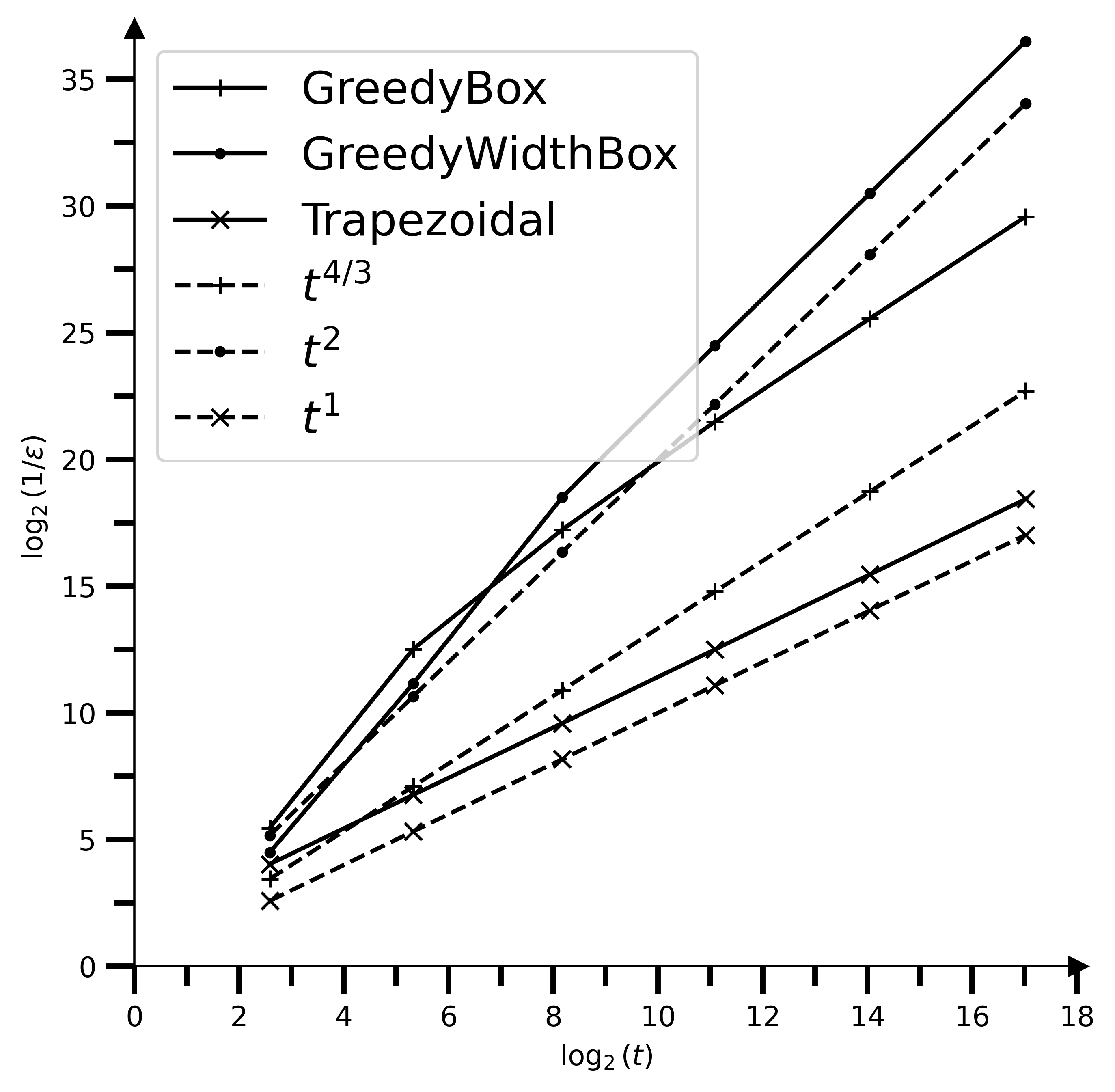}}\hfill
  \caption{Comparison of $\GreedyBox$ and $\GreedyWB$ with the trapezoidal rule.
  Logarithmic scale of the inverse of the error w.r.t. the number of evaluations.}
  \label{fig:plot_rates}
\end{figure}

In order to have best-of-both-worlds theoretical results, we introduced a new algorithm called $\GreedyWB$ (Algorithm~\ref{alg:greedywidthbox}), which achieves the asymptotic rate $\cO(\epsilon^{-1/2})$ for piecewise-$C^2$ functions, as soon as the number of pieces is finite.
In Figure~\ref{fig:plot_rates}, we run the three algorithms for a fixed number $t$ of epochs, and observe the $L^1$ distance between the approximated function $\hat{f}_t$ returned by the algorithm and the true function. Figures~\ref{fig:square_rate}-\ref{fig:nearly1_rate} consider three different functions $f$ with various regularities.   In Figure~\ref{fig:pathological_rate}, we examine time-dependent piecewise-$C^2$ functions $g^t:x \mapsto \frac{1}{2} f^t(2x) \indic_{{x\leq 1/2}} + \indic_{{x >1/2}}$. Here, $f^t$ corresponds to the worst-case function as defined in~\eqref{eq:bad_function} in the proof of Proposition~\ref{prop:counterexample_t}, with an additional introduced discontinuity.
For a better understanding of the results, we plot the inverse of this error, and plot everything with logarithmic scale.
This means that a straight line with slope 1 represents a linear speed of convergence: an error that decreases inversely proportionally with the number $t$ of epochs. Figure~\ref{fig:pathological_rate} precisely confirms the anticipated worst-case rates as determined by the analysis for monotone piecewise-$C^2$ functions.

Remember however that given a desired precision $\epsilon$, $\GreedyBox$ does not stop when its $L^1$ error is smaller than $\epsilon$, but when its certificate (the best upper bound it can get without knowing \textit{a priori} the function) is smaller than $\epsilon$.
Thus, for computation comparison, what really matters is to plot the convergence of the certificate with regard to some target error $\epsilon$. The certificate of GreedyBox appears to be smaller than both GreedyWidthBox and the trapezoidal rule, regardless of the smoothness of the function.
Yet, for most of the existing functions, the certificate of $\GreedyBox$ is of the order of $\epsilon^{-1}$.
However, remark that the same bound apply both for the trapezoidal rule and for $\GreedyWB$. 
This behavior can be seen on Figure~\ref{fig:certificates} for GreedyBox and the trapezoidal rule.
In Figure~\ref{fig:certif_pathological_rate}, we can see that GreedyBox provides certification of $\epsilon$-accuracy approximately 30 times as fast as the trapezoidal rule. 
For $C^2$ functions, this certificate will never be a good upper bound of the real error, that often is of order $\epsilon^{-1/2}$ for all three algorithms.

\begin{figure}[!t]
  \subcaptionbox{Certificate and error on $f\colon x \mapsto x^2$\label{fig:certif_square_rate}}{\includegraphics[width=.5\linewidth]{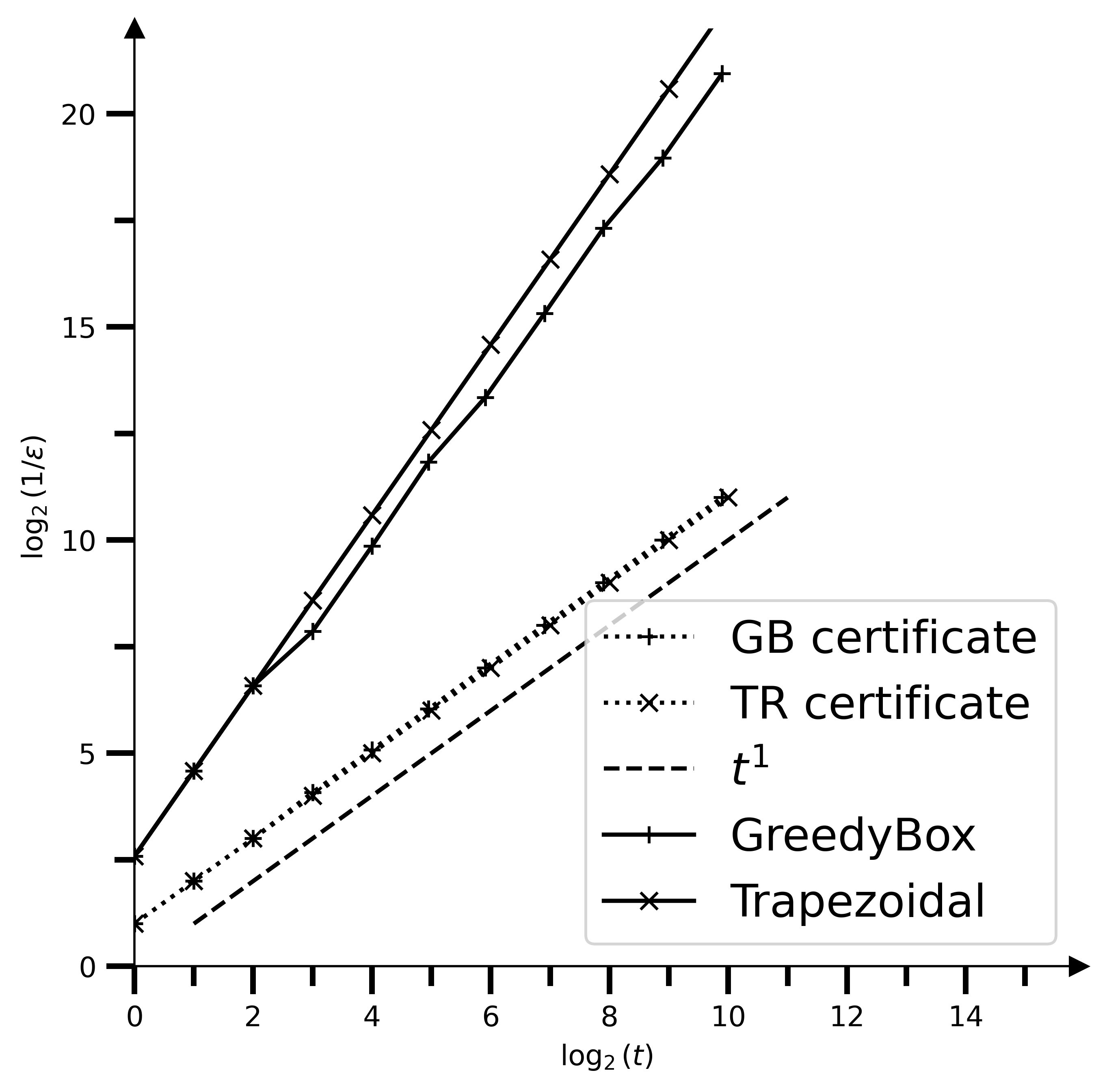}}\hfill
  \subcaptionbox{Certificate and error on $f\colon x \mapsto x^{1/10}$\label{fig:certif_power_rate}}{\includegraphics[width=.5\linewidth]{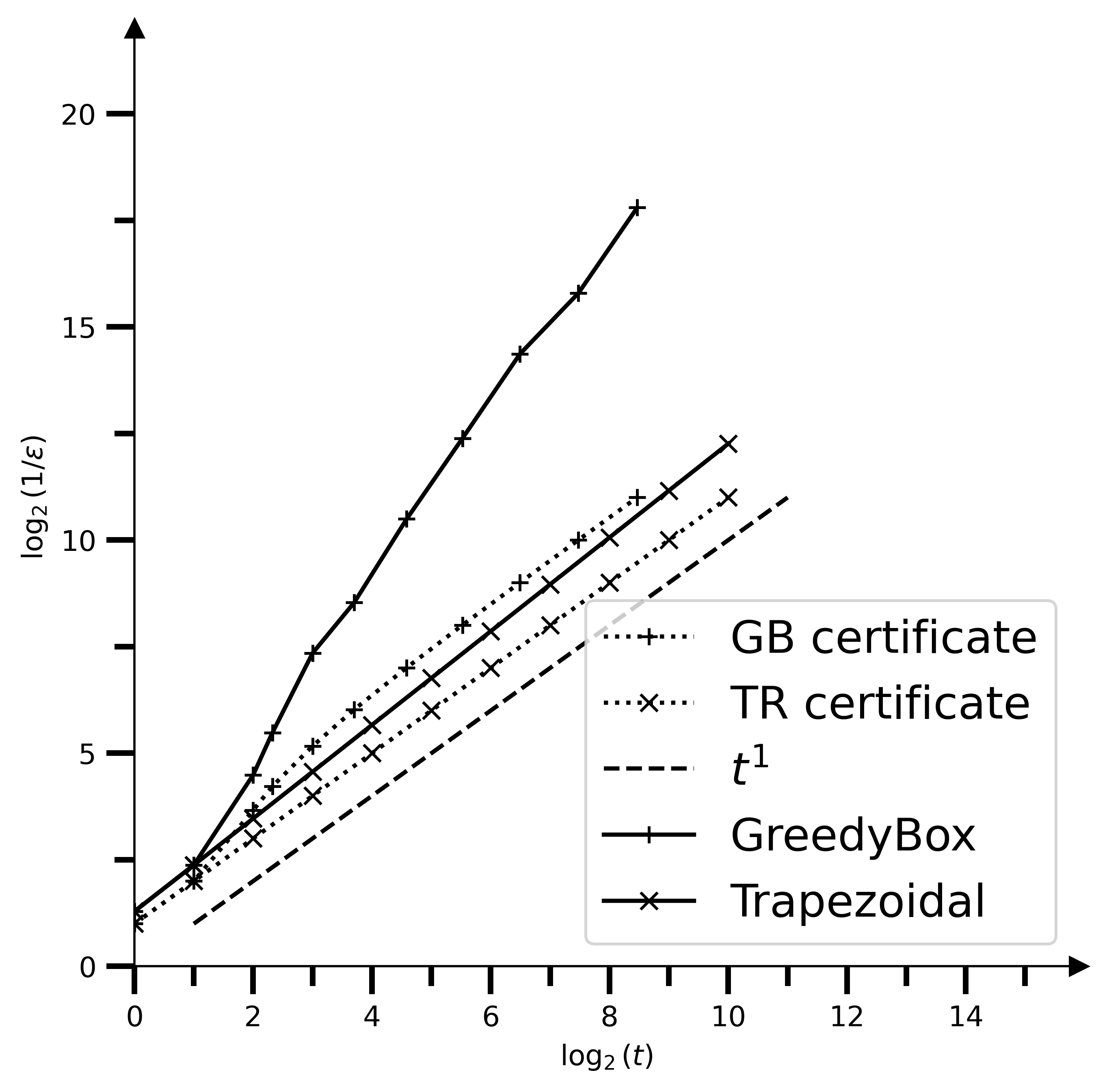}}\hfill\\
  \subcaptionbox{Certificate and error on $f(x)=\mathds{1}_{\{x\geq 0.3\}}$\label{fig:certif_nearly1_rate}}{\includegraphics[width=.5\linewidth]{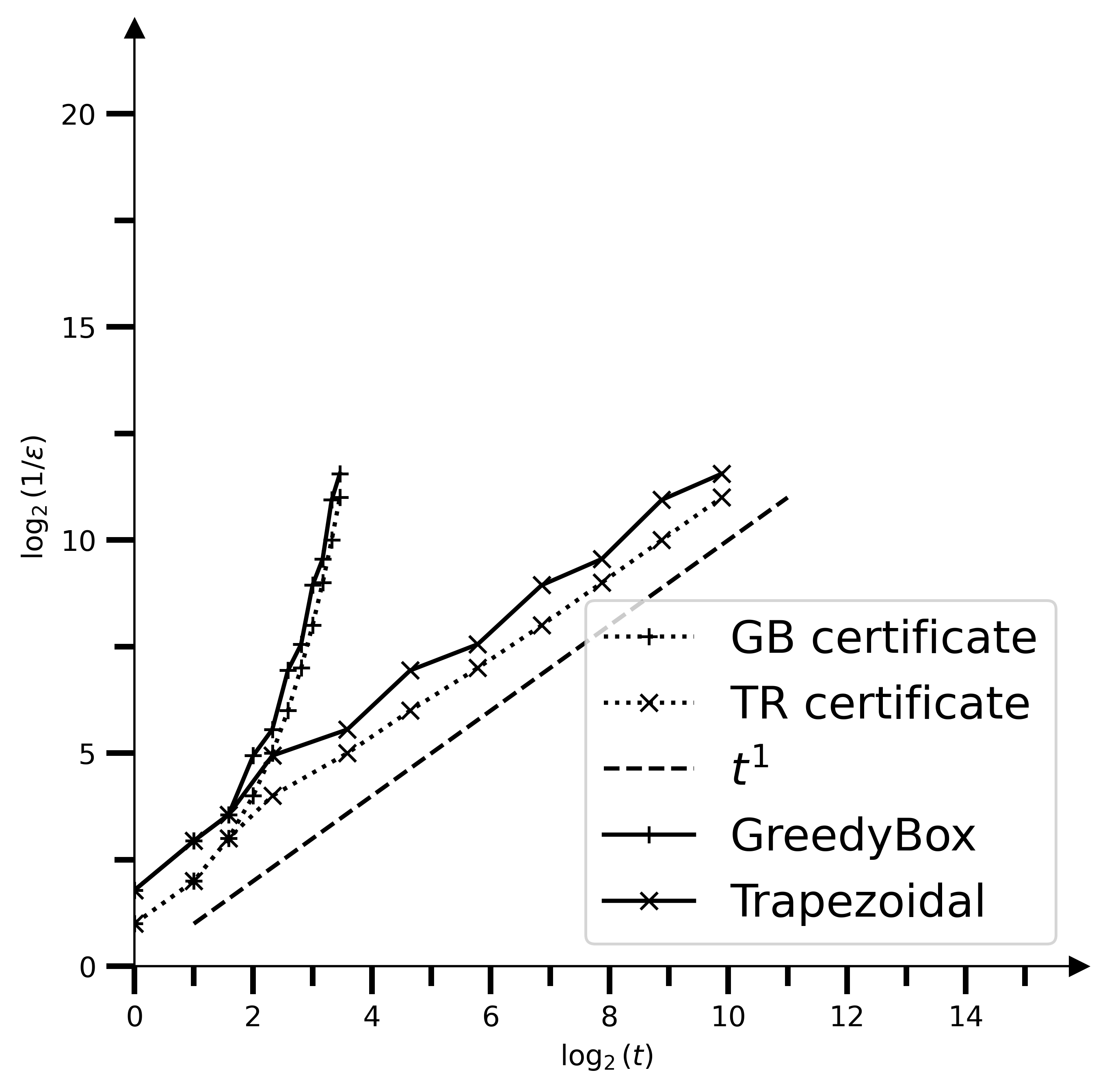}}\hfill
  \subcaptionbox{Certificate and error on $g^t$ \label{fig:certif_pathological_rate}}{\includegraphics[width=.5\linewidth]{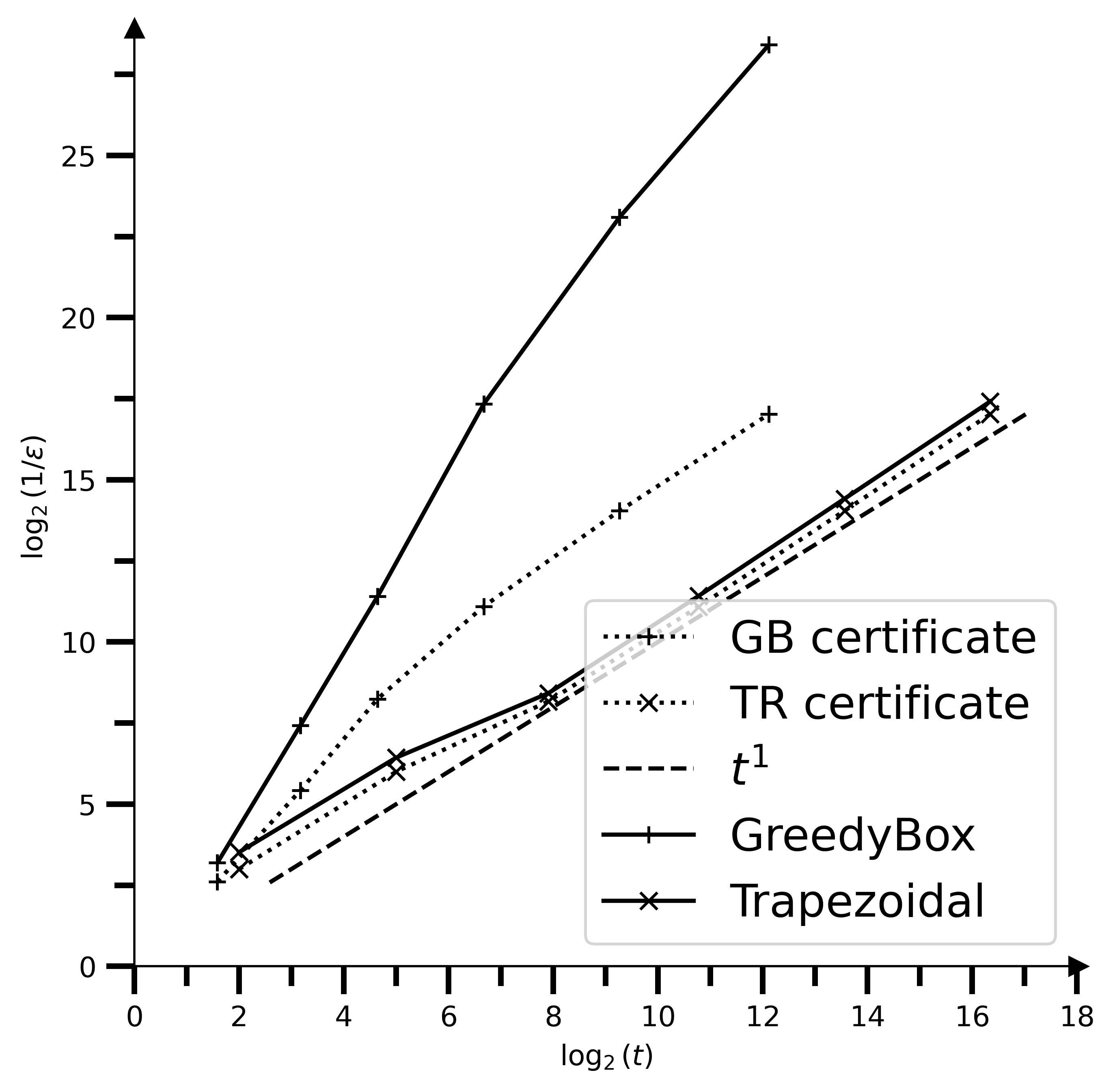}}\hfill
  \caption{Comparison of $\GreedyBox$ with the trapezoidal rule.
  Logarithmic scale of the inverse of the error w.r.t the number of evaluations.}
  \label{fig:certificates}
\end{figure}

\section{Conclusion and future works}
\label{sec:conclusion}

In this paper, we studied the problem of approximating a non-decreasing function~$f$ in $L^p(\mu)$ norm, with sequential (adaptive) access to its values. We first proved an $f$-dependent lower bound on the stopping time that holds for all algorithms with guaranteed $L^p(\mu)$ error after stopping. We then presented the $\GreedyBox$ algorithm (inspired from Novak \cite{Novak1992}) and showed that up to logarithmic factors, it is optimal among all such algorithms, for each non-decreasing function $f$. As a direct consequence, we showed for the integral estimation problem that $\GreedyBox$ can be combined with additional randomization to get an improved rate in expectation. For the $L^p$-approximation problem, we also investigated to what extent the $L^p(\mu)$ error of $\GreedyBox$ can decrease faster (than guaranteed by the algorithm) for piecewise-$C^2$ functions. Put briefly, up to logarithmic factors, $\GreedyBox$ automatically achieves the improved (and optimal) rate of $\smash{\epsilon^{-1+(\frac{1-\alpha}{1+p})_+}}$ for a large number $\epsilon^{-\alpha}$ of singularities, $\alpha \geq 1/2$, and a simple algorithmic variant ($\GreedyWB$) achieves this rate for any value of $\alpha$. In particular, our results highlight multiple performance gaps between adaptive and non-adaptive algorithms, $C^k$ and piecewise-$C^k$ functions, as well as monotone or non-monotone functions. We also provided numerical experiments to illustrate our theoretical results.

Several interesting questions about $\GreedyBox$ are left open. First, similarly to the faster rates proved for piecewise-$C^2$ functions, it would be interesting to investigate optimal rates for the average $L^p(\mu)$ error, when $f$ is drawn at random from a probability distribution over the set of monotone functions (as, e.g., in \citet{Novak1992}). Second, for any fixed non-decreasing function $f$, it would be useful to derive the limit (if any) of the empirical distribution of the points queried by $\GreedyBox$. When $p=1$ and $\mu$ is the Lebesgue measure on $[0,1]$, we conjecture that this limit exists and has a density roughly proportional to the square-root of the derivative of $f$ almost everywhere. Solving this problem would help complete our understanding of the behavior of $\GreedyBox$, and of how it precisely adapts to the complexity of any function $f$.

Several generalizations would also be worth investigating in the future. A seemingly straightforward direction is to work with Lipschitz functions on compact intervals; an algorithm defined similarly to $\GreedyBox$ but with parallelograms instead of rectangular boxes seems perfectly fit for this problem. Other natural directions consist in extending our $f$-dependent bounds to multivariate monotone functions (in the spirit of, e.g., \citet{papageorgiou1993integration}), to functions of known bounded variation, or variants of these function classes (e.g., entirely monotone functions, functions of bounded Hardy-Krause variation, see \cite{fang21-multivariateExtensionsIsotonicTotalVariation}).

Another natural and interesting research avenue is to address the case of noisy evaluations of $f$. We believe that, under some known assumptions on the noise distribution, similar sample complexity guarantees (with slower rates) can be achieved by using a mini-batch variant of $\GreedyBox$, which reduces the impact of noise by sampling multiple points and computing the average. Solving this problem would be a way to efficiently estimate cumulative distribution functions under a special censored feedback in the same spirit as in \citet{abernethy16-ThresholdBandits}. A variant of GreedyBox has been designed and analysed by~\cite{bonnet2020adaptive} for addressing imperfect observations. However, the authors have solely established asymptotic convergence garantees (both point-wise and in terms of $L^1$ or $L^\infty$ norms) when the number of function evaluations approaches infinity, and they do not offer any guarantees regarding sample complexity.

\section{Acknowledgements}
The authors would like to thank Wouter Koolen for insightful discussions at the beginning of the work, as well as Peter Bartlett for fruitful feedback about improved rates for piecewise-regular functions.
This work has benefited from the AI Interdisciplinary Institute ANITI, which is funded by the French ``Investing for the Future – PIA3'' program under the Grant agreement ANR-19-P3IA-0004. S\'ebastien Gerchinovitz and \'Etienne de Montbrun gratefully acknowledge the support of IRT Saint Exupéry and the DEEL project.\footnote{\url{https://www.deel.ai/}}

\newpage
\begin{appendices}

\section{Elementary properties of box-covers}
\label{sec:behaviorOfN}

\subsection{A general upper bound (proof of Lemma~\ref{lem:N-basic})}

\begin{customlemma}{\ref{lem:N-basic}}
For all non-decreasing functions $f:[0,1] \to [0,1]$ and $\epsilon > 0$, the quantity $\cN_p(f,\epsilon)$ is well defined and upper bounded by
\begin{equation*}
\cN_p(f,\epsilon) \leq \lceil 1/\epsilon \rceil \;.
\end{equation*}
\end{customlemma}

\begin{proof}
Let $n =  \left\lceil 1/\epsilon \right\rceil \geq 1$.
In order to prove the lemma, we exhibit a sequence $\mathcal{B}$ of at most $n$ boxes, show that it is a box-cover of $f$, and that the generalized areas of the boxes $B$ are such that $\left( \sum_B \cA_p (B)^p \right)^{1/p} \leq \epsilon$.

For $i \in \{1, \ldots, n\}$ let $x_i = \sup\{ 0 \leq x \leq 1 \colon f(x) \leq \frac{i}{n}\}$.
We also set $x_0 = 0$. Note that $x_n = 1$.
For $i \in \{1, \ldots, n\}$, we define $B_i = [x_{i-1}, x_i] \times \left[ \frac{i-1}{n}, \frac{i}{n} \right]$. Without loss of generality, we assume that $x_{i-1} < x_i$ for all $i \in \{1, \ldots, n\}$. (Otherwise, we just remove all $B_i$'s such that $x_{i-1}=x_i$, and remark below that the $B_i$'s still cover the graph of $f$ except maybe at the $x_i$'s.) 

We start by showing that $\mathcal{B} = (B_i)_{1\leq i \leq n}$ is a box-cover of  $f$. First, $B_1, \ldots, B_n$ are adjacent boxes by construction.
Note also that the graph of $f$ is included in the union of the $B_i$'s except maybe at the $x_i$'s. Indeed, for all $i \in \{1, \ldots, n\}$ and all $x \in (x_{i-1}, x_i)$, we have $\frac{i-1}{n} \leq f(x) \leq \frac{i}{n}$ by definition of $x_{i-1}$ and $x_i$, and by monotonicity of $f$. Therefore, $(x, f(x)) \in B_i$. This proves that $\mathcal{B}$ is a box-cover of $f$.

We are left to show that $\left( \sum_{i=1}^n \cA_p (B_i)^p \right)^{1/p} \leq \epsilon$. We have
\[ \sum_{i=1}^n \cA_p (B_i)^p = \sum_{i=1}^n \left( \frac{i}{n} - \frac{i-1}{n} \right)^p \mu((x_{i-1}, x_i)) = \frac{1}{n^p}\sum_{i=1}^n \mu((x_{i-1}, x_i)) \leq \frac{1}{n^p} \;.\]
This entails that $\left( \sum_{i=1}^n \cA_p (B_i)^p \right)^{1/p} \leq \frac{1}{n} \leq \epsilon$, which concludes the proof.
\end{proof}

\subsection{Technical Lemmas on box-covers}
\label{sec:elementaryproperties}

We state here two elementary lemmas about box-covers. The first one below indicates that box-covers of $f$ with desired properties exist not only for $n = \cN_p(f,\epsilon)$ (or $n = \cN'_p(f,\epsilon)$) but also for all larger values of $n$.

\medskip
\begin{lemma}
\label{lem:existence-larger-n}
Let $\epsilon > 0$ and $f:[0,1]\to [0,1]$ be any non-decreasing function. Then,
\begin{itemize}
    \item for all $n \geq \cN_p(f,\epsilon)$, there exists  a box-cover $B_1,\dots,B_n$ of $f$ such that $\bigl(\sum_{i=1}^n \cA_p(B_i)^p\bigr)^{1/p} \leq \epsilon$;
    \item for all $n \geq \cN'_p(f,\epsilon)$, there exists  a box-cover $B_1,\dots,B_n$ of $f$ such that $\cA_p(B_i) \leq \epsilon$ for all $i=1,\ldots,n$.
\end{itemize}
\end{lemma}

\begin{proof}
The proof of the two items is straightforward. For the first one, consider any box-cover $B_1,\dots,B_N$ of $f$, with $N=\cN_p(f,\epsilon)$, such that $\bigl(\sum_{i=1}^N \cA_p(B_i)^p\bigr)^{1/p} \leq \epsilon$. Then, split the last box $B_N$ vertically into $n-N+1$ sub-boxes $B'_N,\ldots,B'_n$. We can immediately see that
\[
\left(\sum_{i=1}^{N-1} \cA_p(B_i)^p + \sum_{i=N}^n \cA_p(B'_i)^p\right)^{1/p} \leq \left(\sum_{i=1}^n \cA_p(B_i)^p\right)^{1/p} \leq \epsilon\;,
\]
so that the sequence of boxes $B_1,\dots,B_{N-1},B'_N,\ldots,B'_n$ is a box-cover of $f$ with cardinality $n$ that satisfies the desired generalized area property. The second item can be proved similarly.
\end{proof}

The next simple lemma indicates that the values $y_j^-$ and $y_j^+$ in any box-cover of $f$ are necessary lower and upper bounds on the values of $f$ inside the boxes.

\medskip
\begin{lemma}
\label{lem:yrangeofboxes}
Let $f:[0,1]\to [0,1]$ be non-decreasing, and $B_1,\ldots,B_n$ be any box-cover of $f$, with $B_j = [c_{j-1},c_j] \times [y_j^-,y_j^+]$ for all $j=1,\ldots,n$,
for some sequence $0=c_0 < \ldots <  c_n = 1$.
Then, for all $j=1,\ldots,n$,
\[
y^-_j \leq \inf_{x>c_{j-1}} f(x) \qquad \textrm{and} \qquad y^+_j \geq \sup_{x<c_j} f(x) \;.
\]
\end{lemma}

\begin{proof}
Let $x \in (c_{j-1},c_j)$. Since $(x,f(x)) \in \cup_{j'=1}^n B_{j'}$ and $x \notin [c_{k-1},c_k]$ for any $k \neq j$, we have $(x,f(x)) \in B_j$ and therefore
\[
y^-_j \leq f(x) \qquad \textrm{and} \qquad y^+_j \geq f(x) \;.
\]
We conclude by taking the infimum or supremum over all $x \in (c_{j-1},c_j)$ and by using the fact that $f$ is non-decreasing.
\end{proof}

\subsection{Relationship between \texorpdfstring{$\cN_p(f,\epsilon)$}{Np(f, eps)} and \texorpdfstring{$\cN'_p(f,\epsilon)$}{N'p(f, eps)}}
\label{sec:twocomplexitynotions}

In this section, we compare the two complexity quantities $\cN_p(f,\epsilon)$ and $\cN'_p(f,\epsilon)$ defined in Section~\ref{sec:definitions}.

The quantity $\cN_p'(f,\epsilon)$ always satisfies $\cN_p'(f,\epsilon) \leq \cN_p(f,\epsilon)$. The next lemma relates the two complexity notions in a tighter way. Intuitively, if all the boxes $B$ of a minimal box-cover of $f$ satisfying $\sum_B \cA_p(B)^p \leq \epsilon^p$ had similar generalized areas $\cA_p(B)$, these areas would be close to $\epsilon/\cN_p(f,\epsilon)^{1/p}$, so that $$\cN_p'\big(f,\epsilon/\cN_p(f,\epsilon)^{1/p}\big) \lesssim \cN_p(f,\epsilon).$$ The following lemma implies that indeed $$\cN_p'\big(f,\epsilon/\cN_p(f,\epsilon)
^{1/p}\big) \approx \cN_p(f,\epsilon)$$ up to a factor of~$2$.

\medskip
\begin{lemma}
    \label{lem:N}
    Let $\epsilon > 0$ and $n \geq \cN_p(f,\epsilon)$. Then,
    \[
        \cN_p'\Big(f,\frac{\epsilon}{n^{1/p}}\Big) \leq 2 n \qquad \text{and} \qquad \cN_p(f,\epsilon) \leq \cN_p'\Big(f,\frac{\epsilon}{\cN_p(f,\epsilon)^{1/p}}\Big) \,.
    \]
\end{lemma}

In the proof of Theorem~\ref{thm:GreedyBoxUB} we only use the first inequality $\cN_p'\big(f,\frac{\epsilon}{n^{1/p}}\big) \leq 2 n$. The second inequality shows that this step is tight up to a constant of $2$.

\begin{proof}
Let $f:[0,1]\to [0,1]$ be a non-decreasing function and $\epsilon >0$. 
\begin{enumerate}[label={(\alph*)}]
\item We prove that $\cN_p'(f,\epsilon/n^{1/p}) \leq  2n$. For the sake of readability, we assume that $\mu$ is the Lebesgue measure on $[0,1]$, and will explain at the end how to adapt the proof for an arbitrary probability measure $\mu$. By $n \geq \cN_p(f,\epsilon)$ and Lemma~\ref{lem:existence-larger-n}, there exists a box-cover $B_1,\dots,B_n$ of $f$ satisfying $\sum_{i=1}^n \cA_p(B_i)^p \leq \epsilon^p$. Following a technique from \citet[Proof of Theorem~3]{Novak1992}, we now divide the $B_i$'s into as many sub-boxes as necessary so that each of their generalized areas is below $\epsilon/n^{1/p}$. We will show that the overall number of resulting boxes is at most $2n$.

  Let $i \in \{1,\dots,n\}$. We split the box $B_i = [x^-, x^+] \times [y^-,y^+]$ in a vertical fashion by splitting $[x^-,x^+]$ into $k_i = \big\lfloor {n \cA_p(B_i)^p}/{\epsilon^p}\big\rfloor + 1$ intervals of equal sizes $[x_0,x_1],\dots,[x_{k-1},x_{k_i}]$, where $x_j = x^- + j(x^+-x^+)/{k_i}$. The choice of $k_i$ ensures that $\cA_p(B_i)^p \leq k_i \epsilon^p / n$.

  Then, we define the smaller boxes $B_i^{(j)} := [x_{j-1},x_{j}] \times [y^-, y^+]$ for $j=1,\dots,k_i$. Their generalized areas are given by $\cA_p(B_i^{(j)}) = ((x^+-x^-)/k_i)^{1/p}(y^+-y^-)  = \cA_p(B_i)/k_i^{1/p} \leq \epsilon/n^{1/p}$ as required. Furthermore, the new total number of boxes is
    \[
        \sum_{i=1}^{n} k_i = \sum_{i=1}^{n} \left(\left\lfloor \frac{n \cA_p(B_i)^p}{\epsilon^p}\right\rfloor + 1\right) \leq n + \frac{n}{\epsilon^p} \sum_{i=1}^{n} \cA_p(B_i)^p \leq 2n \,.
    \]
    Since $\cup_{i=1}^{n} \cup_{j=1}^{k_i} B_i^{(j)} = \cup_{i=1}^{n} B_i$ contains the graph of $f$ except maybe at the $B_i$'s endpoints (and thus everywhere outside the $B_i^{(j)}$'s endpoints), we have shown that $\cN_p'\big(f,\epsilon/n^{1/p}\big) \leq 2n$, as desired.

    When $\mu$ is not the Lebesgue measure on $[0,1]$, the proof can be slightly adapted as follows. For each box $B_i = [x^-, x^+] \times [y^-,y^+]$ as above, we distinguish two cases. \\
    \textit{Case 1:} if $\mu\bigl((x^-, x^+)\bigr) = 0$, we set $k_i = 1$, $x_0 = x^-$, and $x_1 = x^+$ as above.\\
    \textit{Case 2:} if $\mu\bigl((x^-, x^+)\bigr) > 0$, we set $k_i = \big\lfloor {n \cA_p(B_i)^p}/{\epsilon^p}\big\rfloor + 1$,  $x_0 = x^-$, and $x_{k_i} = x^+$ as above. However, for any $1 \leq j < k_i$ (if any), we take $x_j$ as a quantile of order $j/k_i$ of the conditional distribution $\mu\bigl(\cdot | (x^-, x^+)\bigr)$. Then, among all sub-boxes $B_i^{(j)} := [x_{j-1},x_{j}] \times [y^-, y^+]$, $j=1,\ldots,k_i$, we only keep the non-degenerate ones, i.e., those such that $x_{j-1} < x_j$. 
    
    The rest of the proof remains unchanged. In particular, each remaining sub-box satisfies $\cA_p(B_i^{(j)}) \leq \epsilon/n^{1/p}$, and the total number of sub-boxes is at most of $\sum_{i=1}^n k_i \leq 2 n$.

    \item Before proving the second inequality, we prove an intermediate result: for all $n_1 \geq \cN_p'(f,\epsilon)$, we have
      \begin{equation}
        \cN_p(f,n_1^{1/p} \epsilon) \leq n_1 \;.
        \label{eq:b}
     \end{equation}
    Since $n_1 \geq \cN_p'(f,\epsilon)$, we can consider a box-cover $B_1,\dots,B_{n_1}$ of $f$ with generalized areas at most of $\epsilon$ each. The result immediately follows from
    \[
        \Big(\sum_{i=1}^{n_1} \cA_p(B_i)^p\Big)^{1/p} \leq n_1^{1/p} \epsilon \,.
    \]

    \item We write $n_\epsilon = \cN_p(f,\epsilon)$ for simplicity. We prove that $n_\epsilon \leq \cN_p'(f,\epsilon/n_\epsilon^{1/p})$ by contradiction. Let us assume that $\smash{\cN_p'(f,\epsilon/n_\epsilon^{1/p}) <  n_\epsilon}$ and define
    \[
        \epsilon' := \epsilon \cdot \bigg(\frac{\cN_p'\big(f,\epsilon/n_\epsilon^{1/p}\big)}{n_\epsilon}\bigg)^{1/p} < \epsilon \,.
    \]
    Then using Inequality \eqref{eq:b}, with $\epsilon$ substituted with $\smash{\epsilon/n_\epsilon
   ^{1/p}}$, we get
    \[
         \cN_p(f,\epsilon)
         \stackrel{\epsilon' < \epsilon}{\leq}
         \cN_p(f,\epsilon') =
         \cN_p\bigg(f,\cN_p'\Big(f,\frac{\epsilon}{n_\epsilon^{1/p}}\Big)^{1/p}
         \frac{\epsilon}{n_\epsilon^{1/p}}\bigg)
         \stackrel{\eqref{eq:b}}{\leq}
         \cN_p'\Big(f,\frac{\epsilon}{n_\epsilon^{1/p}}\Big)  \,,
    \]
    which contradicts the assumption and thus proves the desired inequality.
\end{enumerate}
\end{proof}

\section{Omitted proofs}
\label{sec:omittedproofs}

\subsection{Proof of Lemma~\ref{lem:GreedyBoxError}}
\label{app:proof_lemma_greedyboxerror}

\begin{customlemma}{\ref{lem:GreedyBoxError}}
Let $f:[0,1]\to [0,1]$ be non-decreasing, $p\geq 1$ and $\epsilon \in (0,1]$. For any $t \in \{1,\ldots,\tau_\epsilon\}$, \vspace*{-5pt}
\[
  \big\|\hat f_t - f\big\|_p^p := \int_0^1  \abs{ \hat f_t(x) - f(x) }^p d\mu(x) \leq
  \sum_{k=1}^{t} (a_k^{t})^p =: \xi_t \;.
\]
\end{customlemma}

\begin{proof}
Remark that the approximation $\hat f_t$ is the continuous piecewise-affine function such that $\hat f_t(b_k^t) = f(b_k^t)$ for all $k \in \{0,\ldots,t\}$. Therefore, it suffices to show that for each $k=1,\dots, t$,
\[
  \int_{(b_{k-1}^t,b_{k}^t)} \abs{\hat f_t(x) - f(x)}^p d\mu(x) \leq (a_k^t)^p \,,
\]
which follows easily from the fact that $\abs{\hat f_t(x) - f(x)} \leq f(b_k^t) - f(b_{k-1}^t)$ on the $k$-th box (since $f$ is non-decreasing), and by definition of $a_k^t := (\mu(\point_{k-1}^t, \point_k^t))^{1/p} (f(\point_k^t)-f(\point_{k-1}^t))$. Summing over $k=1,\dots,t$ concludes the proof.
\end{proof}

As a supplement, we are left with proving the special case of the Lebesgue measure, for which the result holds with a multiplicative factor of $\nicefrac{1}{(p+1)}$.

\medskip
\begin{lemma}
\label{lem:gainFactor2}
Let $f : [0,1] \to [0,1]$ be non-decreasing. At any round $t \geq 1$,
\[
  \big\|\hat f_t - f\big\|_p^p := \int_0^1  \big( \hat f_t(x) - f(x) \big)^p d x \leq \frac{1}{1+p}\sum_{k=1}^t (a_k^t)^p \,.
\]
\end{lemma}

\begin{proof}
Summing over $k=1,\dots,t$ it suffices to show that for each $k=1,\dots, t$,
\[
  \int_{b_{k-1}^t}^{b_{k}^t} \abs{\hat f_t(x) - f(x)}^p dx \leq \frac{(a_k^t)^p }{1+p} \,.
\]
Let $k \in \{1,\dots,t\}$. To ease the notation, we make a change of variables and define for all $u \in [0,1]$:
\[
  g(u) := \frac{f\big((b_k^t - b_{k-1}^t)u+b_{k-1}^{t}\big) - f(b_{k-1}^t)}{f(b_k^t) - f(b_{k-1}^t)} \,.
\]
We have
\begin{align*}
  \int_{b_{k-1}^t}^{b_{k}^t} \abs{\hat f_t(x) - f(x)}^p dx
  & = \big(f(b_k^t) - f(b_{k-1}^t)\big)^p (b_k^t -b_{k-1}^t) \int_0^1 \abs{u - g(u)}^p du \\
  & = (a_k^t)^p \int_0^1 \abs{u - g(u)}^p du \,.
\end{align*}
The function $g$ is non-decreasing over $[0,1]$ with $g(0) = 0$ and $g(1) = 1$.
It remains to control  $\int_0^1 \abs{u - g(u)}^p du$ by $\nicefrac{1}{(1+p)}$. which is done in the following.

First we remark that we can assume $g(\nicefrac{1}{2}) \neq \nicefrac{1}{2}$. Otherwise, since $g$ is non-decreasing $|g(u)-u| \in [0,\demi]$ for all $u \in [0,1]$ and $\int_0^1 \abs{u - g(u)}^p du \leq 2^{-p} \leq (1+p)^{-1}$, which concludes.

\bigskip
We  consider the two points $u_-\leq \demi \leq u_+$ such that the sign of $g(x) - x$ does not change over $(u_-,u_+)$. More formally, they are defined as:
\[
  u_- = \inf \Big\{u \in [0,\nicefrac{1}{2}]: \quad \forall x \in (u,\nicefrac{1}{2}] \quad  \big(g(\nicefrac{1}{2}) - \nicefrac{1}{2}\big) \big(g(x) - x\big) > 0 \Big\}
\]
and
\[
  u_+ = \sup \Big\{u \in [\nicefrac{1}{2},1]: \quad \forall x \in [\nicefrac{1}{2},u) \quad  \big(g(\nicefrac{1}{2}) - \nicefrac{1}{2}\big) \big(g(x) - x\big) > 0 \Big\} \,.
\]
Note that the sign of $g(x) - x$ is constant over $\{\demi\} \cup (u_-,u_+)$ since $g(\demi) \neq \demi$.
Now, we show the following two facts
\begin{equation}
    \label{eq:facts}
    \forall u < u_-, \quad g(u) \leq u_- \qquad \text{and} \qquad \forall u > u_+, \quad g(u) \geq u_+\,.
\end{equation}
Indeed, let $u < u_-$, by definition of $u_-$, it exists $x \in [u,u_-]$ such that
\[
  \big(g(\nicefrac{1}{2}) - \nicefrac{1}{2}\big) \big(g(x) - x\big) \leq  0 \,,
\]
and thus using again the definition of $u_-$ for all $x' \in (u_-,\nicefrac{1}{2}]$,
$
\big(g(x') - x'\big) \big(g(x) - x\big) \leq  0.
$
If $g(x) \leq x$, we are done since $g(u) \leq g(x) \leq x\leq u_-$ because $g$ is non-decreasing and $u \leq x \leq u_-$. Otherwise using $u \leq x'$,  $g(u) \leq g(x') \leq x'$ and making $x' \to u_-$ concludes the first inequality of ~\eqref{eq:facts}. The second inequality can be proved similarly.

Therefore from~\eqref{eq:facts}, for all $u \leq u_-$, $|g(u) - u| \leq u_-$ and all $u \geq u_+$, $|g(u) - u| \leq 1-u_+$ which yields
\begin{equation}
  \int_0^{u_-} \abs{g(u) - u}^p du \leq u_-^{1+p}  \quad \text{and} \quad \int_{u_+}^1 \abs{g(u) - u}^p du \leq (1-u_+)^{1+p} \,.
  \label{eq:edges}
\end{equation}
Furthermore, $g(u) - u$ does not change sign over $(u_-,u_+)$, which entails:
\begin{equation}
  \int_{u_-}^{u_+} \abs{g(u)-u}^p du
  \leq \max \left\{ \int_{u_-}^{u_+} (u_+-u)^p du  , \int_{u_-}^{u_+} (u - u_-)^p du  \right\}
  = \frac{(u_+-u_-)^{p+1}}{p+1} \,. \label{eq:middle}
\end{equation}
Summing Inequalities \eqref{eq:edges} and \eqref{eq:middle}, we get
\begin{multline*}
  \int_0^1 \abs{g(u)-u}^p du \leq u_-^{1+p} + (1-u_+)^{1+p} + \frac{(u_+-u_-)^{1+p}}{1+p} \\
                              \leq \sup_{ 0 \leq u_- \leq \demi \leq u_+ \leq 1} \bigg\{ u_-^{1+p} + (1-u_+)^{1+p} + \frac{(u_+-u_-)^{1+p}}{1+p}  \bigg\} = \frac{1}{1+p}  \,.
\end{multline*}
The supremum is reached for $(u_-, u_+) = (0,1)$.
\end{proof}

\subsection{Proof of Lemma~\ref{lem:smallboxes}}

\begin{customlemma}{\ref{lem:smallboxes}}
Let $f:[0,1]\to [0,1]$ be non-decreasing, $p\geq 1$ and $\epsilon \in (0,1]$. Define $\tau_\epsilon' := 2 \bigl(1+\lceil p \log_2(1/\epsilon)\rceil\bigr) \cN'_p\bigl(f,\epsilon\bigr)$, and assume that $\GreedyBox$ is such that $\tau_\epsilon > \tau'_\epsilon$. Then, at time $\tau'_\epsilon$, all the boxes maintained by $\GreedyBox$ have a generalized area bounded from above by $\epsilon$, i.e., $\smash{a_k^{\tau_\epsilon'} \leq \epsilon}$ for all $k \in \{1,\dots,\tau_\epsilon'\}$.
\end{customlemma}

\begin{proof}
By definition of $m := \cN_p'(f,\epsilon)$, we can fix a box-cover $B_1,\ldots,B_{m}$ of $f$ such that $\cA_p(B_j) \leq \epsilon$ for all $j$. Recall that these boxes are adjacent, i.e., they are of the form $B_j = [c_{j-1},c_j] \times \bigl[y_j^-,y_j^+\bigr]$ for some nodes $0=c_0 < \ldots < c_{m}=1$. The inequalities $\cA_p(B_j) \leq \epsilon$ thus translate into $(y_j^+-y_j^-)\mu\big((c_{j-1},c_j)\big)^{1/p} \leq \epsilon$ for all $j$.

In this proof we will compare the boxes maintained by $\GreedyBox$ with the boxes $B_j$ above. We first need the following definition: at each round $t \geq 1$, we say that a box $B = \bigl[\point_{k-1}^t,\point_k^t\bigr] \times \bigl[f(\point_{k-1}^t),f(\point_k^t)\bigr]$, $1 \leq k \leq t$, maintained by $\GreedyBox$ is
\begin{itemize}
        \item \emph{internal} if $\bigl[\point_{k-1}^t,\point_k^t\bigr] \subset (c_{j-1},c_j)$ for some $j=1,\ldots,m$;
        \item \emph{overlapping} if $\bigl[\point_{k-1}^t,\point_k^t\bigr] \ni c_j$ for some $j=1,\ldots,m$.
\end{itemize}
Note that exactly one of the two cases above must hold true. In the sequel, we denote by $\hat{B}_t = \bigl[\point_{k_*^t-1}^t,\point_{k_*^t}^t\bigr] \times \bigl[f\bigl(\point_{k_*^t-1}^t\bigr),f\bigl(\point_{k_*^t}^t\bigr)\bigr]$ the box selected by $\GreedyBox$ at time $t$. Since $\hat{B}_t$ is necessary of one of the two types above, we can distinguish between the following two cases.

\medskip
\textit{Case 1}: there exists $t \in \{1,2,\ldots,\tau'_{\epsilon}\}$ such that $\hat{B}_t$ is internal. In this case, letting $j=1,\ldots,m$ be the corresponding index, we have, by Lemma~\ref{lem:yrangeofboxes} in Appendix~\ref{sec:elementaryproperties},
\[
c_{j-1} < \point_{k_*^t-1}^t < \point_{k_*^t}^t < c_j \qquad \textrm{and} \qquad y^-_j \leq f\bigl(\point_{k_*^t-1}^t\bigr) \leq f\bigl(\point_{k_*^t}^t\bigr) \leq y^+_j \;.
\]
Therefore, the generalized areas of $\hat{B}_t$ and $B_j$ satisfy $a^t_{k^t_*} = \cA_p\bigl(\hat{B}_t\bigr) \leq \cA_p(B_j) \leq \epsilon$ by construction of $B_j$. Now, by definition of $k_*^t$ in Algorithm~\ref{alg:GreedyBox}:
\[
\max_{1 \leq k \leq t} a^t_k = a^t_{k^t_*} \leq \epsilon \;.
\]
This concludes the proof of the lemma in this case, since the quantity $\max_{1 \leq k \leq t} a^t_k$ can only decrease between rounds $t$ and $\tau'_{\epsilon}$.

\medskip
\textit{Case 2}: $\hat{B}_t$ is an overlapping box at every round $t \in \{1,2,\ldots,\tau'_{\epsilon}\}$. In this case, we say that all nodes $c_j$ lying in $\bigl[\point_{k_*^t-1}^t,\point_{k_*^t}^t\bigr]$ are \emph{activated} at time~$t$.
More precisely, we say that a node $c_j$ is \emph{left-activated} when $c_j \in \bigl(\point_{k_*^t-1}^t,\point_{k_*^t}^t\bigr]$ (part of the box $\hat{B}_t$ lies on the left of $c_j$), and that it is \emph{right-activated} when $c_j \in \bigl[\point_{k_*^t-1}^t,\point_{k_*^t}^t\bigr)$ (part of the box $\hat{B}_t$ lies on the right of $c_j$).

Since $\point_{k_*^t-1}^t < \point_{k_*^t}^t$ by construction, at least one node $c_j$ is either left-activated or right-activated at every round $t \in \{1,2,\ldots,\tau'_{\epsilon}\}$, so that
\[
\sum_{t=1}^{\tau'_{\epsilon}} \underbrace{\sum_{j=1}^{m} \sum_{\sigma \in \{\lleft,\rright\}} \mathds{1}_{\{\textrm{$c_j$ is $\sigma$-activated at round $t$}\}}}_{\geq 1} \geq \tau'_{\epsilon} \;.
\]
Inverting sums and recognizing $N_{j,\sigma} = \sum_{t=1}^{\tau'_{\epsilon}} \mathds{1}_{\{\textrm{$c_j$ is $\sigma$-activated at round $t$}\}}$ to be the number of rounds when node $c_j$ is $\sigma$-activated, we can see that
\[
\sum_{j=1}^{m} \sum_{\sigma \in \{\lleft,\rright\}} N_{j,\sigma} \geq \tau'_{\epsilon} \;,
\]
so that
\[
\max_{1 \leq j \leq m} \, \max_{\sigma \in \{\lleft,\rright\}} N_{j,\sigma} \geq \frac{\tau'_{\epsilon}}{2\,m} \;.
\]
Combining the last inequality with the definition of $\tau'_{\epsilon}$ and $m := \cN_p'(f,\epsilon)$, we obtain the following intermediate result.

\begin{fact}
\label{fac:activation}
There exists $j \in \{1,\ldots,\cN_p'(f,\epsilon)\}$ and $\sigma \in \{\lleft,\rright\}$ such that the node $c_j$ is $\sigma$-activated at least
\[
\frac{\tau'_{\epsilon}}{2\, \cN_p'(f,\epsilon)} =  1 + \big\lceil p \log_2(1/\epsilon)\big\rceil
\]
times within the set of rounds $\{1,\ldots,\tau'_{\epsilon}\}$.
\end{fact}

\medskip
We now focus on a single pair $(j,\sigma)$ provided by Fact~\ref{fac:activation} above. We follow the evolution of the box $\tilde{B}_t$ maintained by $\GreedyBox$ that lies on the $\sigma$-side of $c_j$. More formally, if we write $\tilde{B}_t = \bigl[\point_{k-1}^t,\point_{k}^t\bigr] \times \bigl[f\bigl(\point_{k-1}^t\bigr),f\bigl(\point_{k}^t\bigr)\bigr]$, this means that $c_j \in \bigl(\point_{k-1}^t,\point_{k}^t\bigr]$ if $\sigma = \lleft$, and that $c_j \in \bigl[\point_{k-1}^t,\point_{k}^t\bigr)$ if $\sigma = \rright$. Note that, at any round $t$, such a box $\tilde{B}_t$ indeed exists and is unique.

Note that, at all rounds $t \in \{1,\ldots,\tau'_{\epsilon}\}$ when $c_j$ is $\sigma$-activated, we have $\tilde{B}_t=\hat{B}_t$, so that the box $\tilde{B}_t$ is replaced (see Step~\ref{step:GB-sort} in Algorithm~\ref{alg:GreedyBox}) by two boxes whose generalized widths are at most half that of $\tilde{B}_t$. This is because $x_{t+1}$ is a median of the conditional distribution $\mu(\cdot |(\point_{k_*^t-1}^t,\point_{k_*^t}^t))$. Since $\tilde{B}_{t+1}$ is among these two boxes, we thus have
\[
\width\bigl(\tilde{B}_{t+1}\bigr) \leq \width\bigl(\tilde{B}_t\bigr)/2
\]
at each round $t \in \{1,\ldots,\tau'_{\epsilon}\}$ when $c_j$ is $\sigma$-activated. Note also that $\tilde{B}_{t+1} = \tilde{B}_t$ at all other rounds $t \in \{1,\ldots,\tau'_{\epsilon}\}$; at such rounds, $\width\bigl(\tilde{B}_{t+1}\bigr) = \width\bigl(\tilde{B}_t\bigr)$. Combining the last two properties, and denoting by $\tau$ the round when $c_j$ is $\sigma$-activated for the $\lceil p \log_2(1/\epsilon)\rceil$-th time, we get
\[
\width\bigl(\tilde{B}_{\tau+1}\bigr) \leq 2^{-\lceil p \log_2(1/\epsilon)\rceil} \leq \epsilon^p \,.
\]
To conclude, denote by $\tau'$ the round when $c_j$ is $\sigma$-activated for the $\big(\lceil p\log_2(1/\epsilon)\rceil+1\big)$-th time. Note that $\hat{B}_{\tau'} = \tilde{B}_{\tau'}$ so that
\[
\width\bigl(\hat{B}_{\tau'}\bigr) = \width\bigl(\tilde{B}_{\tau'}\bigr) \leq \width\bigl(\tilde{B}_{\tau+1}\bigr) \leq \epsilon^p \;,
\]
where the first inequality follows from $\tau' \geq \tau+1$ and the fact that $\width\bigl(\tilde{B}_{t}\bigr)$ is non-increasing over time. Therefore, by definition of $\hat{B}_{\tau'}$, all boxes maintained by $\GreedyBox$ at time $\tau'$ have generalized areas bounded by
\[
\max_{1 \leq k \leq \tau'} a^{\tau'}_k \leq \cA_p\bigl(\hat{B}_{\tau'}\bigr) \leq (1-0) \times \width\bigl(\hat{B}_{\tau'}\bigr)^{1/p} \leq \epsilon \;.
\]
We conclude the proof by noting that $\tau' \leq \tau'_{\epsilon}$, so that $\max_{1 \leq k \leq \tau'_{\epsilon}} a^{\tau'_{\epsilon}}_k \leq \max_{1 \leq k \leq \tau'} a^{\tau'}_k \leq \epsilon$.
\end{proof}

\subsection{Proof of Lemma~\ref{lem:areadivided}}

\begin{customlemma}{\ref{lem:areadivided}}
Let $f:[0,1]\to [0,1]$ be non-decreasing, $p\geq 1$ and $\epsilon \in (0,1]$. For any $t \in \{1,\ldots,\lfloor \tau_\epsilon/2 \rfloor\}$, we have $\xi_{2t} \leq  \xi_t / 2$. Therefore, for all $t \leq s$ in $\{1,\ldots,\tau_\epsilon\}$,
\begin{equation*}
\xi_{s} \leq \frac{\xi_{t}}{2^{\lfloor \log_2(s/t) \rfloor}} \leq \Big(\frac{2 t}{s}\Big) \, \xi_{t} \,.
\end{equation*}
\end{customlemma}

\begin{proof}
We first prove that $\xi_{2t} \leq  \xi_t / 2$ for all $t \in \{1,\ldots,\lfloor \tau_\epsilon/2 \rfloor\}$. To do so, recall that at any round $1 \leq t' \leq \tau_\epsilon$, $\GreedyBox$ maintains $t'$ boxes given by $\bigl[\point_{k-1}^{t'},\point_{k}^{t'}\bigr] \times \bigl[f\bigl(\point_{k-1}^{t'}\bigr),f\bigl(\point_{k}^{t'}\bigr)\bigr]$, $k = 1,\ldots,t'$. We write $a_{(1)}^{t'} \geq a_{(2)}^{t'} \geq \ldots \geq a_{(t')}^{t'}$ for their generalized areas $a^{t'}_k$ sorted in decreasing order.

\medskip
\textit{Part 1}: We first show by induction on $k=1,\ldots,t$ that
\begin{itemize}[nosep]
        \item[(i)] the box selected by $\GreedyBox$ at round $t-1+k$ (Step~\ref{line:GB_selection} in Algorithm~\ref{alg:GreedyBox}) has a generalized area $a^{t-1+k}_{k^{t-1+k}_*}$ larger than or equal to $a_{(k)}^t$;
        \item[(ii)] at least $t-k$ boxes of round $t+k$ are identical to boxes of round~$t$.
\end{itemize}
Both (i) and (ii) are straightforward for $k=1$. Assume they are true for some $k \in \{1,\ldots,t-1\}$. Next we prove (i) and (ii) with the index value $k+1$. At round $t+k$, the generalized area of the box selected by $\GreedyBox$ must be at least as large as the maximum generalized area of the $t-k$ boxes that are identical to boxes of round~$t$ (by (ii) with $k$). Therefore, this generalized area is larger than or equal to $a_{(k+1)}^t$, which proves~(i). Property~(ii) is immediate since only one box is selected at every round. This completes the induction.

\medskip
\textit{Part 2}: Note that, at each round $t-1+k \in \{t,t+1,\ldots,2t-1\}$, the box $\hat{B}_{t-1+k}$ selected by $\GreedyBox$ is replaced with two smaller boxes $B'$ and $B''$ whose generalized areas satisfy
\[
\cA_p(B')^p + \cA_p(B'')^p \leq \cA_p(\hat{B}_{t-1+k})^p / 2 \;,
\]
since $x_{t+1}$ is a median of $\mu(\cdot |(\point_{k_*^t-1}^t,\point_{k_*^t}^t))$ and $(\delta')^p + (\delta'')^p \leq (\delta'+\delta'')^p$ for any $\delta',\delta'' \geq~0$. Therefore, and by Property~(i) above, at least $\cA_p(\hat{B}_{t-1+k})^p / 2 \geq \bigl(a_{(k)}^t\bigr)^p/2$ is lost when summing the generalized areas to the power $p$ at round $t+k$, compared to round $t+k-1$.
Therefore, the certificate $\xi_{t+k}$ of the box-cover at the next round satisfies $\xi_{t+k}  \leq \xi_{t-1+k} - (a_{(k)}^t)^p/2$. Summing over $k=1,\ldots,t$ we get:
        \[
        \xi_{2t} \leq \xi_t - \sum_{k=1}^t \frac{(a_{(k)}^t)^p}{2} = \frac{\xi_t}{2} \,.
        \]
To see why this implies~\eqref{eq:lineardecrease}, it suffices to note that $s \geq \tilde{s} := 2^{\lfloor \log_2(s/t) \rfloor} \cdot t$, so that
\[
\xi_{s} \leq \xi_{\tilde{s}} \leq \frac{\xi_{t}}{2^{\lfloor \log_2(s/t) \rfloor}} \;,
\]
where the first inequality is because the certificate $\xi_s$ can only decrease over time, and where the last inequality follows from the property $\xi_{2t} \leq  \xi_t / 2$ shown above. This concludes the proof.
\end{proof}

\subsection{Proof of Theorem~\ref{thm:SGreedyBoxUB}} 
\label{sub:proof_of_theorem_thm:sgreedyboxub}

\begin{customthm}{\ref{thm:SGreedyBoxUB}}
Let $f:[0,1]\to [0,1]$ be non-decreasing which satisfies Assumption~\ref{ass:polyf} for some $C,\alpha>0$. Let $\epsilon>0$, then Algorithm~\ref{alg:SGreedyBox} satisfies
\[
	\E\Big[ \big| \hat I_{\tau_\epsilon}(f)-I(f) \big| \Big] \leq \epsilon.
\]
Besides the number of function evaluations is bounded from above by
\[
	 \tau_\epsilon =  \cO\big(\log (1/\epsilon)^{3/2} \varepsilon^{-\frac{1}{1/\alpha+1/2}}\big) \,.
\]
\end{customthm}

\begin{proof}
Let $\epsilon>0$. First, we remark that the points $x_0,\dots,x_t$ defined by the deterministic version $\GreedyBox$ (defined in Algorithm~\ref{alg:GreedyBox}) and the stochastic version (Algorithm~\ref{alg:SGreedyBox}) are identical. Only the stopping criterion and definition of $\hat I_{\tau_\epsilon}(f)$ differs from Algorithm~\ref{alg:GreedyBox}. Therefore, we can apply Lemma~\ref{lem:smallboxes}.

Let $\epsilon_1 \geq \epsilon$ that will be fixed later as a function of $\epsilon$.
Denote $n_{\epsilon_1} = C \epsilon_1^{-\alpha} \geq \cN(f,\epsilon_1)$ by Assumption~\ref{ass:polyf}. Applying Lemma~\ref{lem:smallboxes} to $\epsilon_2 = \epsilon_1 / n_{\epsilon_1} = \epsilon_1^{1+\alpha}/C$, we get the following result. At time
\[
        \tau'_{\epsilon_2} := \big(1+\lceil \log_2(1/\epsilon_2)\rceil\big) \cN'(f,\epsilon_2)
\]
all the boxes maintained by Algorithm~\ref{alg:SGreedyBox} have an area bounded from above by $\epsilon_2$, i.e.,
$$
        a_k^{\tau_{\epsilon_2}'} \leq \epsilon_2, \qquad \forall k=1,\dots,\tau_{\epsilon_2}'\,.
$$
From Lemma~\ref{lem:N} since $n_{\epsilon_1} \geq  \cN(f,\epsilon_1)$ by Assumption~\ref{ass:polyf},
\[
        \cN'(f,\epsilon_2) = \cN'\Big(f,\frac{\epsilon_1}{n_{\epsilon_1}}\Big) \leq 2n_{\epsilon_1} = 2C \epsilon_1^{-\alpha}\,,
\]
which substituted into the definition of $\tau'_{\epsilon_2}$ yields
\[
        \tau_{\epsilon_2}' \leq 2C \big(1+\lceil \log_2(1/\epsilon_2)\rceil\big) \epsilon_1^{-\alpha} \,.
\]
Thus, the certificate up to time $\tau_{\epsilon_2}'$ is bounded from above as
\[
        \xi_{\tau_{\epsilon_2}} =  \frac{1}{2}\sum_{k=1}^{\tau_{\epsilon_2}'} (a_k^{\tau_{\epsilon_2}'})^2 \leq \frac{\tau_{\epsilon_2}' (\epsilon_2)^2}{2}   \leq \frac{1}{C} \big(1+\lceil \log_2(1/\epsilon_2)\rceil\big)  \epsilon_1^{2+\alpha}\,.
\]
Now, to get rid of the multiplicative term, similarly to Theorem~\ref{thm:GreedyBoxUB}, we can apply Lemma~\ref{lem:areadivided} (which also works for StochasticGreedyBox) to replace $\xi_{\tau_{\epsilon_2}'}$ with $\xi_s$ for $s \geq \tau_{\epsilon_2}'$. The choice $\smash{s = \tau_{\epsilon_2}' \big(\frac{1}{C}\big(1+\lceil \log_2(1/\epsilon_2)\rceil\big)\big)^{1/2}}$ yields
\[
        \xi_s \leq \epsilon_1^{2+\alpha}\,.
\]
Then, choosing $\epsilon_1 = \epsilon^{\frac{2}{2+\alpha}}$ implies $\xi_s \leq \epsilon^2$. Therefore by definition of the stopping criterion
\[
        \tau_\epsilon \leq s = \Big(\frac{1}{C}\big(1+\lceil \log_2(1/\epsilon_2)\rceil\big) \Big)^{3/2} \epsilon_1^{-\alpha}  = \tilde \cO\big(\epsilon^{-\frac{1}{1/\alpha+1/2}}\big) \,.
\]
This concludes the proof.
\end{proof}

\subsection{Proof of Theorem~\ref{thm:c2upper_bound}}

\begin{customthm}{\ref{thm:c2upper_bound}}
  Let $\alpha >0$ and $\epsilon \in (0,1]$. Let $f: [0, 1] \rightarrow [0, 1]$ be a non-decreasing and piecewise-$C^2$ function with a number of $C^1$-singularities bounded by $\epsilon^{-\alpha}$ and such that $|f''(x)| \leq 1$ whenever it is defined. Then, there exists
  $$
  t_\epsilon = \left\{ \begin{array}{ll}  
  \tilde \cO\Big(\epsilon^{-1+\frac{1}{2p+2}}\Big) & \text{ if } \alpha \leq \frac{1}{2} \\
  \tilde \cO\Big(\epsilon^{-1+\big(\frac{1-\alpha}{1+p}\big)_+}\Big) & \text{ if } \alpha \geq \frac{1}{2} \\
  \end{array}\right.
  $$
  such that $\big\|\hat f_t-f\big\|_p\leq \epsilon$ for all $t\geq t_\epsilon$, where $\hat{f}_{t}$ is the approximation of $f$ returned by $\GreedyBox$ after $t$ rounds.
\end{customthm}

\begin{proof}
Let $c \in (0,1]$, and $\gamma>0$ be two constants to be fixed later by the analysis and set $\epsilon' = c \epsilon^{\gamma}$.

\emph{Step 1.} We will now establish an upper bound on the number of evaluations, denoted by $\tau_\epsilon$, required by GreedyBox to ensure that all individual areas are smaller than $\epsilon'$.  From Lemma~\ref{lem:smallboxes}, we know that
\begin{equation}
    \tau_\epsilon \leq 2 \bigl(1+\lceil p\log_2(1/\epsilon')\rceil\bigr) \cN_1'(f,\epsilon') \,,
    \label{eq:boundt}
\end{equation}
which can be further bounded from above by the use of Lemma~\ref{lem:N} in the appendix with the choice $n := c^{-p} \big\lceil \epsilon^{-\frac{\gamma p}{p+1}}\big\rceil$. Indeed, since $c\leq 1$  and by Lemma~\ref{lem:N-basic}, $  n \geq \big\lceil \epsilon^{-\frac{\gamma p}{p+1}}\big\rceil \geq \cN_p(f,\epsilon^{\frac{\gamma p}{p+1}})$. Thus, Lemma~\ref{lem:N} entails
\[
\cN_1'(f,\epsilon') = \cN_1'(f,c\epsilon^\gamma) = \cN_1'\Big(f, \frac{c \epsilon^{\frac{\gamma p}{p+1}}}{\epsilon^{-\frac{\gamma}{p+1}}}\Big) \leq \cN_1'\Big(f, \frac{ \epsilon^{\frac{\gamma p}{p+1}}}{n^{{1}/{p}}}\Big)
\leq 2n = 2  c^{-p} \big\lceil \epsilon^{-\frac{\gamma p}{p+1}}\big\rceil \,,
\]
Therefore, plugging back into Inequality~\eqref{eq:boundt}, the number of required evaluations is bounded as
\begin{equation}
    \label{eq:tau_epsilon}
\tau_\epsilon \leq 4 \bigl(1+\lceil p\log_2(1/\epsilon')\rceil\bigr)  c^{-p} \big\lceil \epsilon^{-\frac{\gamma p}{p+1}}\big\rceil = \tilde{\mathcal{O}}(c^{-p} \epsilon^{-\frac{\gamma p}{p+1}}) \,.
\end{equation}

\medskip
\emph{Step 2.} We will now fix the values of $c \in (0,1]$ and $\gamma>0$ in a way that ensures an approximation error in the $L^p$-norm smaller than $\epsilon$.

For $1 \leq i \leq t$, let us denote by $B_i = [x_i, x_{i+1}] \times [f(x_i), f(x_{i+1})]$ the $i^{\text{th}}$ box maintained by $\GreedyBox$, define $w_i := x_{i+1} - x_{i}$ to be the width of the $\text{i}^{th}$ box and let $\hat f_t$ be the piecewise-linear function returned by $\GreedyBox$ after $t$ epochs.
By the first step of this proof, for all $1\leq i\leq t$,
\[
\cA_p(B_i) := \Big((f(x_{i+1})- f(x_i))^p (x_{i+1}-x_i)\Big)^{1/p} \leq \epsilon'\,,
\]
which implies by construction of $\hat f_t$ (see Proof of Lemma~\ref{lem:GreedyBoxError}), for all $1\leq i \leq t$
\begin{equation}
 \int_{x_i}^{x_{i+1}} \big|\hat f_t(x) - f(x)\big|^p dx \leq (\epsilon')^p \,.
 \label{eq:bound_nonC2}
\end{equation}
The remaining part of the proof revolves around using the above upper bound for the non-smooth pieces and a more refined upper bound for the $C^1$ pieces.
Denote by $\cJ \subseteq \{1,\dots,t\}$ the indices of all boxes such that $f$ is $C^1$ over $[x_i,x_{i+1}]$. Because $f$ is piecewise-$C^1$ with at most $K$ pieces, $\mathrm{Card}(\cJ^c) \leq K$. Therefore, the $L^p$ error after $t$ evaluations may be decomposed as
\begin{multline}
\label{eq:bound_Lp1}
\big\|\hat f_t - f\big\|_p^p  
    = \int_{0}^1 \big|\hat f_t(x) - f(x)\big|^p dx 
    = \sum_{i=1}^t \int_{x_i}^{x_{i+1}} \big|\hat f_t(x) - f(x)\big|^p dx  \\
    \stackrel{\eqref{eq:bound_nonC2}}{\leq} \min\{K,t\} (\epsilon')^p
+ \sum_{i\in \cJ} \int_{x_i}^{x_{i+1}} \big|\hat f_t(x) - f(x)\big|^p dx
\end{multline}
We are now left with bounding from above the $L^p$ errors on the right-hand-side for all $i \in \cJ$. On one side, the bound~\eqref{eq:bound_nonC2} is valid for all $i \in \cJ$, which we bound further by $(\epsilon')^p$.
On the other side,  we can use Lemma~\ref{lem:approxLp_affine} that bounds the $L^p$ approximation error of $\hat f_t$ for any $C^1$ and piecewise-$C^2$ function to obtain
\begin{equation}
    \label{eq:integral}
    \int_{x_i}^{x_{i+1}} \big|\hat f_t(x) - f(x)\big|^p dx \leq  M^p w_i^{2p+1} \,,
\end{equation}
where $M \geq \frac{3}{2} \sup_{x \notin \cX_2} |f''(x)|$, where $\cX_2$ denotes the set of $C^2$-singularities. This prompt us to introduce the following function $\phi$, that depends on the width of the intervals of $\cJ$:
\[
\phi\left((w_i)_{i \in \cJ}\right) = \sum_{i \in \cJ} \min \left\{(\epsilon')^p, M^p w_i^{2p+1} \right\}\;.
\]
From~\eqref{eq:bound_Lp1}, the total error is thus bounded from above by
\begin{equation}
\label{eq:Lpboundphi}
\big\|\hat f_t - f\big\|_p^p  \leq \min\{K,t\} (\epsilon')^p + \phi\left((w_i)_{i \in \cJ}\right) \,.
\end{equation}
It now remains to bound the function $\phi$ for any set $(w_i)_{i \in \cJ}$ of possible widths that could arise from $\GreedyBox$.  We thus need to solve the maximization problem
\begin{align*}
\max_{(w_i)_{i \in \cJ}} & \phi( (w_i)_{i \in \cJ})
\text{ such that } \sum_{i \in \mathcal{J}} w_i \leq 1\;. 
\end{align*}
$\phi$ may be re-written in two terms:
\begin{equation}
\label{eq:phi2sums}
 \phi( (w_i)_{i \in \cJ}) =  \sum_{i \in \cJ, \epsilon' \leq M w_i^{2+\frac{1}{p}} } (\epsilon')^p +  \sum_{i \in \cJ, \epsilon' > M w_i^{2+\frac{1}{p}}  } M^p w_i^{2p+1}\;.
\end{equation}
Let us first handle the first term.
Since $ \sum_{i \in \cJ} w_i \leq 1$, we have in particular
\begin{equation*}
  \abs{ \left\{i \in \cJ \colon \epsilon' \leq M w_i^{2+\frac{1}{p}}\right\} } \cdot \left(\frac{\epsilon'}{M}\right)^{\frac{p}{2p+1}}
  \leq \sum_{i \in \cJ,  \epsilon' \leq M w_i^{2+\frac{1}{p}}} w_i \leq 1\;.
\end{equation*}
This shows that the number of intervals in the first term verifies
\begin{equation}
\label{eq:cardinality_cJ1}
\abs{\left\{i \in \cJ \colon \epsilon' \leq M w_i^{2+\frac{1}{p}}\right\}} \leq M^{\frac{p}{2p+1}}\left(\epsilon'\right)^{-\frac{p}{2p+1}}\;,
\end{equation}
which implies that the first term of \eqref{eq:phi2sums} is bounded as
\begin{equation}
\label{eq:firstterm}
 \sum_{i \in \cJ, \epsilon' \leq M w_i^{2+\frac{1}{p}} } (\epsilon')^p \leq M^{\frac{p}{2p+1}}\left(\epsilon'\right)^{\frac{2p^2}{2p+1}} \,.
\end{equation}

Now, let us bound from above the second term of the sum.
Since $x \mapsto x^{2p+1}$ is strictly convex on $[0,1]$, Lemma~\ref{lem:convex_prop} (in the appendix) applied with $\alpha = \left(\epsilon'/M\right)^{p/(2p+1)   }$ and $\beta \leq \sum_{i\in \cJ} w_i \leq 1$ states that the second term of \eqref{eq:phi2sums} is bounded in the following way:
\begin{equation}
\label{eq:secondterm}
   M^p  \sum_{i \in \cJ, \epsilon' > M w_i^{2+\frac{1}{p}}  } w_i^{2p+1} \leq  2 M^p \alpha^{2p} = 2M^p \left(\frac{\epsilon'}{M}\right)^{\frac{2p^2}{2p+1}} = 2 M^{\frac{p}{2p+1}} (\epsilon')^{\frac{2p^2}{2p+1}} \,.
\end{equation}
Substituting the two upper bounds \eqref{eq:firstterm} and \eqref{eq:secondterm} into~\eqref{eq:phi2sums}, we have
\[
\phi((w_i)_{i \in \cJ}) \leq  3 M^{\frac{p}{2p+1}} (\epsilon')^{\frac{2p^2}{2p+1}}  \,,
\]
which substituted into~\eqref{eq:Lpboundphi} yields
\begin{align}
  \big\|\hat f_t - f\big\|_p^p  
    & \leq \min\{K,t\} (\epsilon')^p + 3 M^{\frac{p}{2p+1}} (\epsilon')^{\frac{2p^2}{2p+1}} \nonumber \\
    & \leq \min\{\epsilon^{-\alpha},t\} c^p \epsilon^{\gamma p} + 3 M^{\frac{p}{2p+1}} c^{\frac{2 p^2}{2p+1}}\epsilon^{\frac{2\gamma p^2}{2p+1}} \,. \label{eq:C2_intemediateUB}
\end{align}

Now, we finalize the proof by considering three cases depending on the value of $\alpha$ and by optimizing $\gamma$ and $c$ for each situation.

\emph{$\bullet$ Case 1: $0 \leq \alpha \leq 1/2$}. Then,
$\epsilon^{-\alpha} \leq \epsilon^{-1/2}$, which substituted in~\eqref{eq:C2_intemediateUB}, implies
\[
     \big\|\hat f_t - f\big\|_p^p \leq c^p \epsilon^{\gamma p - 1/2} + 3 M^{\frac{p}{2p+1}} c^{\frac{2 p^2}{2p+1}} \epsilon^{\frac{2\gamma p^2}{2p+1}} \,.
\]
Choosing $\gamma = 1 + \frac{1}{2p}$ yields
\[
     \big\|\hat f_t - f\big\|_p^p \leq \Big(c^p  + 3 M^{\frac{p}{2p+1}} c^{\frac{2 p^2}{2p+1}}\Big) \epsilon^{p} \,.
\]
The choice $c = \min\{2^{-1/p}, (6M)^{-\frac{1}{2p}}\}$ implies for $t\geq \tau_\epsilon$
\[
    \big\|\hat f_t - f\big\|_p \leq \epsilon \,,
\]
and by~\eqref{eq:tau_epsilon}, the number of required evaluations is of order
\[
  \tau_\epsilon = \tilde \cO\Big(c^{-p} \epsilon^{-\frac{\gamma p}{p+1}}\Big) = \tilde \cO\Big(\epsilon^{-1 + \frac{1}{2p+2}} \Big) \,,
\]
which concludes the first statement of the theorem. 

\medskip
\emph{$\bullet$ Case 2: $\alpha \geq 1$}. Then, one may then use Theorem~\ref{thm:GreedyBoxUB}, which does note use the piecewise-regularity assumption of $f$, and implies that  $\|\hat f_t - f\|_p \leq \epsilon$ for $t \geq \tau_\epsilon$ with
\[
    \tau_\epsilon = \cO\big(\log(1/\epsilon)^2 \cN_p(f,\epsilon)\big) =  \cO\big(\log(1/\epsilon)^2\epsilon^{-1}\big) = \tilde \cO\Big(\epsilon^{-1 + \big(\frac{1-\alpha}{1+p}\big)_+}\Big)\,.
\]

\medskip
\emph{$\bullet$ Case 3: $1/2\leq \alpha \leq 1$}. Then,~\eqref{eq:C2_intemediateUB} yields
\[
 \big\|\hat f_t - f\big\|_p^p
\leq  c^p \epsilon^{\gamma p-\alpha} + 3 M^{\frac{p}{2p+1}} c^{\frac{2 p^2}{2p+1}}\epsilon^{\frac{2\gamma p^2}{2p+1}} \,.
\]
Substituting the choice $\gamma = 1+ \frac{\alpha}{p}$ into~\eqref{eq:C2_intemediateUB} further entails
\[
\big\|\hat f_t - f\big\|_p^p
\leq  c^p \epsilon^{p} + 3 M^{\frac{p}{2p+1}} c^{\frac{2 p^2}{2p+1}}\epsilon^{\big(\frac{2p+2\alpha}{2p+1}\big) p} \leq \Big(c^p + 3 M^{\frac{p}{2p+1}} c^{\frac{2 p^2}{2p+1}} \Big) \epsilon^{p}\,.
\]
Similary to the first case, choosing $c = \min\{2^{-1/p}, (6M)^{-\frac{1}{2p}}\}$ implies
\[
    \big\|\hat f_t - f\big\|_p \leq \epsilon \,,
\]
and for any $t\geq \tau_\epsilon$ of order
\[
  \tau_\epsilon = \tilde \cO\Big(c^{-p} \epsilon^{-\frac{\gamma p}{p+1}}\Big) = \tilde \cO\Big(\epsilon^{-1 + \frac{1-\alpha}{1+p}} \Big) \,.
\]
This concludes the proof.
\end{proof}

\subsection{Proof of Proposition~\ref{prop:counterexample_t_new}}

\begin{customprop}{\ref{prop:counterexample_t_new}}
    Let $p\geq 1$,  $\epsilon \in (0,1)$ and $\alpha >0$. Then, for any deterministic adaptive algorithm $\cA$ and for any 
    $$
        t < (2\epsilon)^{-1 + \big(\frac{1-\alpha}{1+p}\big)_+}-1\,,
    $$
    there exists a non-decreasing piecewise-affine function $f:[0,1]\to[0,1]$ with at most $\max\{2,\lceil \epsilon^{-\alpha}\rceil\}$ discontinuities, such that $\|f - \hat f_t\|_p > \epsilon$.
\end{customprop}

\begin{proof}
Let $p\geq 1$, $\epsilon > 0$ and $\alpha >0$. Let $\cA$ be an adaptive algorithm and fix $t\geq 1$ a number of evaluations. We aim to design a function $f$ with at most $K = \lceil \epsilon^{-\alpha}\rceil$ discontinuities such that $\|f - \hat f_t\|_p > \epsilon$ if $t$ is sufficiently small. 

\smallskip
Let $g:x \to x$ be the identity function on $[0,1]$. Let $x_0 = 0$ and $x_{t+1} = 1$ and denote by $0 \leq  x_1\leq \dots \leq x_t \leq 1$ the points that the algorithm would have chosen after $t$ evaluations, if it was applied on $g$ and by $\hat f_t$ its estimation. Then, we define two functions $g_-$ and $g_+$ such that 
$$
    \|g_- - g_+\|_p^p \geq \min\{ K t^{-1+p}, t^{-p} \}.
$$
Let $K_- = \min\{K, t+1\}$. For $i\in \{1,\dots,t+1\}$, we define $w_i = x_i - x_{i-1}$ the width of the $i$-th interval. Let $\cJ$ denotes the set of indexes that correspond to the $K_-$ largest intervals (i.e., such that $w_i \geq w_j$ for all $i \in \cJ$ and $j \notin \cJ$ and $|\cJ| = K_-$). Then, we define for all $x \in [0,1]$
\[
    g_-(x) = \left\{ \begin{array}{ll}
            x & \text{if } \exists i \notin \cJ, \text{ such that } x \in [x_{i-1},x_i] \\
            x_{i-1} & \text{if } x \in [x_{i-1},x_{i}) \text{ for } i \in \cJ
        \end{array}\right. 
\]
and
\[
    g_+(x) = \left\{ \begin{array}{ll}
            x & \text{if } \exists i \notin \cJ, \text{ such that } x \in [x_{i-1},x_i] \\
            x_{i} & \text{if } x \in (x_{i-1},x_{i}] \text{ for } i \in \cJ
        \end{array}\right.  \,.
\]
Then, $g_-$ and $g_+$ have at most $K$ discontinuities and are such that $g_-(x_i) = g_+(x_i) = g(x_i) = x_i$ for all $i \in \{1,\dots,t\}$. Thus, since $\cA$ is deterministic, the function estimation and the points chosen by $\cA$ on $g_-$ and $g_+$ after $t$ evaluations would also respectively be $\hat f_t$ and $x_1,\dots,x_t$. 

Furthermore, we have
\begin{align*}
    \big\|g_- - g_+\big\|_p^p 
        & = \int_0^1 \big|g_-(x) - g_+(x)\big|^p dx\\
        & = \sum_{i \in \cJ} \int_{x_{i-1}}^{x_i} \big|g_-(x) - g_+(x)\big|^p dx \\
        & = \sum_{i \in \cJ} w_i^{p+1} \\
        & \geq K_- \Big(\frac{1}{K_-} \sum_{i \in \cJ} w_i \Big)^{p+1} \qquad \leftarrow \quad \text{by Jensen's Inequality} \\
        & \geq K_- \Big(\frac{1}{t+1} \sum_{i=1}^{t+1} w_i \Big)^{p+1} \hspace*{.6cm} \leftarrow \quad \text{by Definition of $\cJ$}\\
        & = K_- (t+1)^{-(p+1)} \hspace*{1.4cm} \leftarrow \quad \text{because } \sum_{i=1}^{t+1} w_i = 1\,.
\end{align*}
By triangular inequality, this yields
\begin{align*}
    \max_{f \in \{g_-,g_+\}} \big\|\hat f_t - f\big\|_p 
        & \geq \frac{1}{2} \Big( \big\|\hat f_t - g_-\big\|_p + \big\|\hat f_t - g_+\big\|_p \Big)\\
        & \geq \frac{1}{2} \big\|g_- - g_+\big\|_p \\
        & \geq \frac{1}{2} K_-^\frac{1}{p} (t+1)^{-\frac{p+1}{p}} \,.
\end{align*}
Therefore, $\max_{f \in \{g_-,g_+\}} \big\|\hat f_t - f\big\|_p > \epsilon$ if
\[
    \frac{1}{2} K_-^\frac{1}{p} (t+1)^{-\frac{p+1}{p}} > \epsilon
\]
which, using  $K_- \geq \min\{t+1,\epsilon^{-\alpha}\}$, is satisfied for 
\[
    t < \min\big\{(2\epsilon)^{-1 + (\frac{1-\alpha}{1+p})}, (2\epsilon)^{-1} \big\} -1 \,.
\]
Noting that $g_-$ and $g_+$ have at most $K \leq \lceil\epsilon^{-\alpha}\rceil$ discontinuities and that $g''(x) = 0$ elsewhere concludes the proof. 
\end{proof}

\subsection{Proof of Proposition~\ref{prop:counterexample_t}}

\begin{customprop}{\ref{prop:counterexample_t}}
  Let $\epsilon \in (0,1/12)$. Then, there exists piecewise-$C^2$ function $f_\epsilon$ with one $C^1$-singularity, such that there exists $t \geq 2^{-7} \epsilon^{-3/4}$ with
  $
    \big\|\hat f_t - f_\epsilon \big\|_1 > \epsilon
  $
  where $\hat f_t$ is the $\GreedyBox$ approximation at $t$ evaluations.
\end{customprop}

\begin{proof}
Let $k \in \N$ that will be chosen later by the analysis.
Set $s = 2^{3k}$, $s' = 2^{2k}$ and $t = \frac{s + s'}{2}$.

\smallskip
\emph{Step 1. Design of a worst-case function} We will design a worst-case function $f^t$ that will cause GreedyBox to incur a large $L^1$-error after $t$ iterations. The function will consist of two components: one for $x \leq 1/2$ that oscillates with a second derivative $|({f^t})''(x)| = 1$, and another that is linear for $x \geq 1/2$.
$f^t$ has a continuous derivative and its second derivative is piecewise-continuous wit $K = \frac{s'}{2}$ singularities.
An illustration is given in Figure~\ref{fig:f6} for $k=2$ ($t=40$).

\smallskip
Let us first design the first oscillating part, that we call $g_{s'}$, and is defined recursively for all $x \in [0, 1]$ by
\begin{equation}
\label{eq:gfunction}
g_{s'}(x) = \left\{
  \begin{array}{ll}
    x^2 &\text{ if }  x \in \left[0, \frac{1}{s'}\right]\\
    -\left(x - \frac{1}{s'}\right)^2 + \frac{2}{s'}\times\left(x-\frac{1}{s'}\right) + \left(\frac{1}{s'}\right)^2 & \text{ if }  x \in \left[\frac{1}{s'}, \frac{2}{s'}\right]\\
    g_{s'}\left(x-\frac{2i}{s'}\right) + 2i \left(\frac{1}{s'}\right)^2 & \hspace*{-1.5cm} \text{ if }  x \in \left[\frac{2i}{s'}, \frac{2(i+1)}{s'}\right], i \in \big\{1, \ldots, \frac{s'}{2}-1\big\}\;.\\
  \end{array}
\right.
\end{equation}
Then, we define $f^t$ by: for all $x \in [0,1]$
\begin{equation}
\label{eq:bad_function}
f^t(x) = \mathds{1}_{\{x \leq 1/2\}} g_{s'}(x) + \mathds{1}_{\{x > 1/2\}} \left(x - \frac12 + \frac{1}{2s'}\right) \,.
\end{equation}

\medskip
\emph{Step 2. Expression of the approximation $\hat f_t$ provided by GreedyBox after $t$ iterations on $f^t$}. Let us understand how $\GreedyBox$ behaves during the first $t$ iterations when given this function.
Once done, we will be able to retrieve the expression of the approximation $\hat{f}_t$ of $f^t$ made by $\GreedyBox$.

Remark that for $i \in \{0, \ldots, s'/2-1\}$, (during the first oscillating part), the area of the box $B_i = \left[\frac{i}{s'}, \frac{i+1}{s'}\right] \times \left[f^t\left(\frac{i}{s'}\right), f^t\left(\frac{i+1}{s'}\right)\right]$ is
$$
\cA_p(B_i) = \frac{1}{s'} \times \frac{1}{{s'}^2}=\frac{1}{{s'}^3} = \frac{1}{s^2}.
$$
Similarly, for $j \in \left\{\frac{s}{2}, \ldots,  s-1\right\}$, (during the linear part), the area of the box $B_j = \left[\frac{j}{s}, \frac{j+1}{s}\right] \times \left[f^t\left(\frac{j}{s}\right), f^t\left(\frac{j+1}{s}\right)\right]$ is
$$
\cA_p(B_j) = \frac{1}{s^2}.
$$
Thus, because $\GreedyBox$ tends to equalize the areas of the different boxes, and because it maintains areas of the form $[i/2^j, (i+1)/2^j]$ for some $j$ in $\N^*$ and  $ i \in \{0, \ldots,  2^j-1\}$ (it can only splits intervals in $2$), the box cover we just described is a potential output of $\GreedyBox$ after $t = (s'+s)/2$ epochs.

Let us formally prove that it is precisely the case. For any $0 \leq j'\leq 3k$ and $i \in \{1, \ldots, 2^{j'-1}-1\}$, the area of the box $[i2^{-j'} , (i+1)2^{-j'}] \times [\hat{f}_t(i2^{-j'}) + \hat{f}_t((i+1)2^{-j'})]$ is $1/8^{j'}$.
Likewise, for any $0 \leq j \leq 2k$ and $i \in \{2^{j-2}, \ldots 2^{j-1}-1\}$, the area of the box $[i2^{-j} , (i+1)2^{-j}] \times [\hat{f}_t(i2^{-j}) + \hat{f}_t((i+1)2^{-j})]$ is $1/4^{j}$.
Thus, if for some epoch $t' \leq t$, one box on the left of $1/2$ has a width greater than $3k$, its area will be greater than or equal to $8/s^2$, and if one box on the right of $1/2$ has a width greater than $2k$, its area will be greater than or equal to $4/s^2$.
In both cases, the area is strictly greater than $1/s^2$, and  $\GreedyBox$ would choose at Line~\ref{line:GB_selection} to split this box rather than any other box that already has area $1/s^2$.

Splitting one by one boxes with area greater than $1/s^2$, $\GreedyBox$ obtains at time $t$ the $t$ boxes with exact area $1/s^2$ described previously.
$\hat{f}_t$ is the linear interpolation between all the points $(x_i, f^t(x_i))$ where $x_i$ is the extremity of one of the box.

\medskip
\emph{Step 3. Error of GreedyBox on $f^t$ after $t$ iterations}. Now that we exhibited the exact value of $\hat{f}_t$, we are left with computing the total error made by $\GreedyBox$ on $f^t$ after $t$ steps.
Since $f^t$ is linear on the interval $[1/2, 1]$, $\hat{f}_t$ equals $f^t$ on this segment. Then,
\begin{align*}
\big\|f^t - \hat{f}_t\big\|_1 &= \int_0^{\frac{1}{2}} \abs{f^t(x) - \hat{f}_t(x)}dx
= \sum_{i = 1}^{s'/2-1} \int_{\frac{i}{s'}}^{\frac{i+1}{s'}} \abs{f^t(x) - \hat{f}_t(x)}dx\\
&= \frac{s'}{2} \int_{0}^{\frac{1}{s'}} \abs{f^t(x) - \hat{f}_t(x)}dx = \frac{s'}{2} \int_{0}^{\frac{1}{s'}} \abs{x^2 - \frac{x}{s'}}dx\\
&= \frac{s'}{2} \left( \frac{1}{2{s'}^3} - \frac{1}{3{s'}^3} \right) = \frac{1}{12{(s')}^{2}}\;.\\
\end{align*}

\medskip
\emph{Step 4. Choice of $k$.} We are left with choosing $k$ as large as possible such that the above error is at least $\epsilon$. That is,
\[
\big\|f^t - \hat{f}_t\big\|_1 > \epsilon
\]
which can be rewritten as
\[
\frac{2^{-4k}}{12} = \frac{1}{12{s'}^{2}} > \epsilon  \qquad \Leftrightarrow \qquad k < \frac{1}{4} \log_2\big(\frac{1}{12\epsilon}\big) \,.
\]
Thus choosing $k = \big\lfloor \frac{1}{4} \log_2\big(\frac{1}{12\epsilon}\big) \big\rfloor$, yields
\[
t = \frac{2^{2k}+2^{3k}}{2} \geq 2^{3k-1} \geq \big(12\epsilon\big)^{-\frac{3}{4}}2^{-4} > 2^{-7} \epsilon^{-\frac{3}{4}} \,.
\]
Note that the designed function has a number of $C^2$-singularities
\[
K = \frac{s'}{2} = 2^{2k-1} \leq \frac{1}{2}(12\epsilon)^{-1/2} \leq \epsilon^{-1/2}\,,
\]
but only has one $C^1$-singularity at $x=1/2$. This concludes the proof.
\end{proof}

\subsection{Technical lemmas for regular functions}

\label{app:C2trapeze}
In this section, we establish two technical lemmas that are used in our analysis of GreedyBox for piecewise-$C^2$ functions. The first lemma, presented below, is a property of strictly convex functions.

\medskip
\begin{lemma}
  \label{lem:convex_prop}
  Let $n \in \N^*$, $\beta \in \R^+$, $\alpha \in [\beta/n, \beta]$ and $f$ be a strictly convex function on $[0, \beta]$.
  Then
  \[ \max \left\{\sum_{i=1}^n f(x_i) \colon \forall i, x_i \leq \alpha, \sum_{i=1}^n x_i = \beta \right\} = m f(\alpha) + f(\beta - m\alpha) + (n-m+1) f(0)\;,\]
  and the maximum is reached for $x^\star_1 = \ldots = x^\star_m = \alpha$, $x^\star_{m+1} = \beta - m\alpha$, $x^\star_{m+2} = \ldots = x^\star_n = 0$, where $m = \lfloor\beta/\alpha \rfloor$.
\end{lemma}

\begin{proof}
Let $x^\star \in \R^n$ be as in the statement of the theorem, and let us show that it is optimal.
Let $x$ be a real-number.
Since only the sum of the $f(x_i)$ matters, we can assume without loss of generality that the $x_i$'s are sorted in decreasing order.
Now, assume that $x_m< \alpha = x^\star_m$, and let us show that $\sum_{i=1}^n f(x_i)$ is not a maximum.

Because the sum of the $x_i$'s still needs to be equal to $\beta$, either $x_{m+1} > x^\star_{m+1} = \beta - m\alpha$, or $x_{m+2}> x^\star_{m+2} = 0$.
Assume first that $x_{m+2} > x^\star_{m+2}$.
Furthermore, assume that $\min\{x^\star_m -x_m, x_{m+2} - x^\star_{m+2}\} = x^\star_m - x_m$.
This means that $x_m$ is closer to $x^\star_m$ than $x_{m+2}$ is from $x^\star_{m+2}=0$.
Let  $x = x^\star_m$, $\lambda = \frac{x^\star_m - x_{m+2}}{2x^\star_m - x_m - x_{m+2} } \in (0,1)$ and $y = x_{m+2}+x_m-x^\star_m\in [0, x_{m+2})$.
Then,
\begin{align*}
f(x_{m}) + f(x_{m+2})
&= f(\lambda x+ (1-\lambda)y) + f((1-\lambda) x + \lambda y) \\
&< \lambda f(x) + (1-\lambda)f(y) + (1-\lambda)  f(x) + \lambda f(y) \\
&\leq f(x^\star_m) + f(y) \;.
\end{align*}
Then for $x'_m = x$, $x'_{m+2} = y$ and $x'_i = x_i$ for $i \neq \{m, m+2\}$, $\sum_{i=1}^n f(x'_i) < \sum_{i=1}^n f(x_i)$, which shows that $(x_i)_{1 \leq i \leq n}$ is suboptimal.
We can do the same kind of construction when $x_{m+1} > x^\star_{m+1}$ or when $\min\{x^\star_m -x_m, x_{m+2}-x^\star_{m+2}\} = x_{m+2} - x^\star_{m+2}$.
All these different cases show that $\sum_{i=1}^n f(x_i)$ is maximized only when the largest possible amount of the $x_i$'s are at the extremity of the constraint set, that is when $x_i = 0$ or $x_i = \alpha$.
This concludes the proof.
\end{proof}

The second technical Lemma below recalls a classical result and provides an upper bound on the $L^p$-error achieved by an affine approximation of a $C^2$ function. In particular, this lemma implies a sample complexity of order $\cO(\epsilon^{-\frac{1}{2}})$ for the trapezoidal rule to provide an $\epsilon$-approximation in $L^p$ norm for $C^2$ functions.

\medskip
\begin{lemma}
\label{lem:approxLp_affine}
Let $a < b$. Assume that $f: [a,b] \to [0,1]$ is $C^1$ and piecewise-$C^2$ and such that  $|f^{(2)}(x)| \leq M$ for all $x$ where it is defined. Then,
\[
    \int_a^b \Big|\hat f(x) - f(x)\Big|^p dx \leq \Big(\frac{3M}{2}\Big)^p (b-a)^{2p+1}\,,
\]
where $\hat f$ is the affine approximation defined for all $x \in [a,b]$ by:
\[
  \hat f(x) = f(a) +  \big(f(b)-f(a)\big)\frac{x-a}{b-a} \,.
\]
\end{lemma}

\begin{proof}
Since $f$ is piecewise-$C^2$ with bounded second derivative, $f'$ is absolutely continuous, we can thus apply the Taylor-Lagrange Theorem with the integral form of the remainder with $k = 1$. We have for all $x,y \in [a,b]$
\[
    \big|f(y) - f(x) - f'(x)(y-x)\big| = \Big|\int_x^y f''(u)(y-u) du\Big| \leq \frac{M}{2} (y-x)^2\,.
\]
We now control the $L^p$-error of $\hat f$. The above inequality yields that for all $x \in [a,b]$ there exists $r(x)$ such that $|r(x)| \leq  \frac{M}{2} (x-a)^2$ and
\[
    f(x) = f(a) + f'(a)(x-a) + r(x) \,.
\]
Thus, applying it with $x=b$ entails
\[
    \frac{f(b)-f(a)}{b-a} = f'(a) +  \frac{r(b)}{b-a} \,,
\]
which in turns implies that
\begin{align*}
    \big|\hat f(x) - f(x)\big|
        & = \big| f(a) + (x-a)\frac{\big(f(b)-f(a)\big)}{b-a} - f(x)\big| \\
        & = \big| (x-a) \frac{r(b)}{b-a} + r(b) - r(x) \big| \\
        & \leq \frac{M}{2} \big[ (b-a)(x-a) + (b-a)^2 + (x-a)^2\big] \,.
\end{align*}
Therefore, with the change of variable $x= a+ (b-a)u$, we get
\begin{align*}
    \int_a^b \big|\hat f(x) - f(x)\big|^p dx
        & \leq \Big(\frac{M}{2}\Big)^p \int_a^b \big[ (b-a)(x-a) + (b-a)^2 + (x-a)^2\big]^p dx \\
        & = \Big(\frac{M}{2}\Big)^p (b-a)^{2p+1} \int_0^1 (1+ u + u^2)^p du\\
        & \leq \Big(\frac{3M}{2}\Big)^p (b-a)^{2p+1} \,,
\end{align*}
which concludes the proof.
\end{proof}

\section{Numerical experiments: the \texorpdfstring{$L^2$}{L2}-norm case}
\label{sec:more_experiments}

In the core of the paper, we gave some plots on the error of $\GreedyBox$, and $\GreedyWB$ as compared to the trapezoidal rule in Section~\ref{sec:experiments}.
In this section, we complete the comparison by displaying some plots for the $L^2$-norm.
In this case, the area of a box $B = [c^-, c^+] \times [y^-, y^+]$ is worth $(y^+-y^-)(c^+ - c^-)^2$.

We can notice several interesting facts on the plots of Figure~\ref{fig:plot_rates_2norm}.
First, the trapezoidal rule, and $\GreedyWB$ behaves better than $\GreedyBox$ for the square function. This is the first case we could find where the trapezoidal rule has a better speed of convergence than $\GreedyBox$.
Another point is that on Figure~\ref{fig:power_rate_2norm}, the trapezoidal rule has a bound of order $t^{11/20}$.
However, in Theorem~\ref{thm:GreedyBoxUB}, we proved that $\GreedyBox$ has an error at worst of the order of $\cN(f, \epsilon)$, which we proved to be smaller than $\left\lceil\epsilon^{-1}\right\rceil$.
Yet, for this example, the trapezoidal does not converge in $\mathcal{O}(t)$, which proves that $\GreedyBox$ satisfies better worst case property than the trapezoidal rule for $L^p$-norm with $p>1$.
Figure~\ref{fig:nearly1_rate_2norm} shows results similar to the $L^1$-norm. In, Figure~\ref{fig:pathological_rate_2norm}, we consider the $L^2$ for approximating $g^t:x \mapsto \frac{1}{2} f^{t^{9/10}}(2x) \indic_{x\leq 1/2} + \indic_{x>1/2}$, where $f^s$ is the function defined in equation~\eqref{eq:bad_function} at $s = t^{9/10}$. The empirical rates seem to confirm once again the anticipated worst-case rates as determined by our analysis. 

\begin{figure}[ht]
  \subcaptionbox{Error rate in 2-norm on $f\colon x \mapsto x^2$\label{fig:square_rate_2norm}}{\includegraphics[width=.5\linewidth]{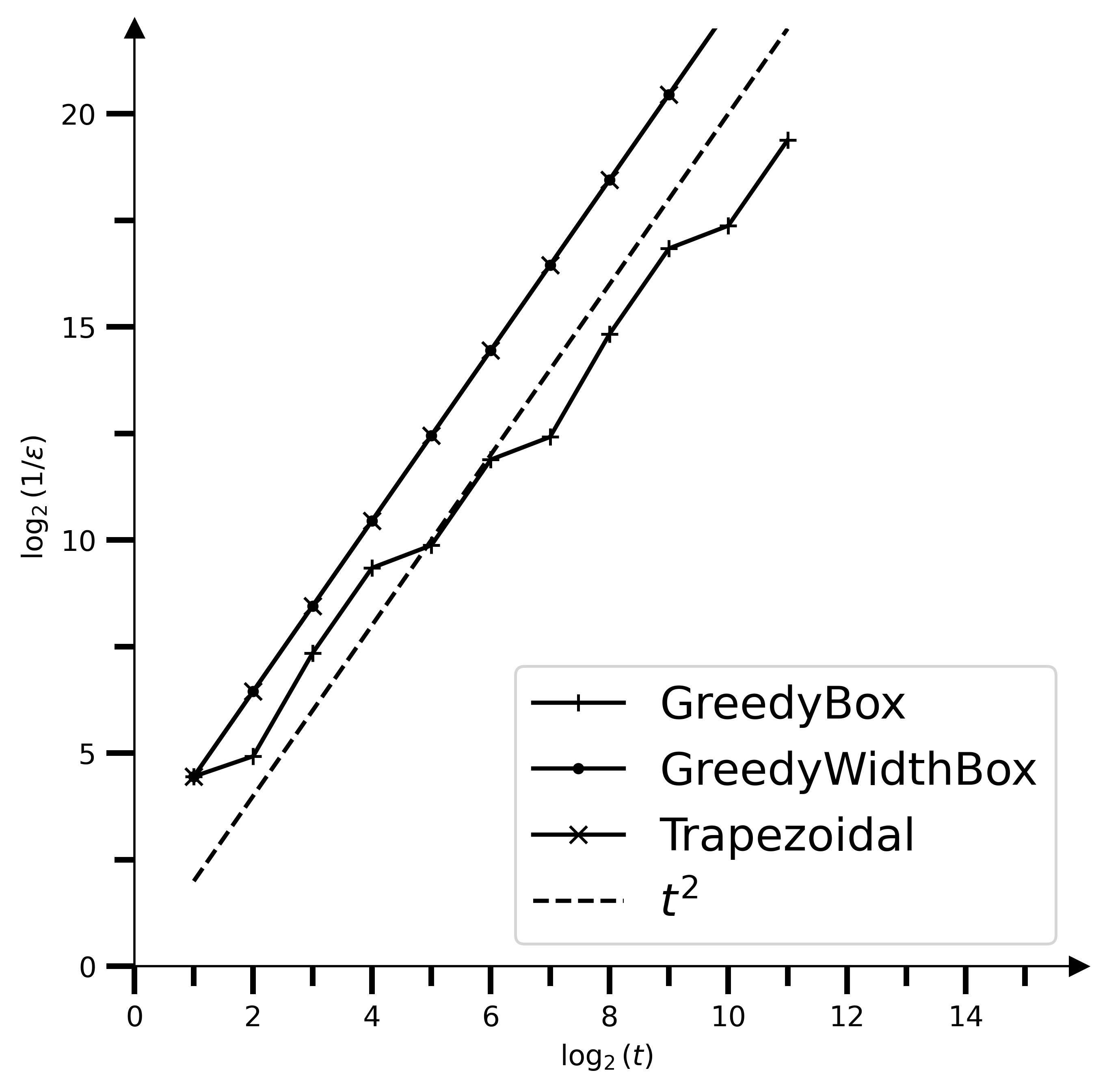}}\hfill
  \subcaptionbox{Error rate in 2-norm on $f\colon x \mapsto x^{1/10}$\label{fig:power_rate_2norm}}{\includegraphics[width=.5\linewidth]{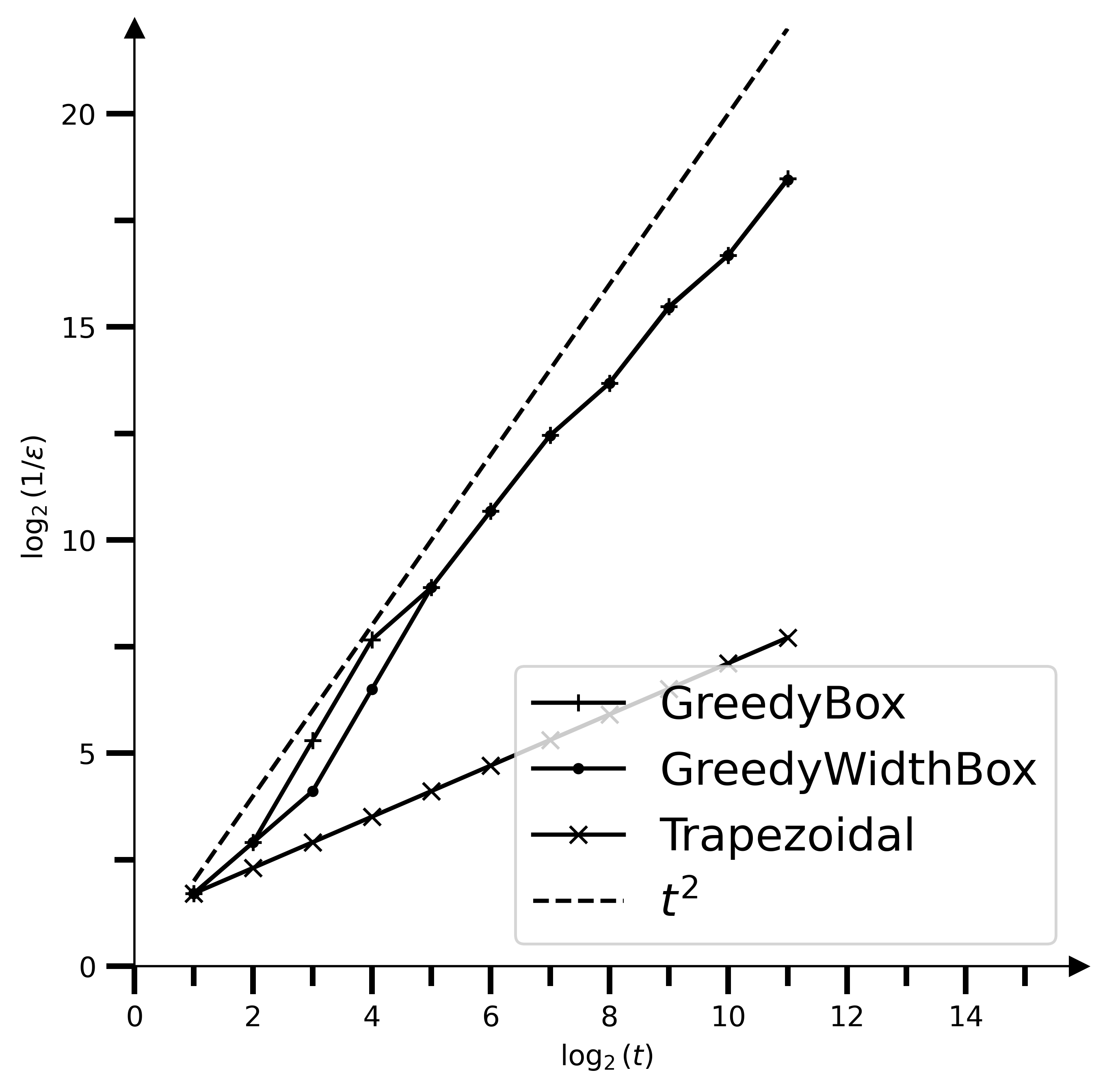}}\hfill\\
  \subcaptionbox{Error rate in 2-norm on $f(x)=\mathds{1}_{\{x\geq 0.3\}}$\label{fig:nearly1_rate_2norm}}{\includegraphics[width=.5\linewidth]{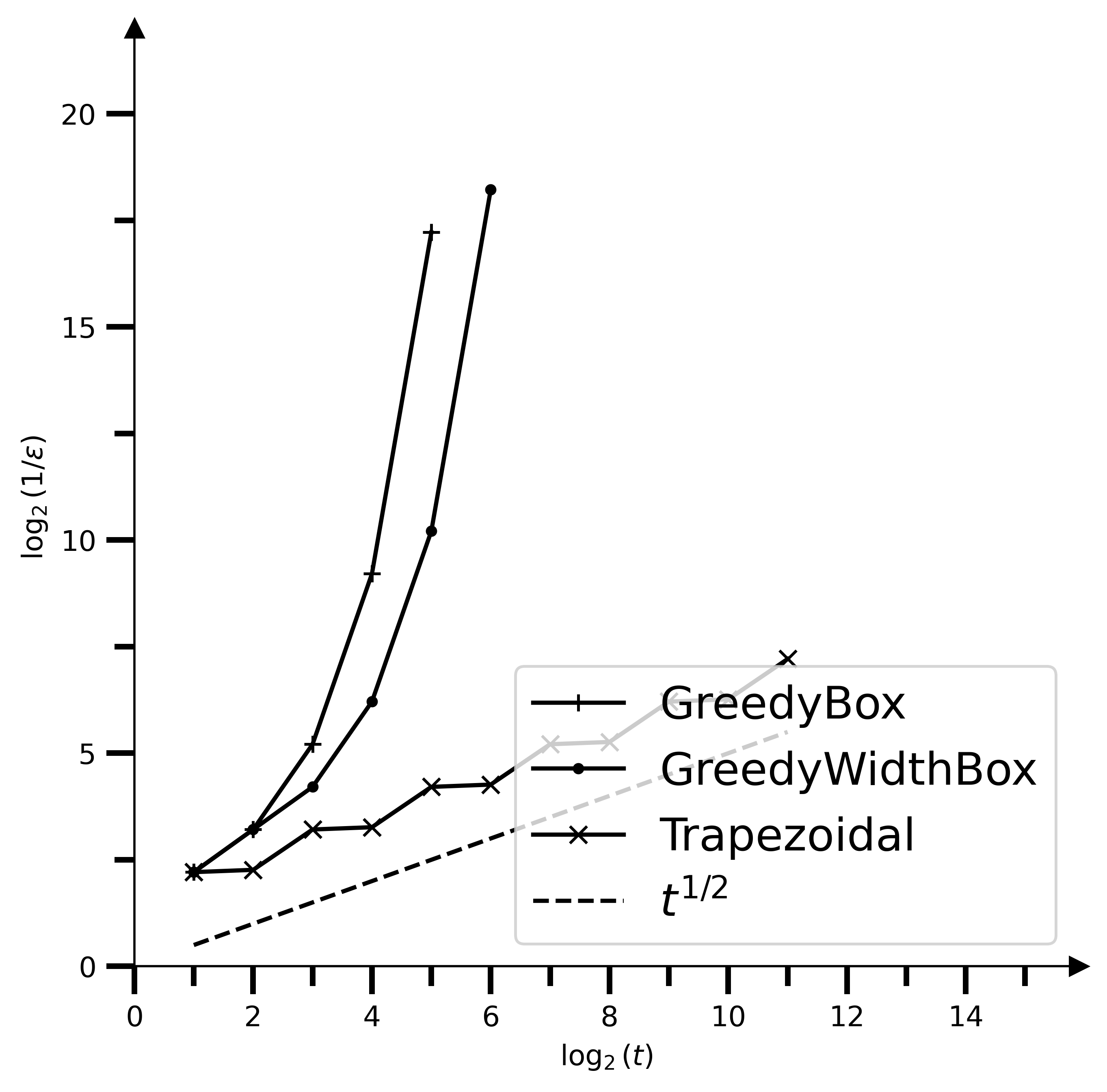}}\hfill
  \subcaptionbox{Error rate in 2-norm on $g^{t}$ \label{fig:pathological_rate_2norm}}{\includegraphics[width=.5\linewidth]{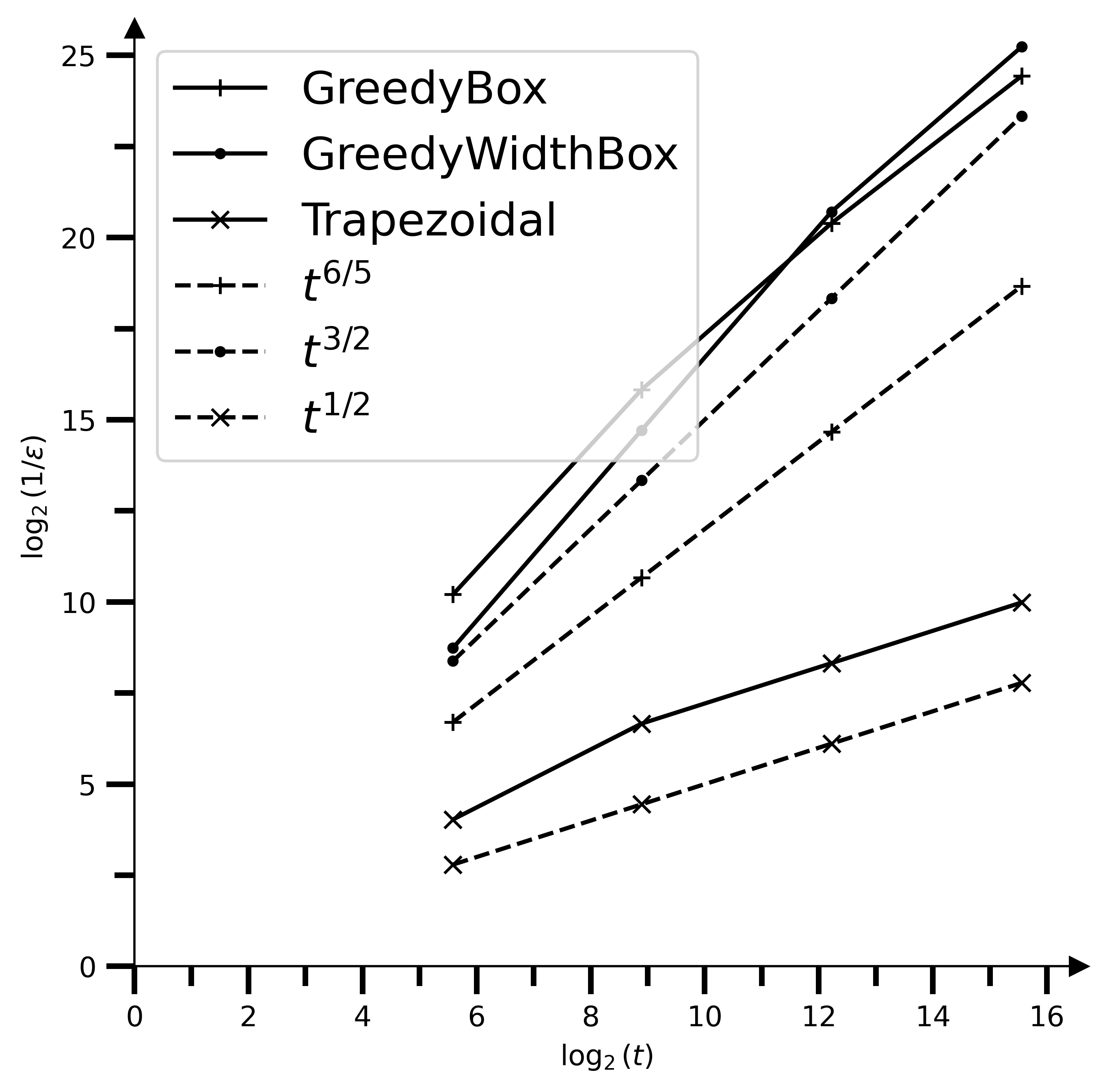}}\hfill
  \caption{Comparison of $\GreedyBox$ and $\GreedyWB$ with the trapezoidal rule for the 2-norm.}
  \label{fig:plot_rates_2norm}
\end{figure}

\end{appendices}

\bibliography{biblio}% common bib file

\end{document}

%% file: example_greedybox.tex
\begin{tikzpicture}[scale=4]
  % Draw the axes
  \draw[->, >=latex, line width=1] (0, 0) -- (0,1.1);
  \draw[->, >=latex, line width=1] (0, 0) -- (1.1,0);

  % Draw the boxes
  \draw[fill=gray, opacity=0.5] (0,0) rectangle (0.25, 0.25^2);
  \draw[fill=gray, opacity=0.5] (0.25, 0.25^2) rectangle (0.5, 0.5^2);
  \draw[fill=gray, opacity=0.5] (0.5, 0.5^2) rectangle (0.75, 0.75^2);
  \draw[fill=gray, opacity=0.5] (0.75, 0.75^2) rectangle (0.875, 0.875^2);
  \draw[fill=gray, opacity=0.5] (0.875, 0.875^2) rectangle (1, 1);
  \draw[blue, dashed] (0.625,0) -- (0.625,0.625^2);

  % Draw the function
  \draw[domain=0:1, samples=101, line width=1, variable=\x] plot ({\x}, {\x^2});

  % Draw the approximation function
  \draw[line width=1, dashed] (0,0) -- (0.25, 0.25^2) -- (0.5, 0.5^2) --(0.75, 0.75^2) -- (0.875, 0.875^2) -- (1, 1);

  % Draw the points chosen by GreedyBox
  \draw (0,0) node[below]{$x_0$};
  \draw (0,0) node{$|$};
  \draw (0.25,0) node[below]{$x_1$};
  \draw (0.25,0) node{$|$};
  \draw (0.5,0) node[below]{$x_2$};
  \draw (0.5,0) node{$|$};
  \draw (0.75,0) node[below]{$x_3$};
  \draw (0.75,0) node{$|$};
  \draw (0.875,0) node[below]{$x_4$};
  \draw (0.875,0) node{$|$};
  \draw (1,0) node[below]{$x_5$};
  \draw (1,0) node{$|$};
  \draw (0.625,0) node[blue]{$|$};
  \draw (0.625,0) node[below, blue]{$x_6$};
\end{tikzpicture}